\documentclass{article}
\usepackage[T1]{fontenc}
\usepackage{preprint,times}
\usepackage{amsmath,amssymb,amsfonts,amsthm}
\usepackage{booktabs}
\usepackage{graphicx}
\usepackage{flafter}
\usepackage{wrapfig}
\usepackage{needspace}
\usepackage{xcolor}
\usepackage{algorithm}
\usepackage{algpseudocode}
\usepackage[hyphens]{url}
\usepackage{hyperref}
\hypersetup{hidelinks,pdfauthor={Alexey Kurennoy, Ramil Yarullin, Fergal Reid}}

\newtheorem{theorem}{Theorem}
\newtheorem{proposition}{Proposition}
\newtheorem{corollary}{Corollary}
\newtheorem{lemma}{Lemma}
\newtheorem{example}{Example}
\DeclareMathOperator*{\argmax}{arg\,max}
\DeclareMathOperator*{\var}{var}

\newcommand{\E}{\mathbb{E}}
\newcommand{\KL}{\mathrm{KL}}

\newcommand{\reward}{\mathcal{R}}
\newcommand{\piref}{\pi_{\mathrm{ref}}}
\newcommand{\ptil}{\widetilde{p}}
\newcommand{\support}{\mathcal{S}}
\newcommand{\lrho}{\ell_{\rho,w}}
\newcommand{\vocab}{\mathcal{V}}
\newcommand{\maxlr}{\ensuremath{\text{max-}\lrho\text{-of-}9}}

\title{Metropolis--Hastings Dominates Importance\\Resampling for Policy Composition}

\preprintfinalcopy
\author{Alexey Kurennoy\thanks{Equal contribution.} \quad Ramil Yarullin\footnotemark[1] \quad Fergal Reid \\
fin.ai \\
\texttt{\{alexey.kurennoy,ramil.yarullin,fergal.reid\}@fin.ai}}

\begin{document}
\maketitle
\suppressfloats[t]

\begin{abstract}
    Post-training a large language model (LLM) often requires exploring trade-offs between multiple rewards, but retraining for each trade-off is expensive. Decoding-time policy composition allows these trade-offs to be adjusted by combining reward-specific policies at inference time. This composition targets a weighted product of the policies' probabilities over complete responses, but standard implementations combine their next-token probabilities, generally introducing sampling bias. We analyze a known iterative correction based on independence Metropolis--Hastings (MH). Our main result shows that, for every rollout budget, MH produces an output distribution at least as close to the target as sampling-importance-resampling (SIR) with the same budget, as measured by every convex $f$-divergence. We also derive a lower bound on MH's improvement over the uncorrected decoder in a consensus objective measuring agreement with the supplied policies. We further characterize the correction's sampling error in two asymptotic regimes: when the reward-specific policies approach agreement, and when fluctuations in the log ratio between target and uncorrected-decoder probabilities grow, as can happen for long responses.
    We complement our analysis with experiments in enumerable and LLM-scale settings.
\end{abstract}

\section{Introduction}\label{sec:introduction}

\begin{figure}[t]
\centering
\includegraphics[width=0.90\textwidth]{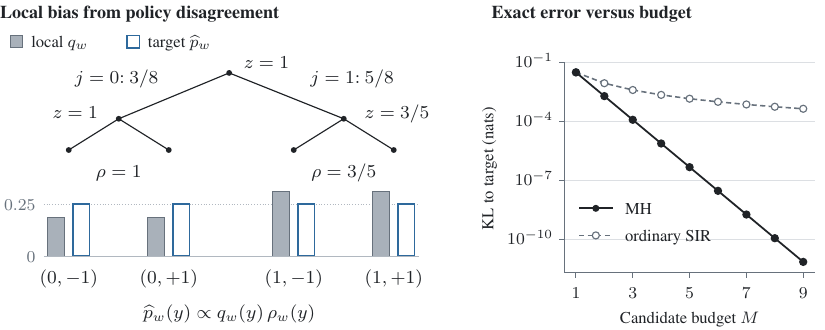}
\caption{\textbf{Local bias and finite-budget correction in Example~\ref{ex:local}.} Left: the two policies agree after $j=0$ and disagree after $j=1$. Reweighting the local law $q_w$ by the recorded normalizers $\rho_w$ recovers the uniform target. Right: exact KL to the target for MH and ordinary SIR on the same four-response toy example.}
\label{fig:overview}
\end{figure}

Large language model (LLM) alignment involves balancing objectives such as
helpfulness, safety, factual accuracy, and brevity
\citep{bai2022hh,dai2024saferlhf,wu2023finegrained,wang2024armorm}. Doing so requires exploring different trade-offs, but retraining a policy for every trade-off is expensive. Decoding-time
composition offers an alternative: it takes policies optimized individually for each objective and combines them at inference time
\citep{shi2024mod,shen2026mahalo}. By varying the weights of the reward-specific policies in the composition, one can explore different trade-offs without retraining.

The composition aims to sample from a weighted geometric average of complete-response probabilities:
\begin{equation}\label{eq:intended}
    \widehat p_w(y\mid x)\;\propto\;
    \prod_{i:w_i>0}\pi_i(y\mid x)^{w_i},
    \qquad w_i\ge0,\quad \sum_{i=1}^m w_i=1,
\end{equation}
where $\pi_i(y\mid x)$ is the probability that policy $i$ assigns to response
$y$ given prompt $x$.
However, sampling from the target~\eqref{eq:intended} token by token generally requires summing over all
possible continuations. A practical decoder instead multiplies the
next-token probabilities and normalizes at each step. We call this process
\emph{local product decoding}. The two constructions generally differ because the policies can agree on next-token probabilities but disagree on later continuations (see Example~\ref{ex:local}). The sequence target accounts for the total product mass of these continuations, whereas the local rule ignores it. Consequently, even exact reward-specialized policies
can produce a biased decoder for their combined objective.

Let $q_w$ denote the local decoder and $\rho_w$ the product of its stepwise
normalization constants. Then $\widehat p_w\propto q_w\rho_w$.
Each rollout therefore provides its importance weight: the target-to-proposal
probability ratio, up to a response-independent constant.
The independence Metropolis--Hastings (MH) construction of
\citet{gareev2024local} uses
these weights to accept or reject independently generated responses.
It requires additional rollouts, but no separate reward model or
post-generation model evaluation to score a candidate. Proposals can be
batched, with selection performed afterward.

An alternative to MH is sampling-importance-resampling (SIR),
which selects a candidate with probability proportional to its
importance weight. Both MH and SIR use the same recorded
information and the same budget of $M$ rollouts, including initialization.

In this paper, we present a new theoretical analysis of the MH correction. Our central result states that \emph{at every budget $M$, MH's output distribution is at least as close to the sequence target as ordinary SIR's under every convex $f$-divergence.} In more detail, our contributions are as follows.

\begin{enumerate}
    \setlength{\itemsep}{3pt}
    \setlength{\parskip}{0pt}
    \setlength{\parsep}{0pt}
    \item We prove that \textbf{MH is no farther from the target than ordinary SIR at any budget.}
    Moreover, for KL in both directions and total variation, MH is strictly closer
    for $M\ge2$ whenever the local decoder differs from the target
    (Theorem~\ref{thm:dominance}).

    \item We derive a \textbf{lower bound on MH's improvement over the uncorrected decoder} in a consensus objective measuring agreement with the supplied policies. The bound can be estimated directly from quantities recorded during generation, and we give conditions under which it extends to scalarized reward optimization (Proposition~\ref{prop:monotone}).

    \item We characterize the \textbf{behavior of the correction's sampling error in two asymptotic regimes}.
    The first regime concerns reward-specific policies approaching agreement. In this setting, MH's sampling error decays at a higher asymptotic order than SIR's (Theorem~\ref{thm:nearby}). The second regime captures limitations that can arise for long responses. In this regime, both MH and SIR attain the asymptotically minimal total variation error among methods selecting from independently generated rollouts at the same budget (Theorem~\ref{thm:critical}).
\end{enumerate}

We complement the theory with experiments in enumerable and LLM-scale settings (Section~\ref{sec:experiments}).\footnote{The associated code is available at \url{https://github.com/akurennoy/mh-policy-composition.git}.}

\section{Correcting local product decoding}\label{sec:algorithm}

Consider $m$ autoregressive policies $\pi_1,\ldots,\pi_m$ and weights
$w_i\ge0$ that sum to $1$. The active policies, those with nonzero weights, are indexed by $I_w=\{i:w_i>0\}$. We assume that the policies share a tokenizer and a finite vocabulary $\vocab$.

For a prompt $x$, a response ends at the first end-of-sequence (EOS) token
or at the length cap $T$. We write $y=(y_1,\ldots,y_T)$, padding shorter responses deterministically with EOS tokens. At prefix $h_t=(x,y_{<t})$, policy $i$ assigns next-token
probability $\pi_{i,t}(v\mid h_t)$, so
$\pi_i(y\mid x)=\prod_t\pi_{i,t}(y_t\mid h_t)$.
We assume that all active policies share the same support tree: they assign positive
probability to every allowed transition and probability $0$ elsewhere. We denote the finite set of complete responses on this tree by $\support_x$.

Let $\ptil_w(y\mid x)$ denote the unnormalized weighted geometric average of the policies, $\ptil_w(y\mid x)=\prod_{i\in I_w}\pi_i(y\mid x)^{w_i}$, and let $Z_w(x)=\sum_{y\in\support_x}\ptil_w(y\mid x)$ be its normalizer. With this notation, the target policy~\eqref{eq:intended} can be written as $\widehat p_w=\ptil_w/Z_w$.
Our goal is to sample from $\widehat p_w$ without computing $Z_w$, which
sums over exponentially many responses. Local product decoding avoids
this summation by normalizing at each token:
\begin{equation}\label{def:q}
    g_{w,t}(v\mid h_t)=\prod_{i\in I_w}\pi_{i,t}(v\mid h_t)^{w_i},
    \quad z_{w,t}(h_t)=\sum_{v\in\vocab}g_{w,t}(v\mid h_t),
    \quad q_{w,t}=g_{w,t}/z_{w,t}.
\end{equation}
Repeating this rule gives the \emph{local product distribution}, $q_w(y\mid x)=\prod_tq_{w,t}(y_t\mid h_t)$. It generally differs from $\widehat p_w$, as the following example illustrates.

\begin{example}\label{ex:local}
    \textup{
    In Figure~\ref{fig:overview}, we compose two policies with equal weights ($w_1=w_2=1/2$).
    Both policies assign probability $3/8$ to the first token $j=0$ and $5/8$ to $j=1$.
    After $j=0$, both are uniform over the second token $s\in\{-1,+1\}$. After $j=1$, each assigns probability $0.9$
    to a different value of $s$. The continuation product mass is
    therefore $1$ after $j=0$ and $2\sqrt{0.9\times0.1}=0.6$ after $j=1$.
    Since $(3/8)\times1=(5/8)\times0.6$, the sequence target is uniform over
    the four responses. The local rule ignores these continuation masses
    when choosing the first token and therefore assigns too much probability, $5/8$ instead of $1/2$, to the right-hand ($j=1$) branch.
    }
\end{example}

\begin{lemma}[Local-normalizer identity]\label{lem:tilt}
On the shared support $\support_x$, $q_w$ and $\widehat p_w$ are positive, with
\begin{equation}\label{eq:tilt}
    \widehat p_w(y\mid x)=\frac{q_w(y\mid x)\rho_w(y\mid x)}{Z_w(x)},
    \qquad \rho_w(y\mid x)=\prod_{t=1}^Tz_{w,t}(h_t),
    \qquad Z_w(x)=\E_{q_w}[\rho_w].
\end{equation}
Moreover, $0<z_{w,t}\le1$ for $t=1,\ldots,T$, $0<\rho_w\le1$, and $0<Z_w\le1$.
\end{lemma}
\begin{proof}
See Appendix~\ref{app:proof-tilt}.
\end{proof}

The leftmost identity in~\eqref{eq:tilt} is the local-to-global importance correction used
in sequential Monte Carlo (SMC), specialized to policy products \citep{lew2023smc,chan2026ensembling}.
The normalizers $z_{w,t}$ measure how much the policies disagree at the next token
\citep{heskes1997opinion}:
\[
    -\log z_{w,t}(h)=\sum_iw_i\KL\!\left(q_{w,t}(\cdot\mid h)
    \,\middle\|\,\pi_{i,t}(\cdot\mid h)\right).
\]
Thus $-\log\rho_w$ accumulates disagreement along a response, and
$q_w/\widehat p_w=Z_w/\rho_w$ increases with that accumulated disagreement. Consequently, $q_w$ and $\widehat p_w$ differ whenever $\rho_w$ varies across responses, which is generally the case.

We consider a correction to the local product distribution $q_w$ based on independence Metropolis--Hastings (MH) \citep{gareev2024local}. From the current response $y$, the algorithm proposes a fresh $y'\sim q_w$ and accepts it with probability
$\min\{1,\rho_w(y'\mid x)/\rho_w(y\mid x)\}$. The unknown $Z_w$ cancels, so this decision rule uses only the recorded local normalizers. MH always accepts a rollout with a higher $\rho_w$,
but it also accepts some rollouts with a lower $\rho_w$. In Example~\ref{ex:local}, MH always accepts
a move from the disagreement branch to the agreement branch and accepts
the reverse with probability $0.6$, with the first update moving the disagreement
branch's mass from $5/8$ to $17/32$, closer to the target's $1/2$.
Algorithm~\ref{alg:mh} implements the correction in log space using
$\lrho=\log\rho_w=\sum_t\log z_{w,t}$. The correction evaluates every active policy at each token of $M$
responses, then selects one in $O(M)$ scalar operations.

\begin{algorithm}[t]
\caption{Independence-MH correction of local product decoding}\label{alg:mh}
\begin{algorithmic}[1]
\Require Prompt $x$, weights $w$, total rollout budget $M\ge1$
\State Generate $M$ independent rollouts $\widehat y_1,\ldots,\widehat y_M\sim q_w(\cdot\mid x)$, recording $\lrho$ for each
\State $y\gets\widehat y_1$
\For{$j=2,\ldots,M$}
    \State Draw $U\sim\mathrm{Uniform}(0,1)$ independently
    \If{$\log U\le\min\{0,\lrho(\widehat y_j\mid x)-\lrho(y\mid x)\}$}
        \State $y\gets\widehat y_j$
    \EndIf
\EndFor
\State \Return $y$
\end{algorithmic}
\end{algorithm}

\paragraph{Optimization interpretation.}
For a response distribution $\mu$ on $\support_x$, the
\emph{consensus objective} measures agreement through the negative
weighted sum of KL divergences to the policies:
\begin{equation}\label{eq:consensus}
    \widetilde J_x(\mu):=-\sum_{i\in I_w}w_i\KL(\mu\|\pi_i)
    =\E_\mu[\log\ptil_w-\log\mu]
    =\log Z_w-\KL(\mu\|\widehat p_w).
\end{equation}
The product $\widehat p_w$, also known as the logarithmic opinion pool
\citep{heskes1997opinion}, uniquely maximizes this objective, regardless of
how the policies were trained. For external rewards $\reward_i(x,y)\in\mathbb R$ and a positive reference policy
$\piref$ on $\support_x$, consider the KL-regularized objectives
\begin{equation}\label{eq:problem}
    J_{i,x}^{\rm ext}(\mu)=\E_\mu[\reward_i]
    -\beta\KL(\mu\|\piref),\qquad \beta>0.
\end{equation}
Let $\pi_i^*$ maximize $J_{i,x}^{\rm ext}$ over response distributions,
and let $\pi_w^*$ maximize the corresponding weighted objective
$J_{w,x}^{\rm ext}=\sum_iw_iJ_{i,x}^{\rm ext}$, which
is~\eqref{eq:problem} with reward $\reward_w=\sum_iw_i\reward_i$. The individual optima have the form
$\pi_i^*\propto\piref\exp(\reward_i/\beta)$
\citep{ziegler2019finetuning,rafailov2023dpo}. If the supplied policies
satisfy $\pi_i=\pi_i^*$ for this common reference and $\beta$, then
\begin{equation}\label{eq:product}
    \widehat p_w=\pi_w^*
    \propto\prod_i\bigl(\piref e^{\reward_i/\beta}\bigr)^{w_i}
    =\piref e^{\reward_w/\beta}.
\end{equation}
Sampling from the product thus recovers the scalarized optimum without retraining. For approximately optimal specialists, the product can differ from $\pi_w^*$ (Appendix~\ref{app:approx}).

In either setting, the gap to the optimal objective is
\begin{equation}\label{eq:objective-gap}
    J(p)-J(\mu)=\lambda\KL(\mu\|p),\qquad p=\widehat p_w,
\end{equation}
with $J=\widetilde J_x$ and $\lambda=1$ for the consensus objective, and with $J=J_{w,x}^{\rm ext}$ and $\lambda=\beta$ for multi-reward alignment. In both cases, reducing KL to the product
improves the objective.

\section{Theoretical analysis of the MH correction}\label{sec:theory}

We fix the prompt and weights and abbreviate the target, proposal, and
recorded weight as $p=\widehat p_w$, $q=q_w$, and $\rho=\rho_w$, so
$p=q\rho/Z$ with $Z=\E_q\rho$. More generally, an \emph{instance} $(q,\rho)$ consists of a finite response space, a strictly positive proposal law $q$ on that space, and a strictly positive weight function $\rho$; these determine $Z=\E_q\rho$ and the target $p=q\rho/Z$. Recall that, given $M$ iid candidates $Y_1,\ldots,Y_M$ from $q$, MH starts at the first and processes the remaining $M-1$ according to Algorithm~\ref{alg:mh}. We denote its output distribution by $H_M$. Ordinary SIR uses the same pool of $M$ rollouts but instead selects candidate $j$ with probability $\rho(Y_j)/\sum_k\rho(Y_k)$. We denote SIR's output distribution by $S_M$.

\subsection{Dominance of MH over SIR at every candidate budget}

We consider the $f$-divergences
\begin{equation}\label{eq:f-divergence}
    D_f(\mu\|p)=\sum_y p(y)f(\mu(y)/p(y)),
\end{equation} 
where
$f:(0,\infty)\to\mathbb R$ is finite and convex, with $f(1)=0$.
These include KL in both directions, total variation, and the chi-square divergence.

\begin{theorem}[Finite-budget dominance over ordinary SIR]\label{thm:dominance}
For every integer $M\ge1$ and every $f$-divergence $D_f$ of the form~\eqref{eq:f-divergence},
\begin{equation}\label{eq:dominance}
    D_f(H_M\|p)\le D_f(S_M\|p).
\end{equation}
The inequality~\eqref{eq:dominance} is strict for $M\ge2$ and $p\neq q$ if $D_f$ is total variation or if $f$ is strictly convex.
\end{theorem}
\begin{proof}[Proof idea]
    A priority-sampling representation of MH lets us couple the MH and SIR outputs so that the importance weight selected by MH is almost surely at least as large as that selected by SIR. Relative to the target, MH's probability ratio is nonincreasing in the weight. Convexity then converts these two orderings into the $f$-divergence comparison. The full proof is in Appendix~\ref{app:proof-dominance}.
\end{proof}

Because Theorem~\ref{thm:dominance} covers $\KL(\cdot\|p)$, the objective-gap identity~\eqref{eq:objective-gap} also gives $J(H_M)\ge J(S_M)$ for either objective under its stated assumptions.

Since the proof does not rely on autoregressive structure, Theorem~\ref{thm:dominance} extends directly to reward-tilted targets, where SIR corresponds to soft best-of-$M$ \citep{verdun2025soft}.

\subsection{Objective gain guarantee}\label{subsec:objective-gain-guarantee}

We now derive a lower bound on MH's improvement over the local decoder $q$. For iid $Y_1,\ldots,Y_M\sim q$, define the recorded-weight gap
\begin{equation}\label{eq:pool-gap}
    G_M=\log\!\left(\frac1M\sum_{j=1}^M\rho(Y_j)\right)
    -\frac1M\sum_{j=1}^M\log\rho(Y_j),
\end{equation}
the log ratio of the arithmetic and geometric means of the pool's weights, so that $G_M\ge0$ and $G_1=0$. Note that this bound can be estimated from recorded rollout weights, without evaluating $Z$ or the sampler's output distribution.

\begin{proposition}[Objective improvement and diminishing returns]\label{prop:monotone}
    In the setting of Theorem~\ref{thm:dominance}, take
    $(J,\lambda)=(\widetilde J_x,1)$, or
    $(J,\lambda)=(J_{w,x}^{\rm ext},\beta)$ if every active policy
    exactly optimizes~\eqref{eq:problem} with a common positive reference
    and coefficient $\beta>0$.
    Then the gains $J(H_{M+1})-J(H_M)$ are nonnegative and nonincreasing
    for $M\ge1$, and
    \begin{equation}\label{eq:certificate}
        J(H_M)-J(q)\ge\lambda\E_{q^M}[G_M]\ge0
        \qquad\text{for all }M\ge1.
    \end{equation}
  
\end{proposition}
\begin{proof}
See Appendix~\ref{app:proof-certificate}.
\end{proof}

We refer to $\lambda\E_{q^M}[G_M]$ as a \emph{certificate} of improvement over $q$. The same lower bound holds for SIR, so it does not measure the difference between MH and SIR. The certificate bounds improvement, not the remaining divergence (Appendix~\ref{app:certificate-barrier}).

\subsection{Nearby component policies}

The near-agreement regime describes the composition of specialists that remain close to a common policy, as may occur under regularized post-training. Since local decoding is exact when the policies coincide, this regime isolates how effectively additional rollouts correct the bias introduced by their disagreement. It reveals a qualitative distinction between MH and SIR: \textbf{additional MH rollouts improve the error order}, whereas SIR generally improves only the leading coefficient.

Specifically, fix a common positive policy $\pi_0(\cdot\mid h)$ on a fixed support tree and bounded functions $a_i(v,h)$, and consider the perturbed policies
\begin{equation}\label{eq:nearby-policies}
    \pi_{i,\delta}(v\mid h)=
    \frac{\pi_0(v\mid h)e^{\delta a_i(v,h)}}
    {\E_{V\sim\pi_0(\cdot\mid h)}e^{\delta a_i(V,h)}}.
\end{equation}
The parameter $\delta$ controls the separation of the policies $\pi_{i,\delta}$, and $\delta=0$ corresponds to identical policies. We denote the resulting local proposal and product target by
$q_\delta$ and $p_\delta$, respectively, and the corresponding MH and SIR output laws by $H_{M,\delta}$ and $S_{M,\delta}$.

\begin{theorem}[Exact nearby-policy error]\label{thm:nearby}
    For the family of policies~\eqref{eq:nearby-policies}, there exist constants $C_M\ge0$ such that, for every $M\ge 1$, as $\delta\to0$,
    \begin{equation}\label{eq:nearby}
    \begin{aligned}
        \KL(q_\delta\|p_\delta)&=C_1\delta^4+O(|\delta|^5),\\
        \KL(H_{M,\delta}\|p_\delta)&=C_M\delta^{4M}+O(|\delta|^{4M+1}),\\
        \KL(S_{M,\delta}\|p_\delta)&=\frac{C_1}{M^2}\delta^4+O(|\delta|^5).
    \end{aligned}
    \end{equation}
    Moreover, the constants $C_M$ are either all positive or all $0$.
\end{theorem}
\begin{proof}
    See Appendix~\ref{app:proof-nearby}.
\end{proof}

The exponent $4M$ cannot be improved uniformly over nearby-policy families by selectors using only $q$ and $M$ sampled candidates with their recorded weights, without access to $Z$ or unobserved weights (Appendix~\ref{app:selector-lower}).

\subsection{Candidate budget and response length}\label{subsec:horizon}

As responses become longer, small differences between the target and proposal can accumulate across tokens. Consequently, the ratio of their probabilities can vary greatly from one response to another. Responses with a large target-to-proposal ratio are underrepresented among the generated candidates. This poses a challenge for selection rules such as MH and SIR, because neither can return a response that was never proposed.

To study this limitation, we consider a sequence of problems in which fluctuations in the log target-to-proposal ratio grow along with the rollout budget. The following theorem describes the balance between these two effects. At the specified budget scaling, some target mass remains inaccessible often enough that even the best selection rule has a positive limiting error.

\begin{theorem}[Critical candidate budget]\label{thm:critical}
    For each integer $N\ge1$, let $p_N,q_N$ be positive probability
    laws on the same finite set $\mathcal X_N$. Suppose that there exist
    $d_N\in\mathbb R$ and $s_N>0$ with $s_N\to\infty$ such that, for $Y_N\sim p_N$,
    \[
        \frac{\log[p_N(Y_N)/q_N(Y_N)]-d_N}{s_N}
        \Rightarrow W
    \]
    in distribution as $N\to\infty$, where $W$ has distribution function $\Psi$.
    Consider integer rollout budgets $M_N\to\infty$ such that
    \[
        \lim_{N\to\infty}\frac{\log M_N-d_N}{s_N}=b
    \]
    for some $b\in\mathbb R$. If $\Psi$ is continuous at $b$ and
    $0<\Psi(b)<1$, then
    \begin{equation}\label{eq:critical}
        \inf_{\nu\in\mathcal S_{M_N}(q_N)}
        \|\nu-p_N\|_{\rm TV}
        \longrightarrow 1-\Psi(b)
        \quad\textup{as }N\to\infty,
    \end{equation}
    where $\mathcal S_M(q_N)$ is the set of output laws of all rules
    that return one of $M$ iid candidates from $q_N$, including rules
    with full knowledge of $p_N$ and $q_N$. Both MH, initialized from $q_N$,
    and ordinary SIR attain this limit.
\end{theorem}
\begin{proof}
    See Appendix~\ref{app:horizon-budget}.
\end{proof}

As an example, consider a response consisting of $N$ independent repetitions of Example~\ref{ex:local}. In this family, $d_N=\KL(p_N\|q_N)$ grows linearly with response length and $s_N$ grows as its square root (Appendix~\ref{app:horizon-budget}). For any fixed $b\in\mathbb R$, the scaling $\log M_N=d_N+bs_N+o(s_N)$ yields a limiting TV error strictly between zero and one. Thus these budgets satisfy $\log M_N=\KL(p_N\|q_N)+O(\sqrt N)$ and grow exponentially with response length.

\begin{figure}[t]
\centering
\includegraphics[width=0.90\textwidth]{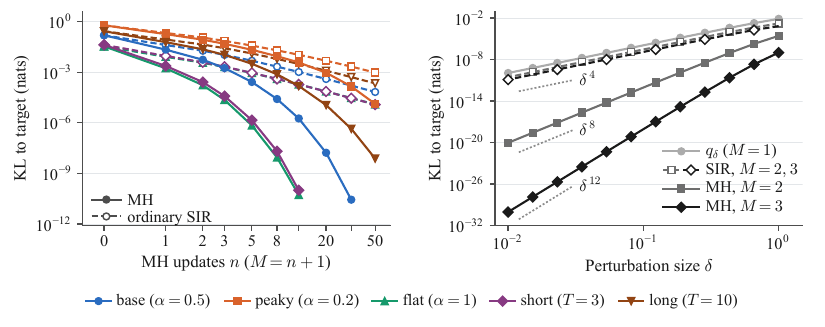}
\caption{\textbf{Exact sampling error on enumerable problems.}
Left: mean KL in five enumerable settings (values $\le10^{-12}$ omitted).
Right: the nearby-policy rates of Theorem~\ref{thm:nearby}. As $\delta\to0$, the KL of MH decays as $\delta^{4M}$, whereas ordinary SIR keeps the $\delta^4$ order of $q_\delta$ (dotted reference slopes).
Solid lines with filled markers show MH, and dashed lines with open markers show ordinary SIR.}
\label{fig:toy}
\end{figure}

\section{Empirical study}\label{sec:experiments}

We measure MH and SIR sampling errors directly on enumerable problems and restricted continuation trees of trained LLM specialists.
For full responses, the samplers' marginal output probabilities are difficult to evaluate. We instead use the shared certificate to bound improvement over local decoding and separately assess task-level effects and serving costs on GSM8K and HH-RLHF. Both experiments compose DPO-specialized adapters of Gemma-4-E2B-it \citep{gemmateam2026gemma4}.

\subsection{Sampling error in an enumerable setting}\label{subsec:enumerable}

We compose 3 autoregressive policies over a 4-token vocabulary, drawing
conditional distributions from a symmetric Dirichlet law and weights
uniformly from the simplex. Five settings vary policy concentration
(base, peaky, and flat) and response length (short and long), with horizons from
3 to 10 tokens. Enumeration lets us propagate MH exactly and evaluate SIR's
marginal law by deterministic quadrature. Appendix~\ref{app:toy-table} gives
configurations and replication counts.

The first MH update gives the largest gain in every setting, removing approximately 70--95\% of the mean initial KL (Figure~\ref{fig:toy}, left), consistent with Proposition~\ref{prop:monotone}. Convergence is slowest in the peaky
setting. At $M=9$, MH has lower mean KL than ordinary SIR in all five settings. In the base, flat, and short settings, MH with 9 candidates also has lower mean KL than SIR with 51.

At a budget of 9 full-sequence equivalents of conditional policy evaluations, MH has lower estimated binned TV than SIR on all 23 tested instances and the lowest point estimate among MH, SMC, and replica-exchange methods on 17. Adaptive SMC outperforms MH on one peaky instance. Maximum-weight selection can be farther from the target than local decoding (Appendix~\ref{app:matched-protocol}).

Finally, the right panel of Figure~\ref{fig:toy} tests the nearby-policy prediction on a fixed
27-response problem with 2 equally weighted policies. As $\delta$
decreases, MH at $M=2$ and $M=3$ approaches eighth- and twelfth-order
error, while ordinary SIR retains fourth-order error with a smaller
coefficient.

\subsection{Direct comparisons on trained-policy continuation trees}\label{subsec:llm-tree}

We compare MH and SIR using the trained accuracy- and brevity-specialized policies on 50 GSM8K test questions \citep{cobbe2021gsm8k}, at the original prompts and halfway through greedy reference-model responses. We enumerate continuations of up to 4 tokens, restricting both specialists to the reference model's top 8 tokens at each prefix and renormalizing each separately. The target is the sequence product of these restricted policies. At accuracy weights $w_1\in\{0.25,0.5,0.75\}$, deterministic MH propagation and SIR quadrature give the full output laws and hence their consensus gains through~\eqref{eq:consensus}.

At midpoint contexts with equal weights, the local decoder's mean KL is $1.31\times10^{-3}$ nats. With 2 candidates, MH removes 96.8\% of this mean divergence, compared with 76.0\% for SIR. With 3 candidates, MH has lower mean KL than SIR with 9 candidates.

Across the six context--weight settings, MH's mean KL advantage over SIR at $M=9$ ranges from $4.60\times10^{-6}$ to $1.66\times10^{-4}$ nats. All six paired, pointwise 95\% bootstrap intervals exclude $0$, although the gaps are small and vary substantially across questions. At this budget, the estimated mean certificate is 75--90\% of MH's mean gain over local decoding. Full curves and sensitivity checks are in Appendix~\ref{app:llm-tree}.

\Needspace{10\baselineskip}
\subsection{Consensus gains and task-level effects on GSM8K}\label{subsec:llm}

\begin{figure}[t]
\centering
\includegraphics[width=\linewidth]{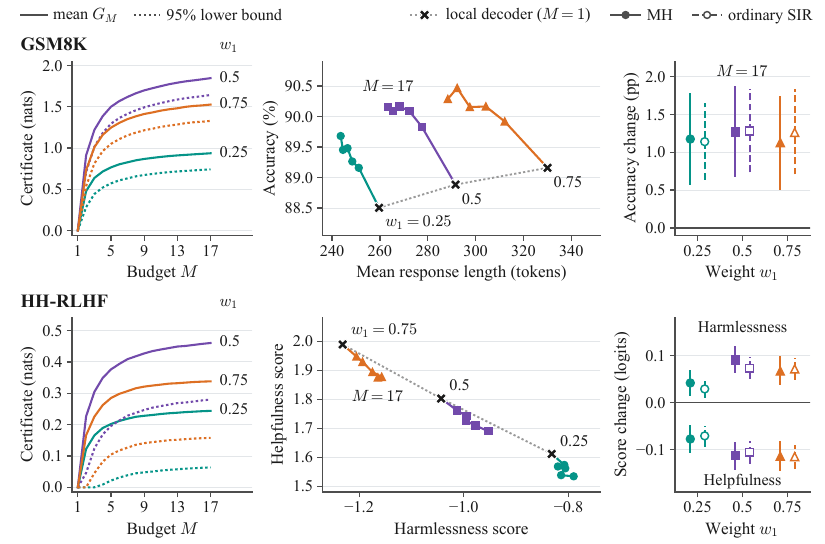}
\caption{\textbf{Certified consensus gains and task-level effects on GSM8K (top) and HH-RLHF (bottom).}
Left: mean $G_M$ and simultaneous 95\% lower bounds within each study, conditional on the prompts and checkpoints; $G_M$ depends only on the candidate pool, so it certifies both MH and ordinary SIR.
Middle: MH trade-offs at $M\in\{2,3,5,9,17\}$; crosses mark local decoding ($M=1$). Ordinary SIR nearly coincides with MH (Tables~\ref{tab:gsm8k-output} and~\ref{tab:hh-output}).
Right: changes from the local decoder at the same weights for $M=17$, with paired, pointwise 95\% bootstrap intervals. HH-RLHF shows both helpfulness and harmlessness changes in reward-model logits; pp denotes percentage points.}
\label{fig:llm-results}
\end{figure}

In both GSM8K and HH-RLHF experiments we train two specialist policies and then compose them with weights $(w_1,1-w_1)$, $w_1\in\{0.25,0.5,0.75\}$, on 1,319 test prompts. We generate 8 independent groups (replications) of 17 rollouts per test prompt, sampling from the full vocabulary at temperature 1. At budget $M$, both methods use the first $M$ responses from the same pool. The checkpoints were fixed before the evaluation, and test examples were excluded from specialist training. 

For GSM8K, we compose the accuracy- and brevity-specialized policies on the full official test set, with weight $w_1$ on accuracy.

At $M=17$, the simultaneous 95\% lower bounds on average consensus improvement for both MH and SIR are \textbf{0.742}, \textbf{1.644}, and \textbf{1.330 nats}, in increasing order of $w_1$ (Figure~\ref{fig:llm-results}, top-left). These also lower-bound the local decoder's mean divergence from the product target (Appendix~\ref{app:gsm8k-protocol}).

At every weighting, both corrections produce shorter and more accurate responses than the local decoder with the same weights (top-middle panel). With $M=3$, MH raises accuracy by 0.8--1.2 percentage points and reduces mean returned length by 11--26 tokens. At $M=17$, these changes are 1.1--1.3 points and 16--42 tokens (top-right panel). For both MH and SIR, the paired 95\% intervals for accuracy improvement over $q$ exclude zero at every tested weight and budget $M\ge2$.

MH and SIR differ by at most 0.14 percentage points in accuracy, with intervals that include $0$, and MH responses are shorter by at most 1.3 tokens. Majority voting on the same pools is 1.2--1.6 percentage points more accurate than MH at $M=17$, but returns responses 12--36 tokens longer (Table~\ref{tab:gsm8k-output}). These task-level comparisons do not measure fidelity to the product target.

In a separate serving profile of 16 prompts on one H200, batching 3 candidates per prompt increases mean latency by 16--20\% over local decoding across the three weightings. Despite shorter returned responses, total generated tokens grow approximately in proportion to $M$. This profile does not measure serving throughput. Appendix~\ref{app:output-evaluation} gives the full results and serving protocol.

\Needspace{6\baselineskip}
\subsection{Consensus gains on HH-RLHF}\label{subsec:hh}

We compose helpfulness- and harmlessness-specialized adapters on held-out HH-RLHF conversations \citep{bai2022hh}, with weight $w_1$ on helpfulness. Both adapters train on a shared conversation pool using preferences from separate GPT-2-large reward models \citep{yang2024ric}. These models also evaluate the outputs, so their scores are proxies for quality and safety, not human judgments.

At $M=17$, the simultaneous 95\% lower bounds on average consensus improvement are \textbf{0.063}, \textbf{0.280}, and \textbf{0.158 nats}, in increasing order of $w_1$ (Figure~\ref{fig:llm-results}, bottom-left).

Increasing $w_1$ raises helpfulness and lowers harmlessness under both corrections (bottom-middle panel and Table~\ref{tab:hh-output}). At fixed weights, MH
with $M=17$ lowers helpfulness by 0.077--0.114 logits and raises harmlessness
by 0.042--0.090 logits relative to local decoding. All six paired, pointwise
95\% intervals exclude $0$ (bottom-right panel). The weighted-score change
is negative at $w_1=0.5$ and $0.75$ and statistically unresolved at $w_1=0.25$.
MH and SIR have similar weighted scores, and all three paired intervals
for their differences include $0$.

MH favors responses with less accumulated policy disagreement, measured by $-\log\rho_w$. Even with exact specialists, the guarantee concerns a KL-regularized objective, not the raw weighted score. Appendix~\ref{app:hh-protocol} gives the protocol, paired comparisons, and serving costs.

\subsection{Scope and practical implications}\label{subsec:llm-validation-limits}

The full-response experiments cover one model family and two tasks. Their certificates bound improvement over local decoding, not the MH--SIR gap, which we measure directly only on restricted continuation trees.
At the same candidate budget and model-evaluation cost, Theorem 1 guarantees that MH's output distribution is no farther from the product target than ordinary SIR's.

\section{Related work}\label{sec:literature}

\paragraph{Decoding-time policy composition.}
MOD derives the weighted sequence product for KL-regularized multi-objective
alignment and implements greedy or beam decoding \citep{shi2024mod}.
MAHALO combines objective-specific heads through weighted logits
\citep{shen2026mahalo}. Its untruncated softmax follows our local product
law, although process-reward selection and truncation change the output
distribution. We study sampling at fixed weights, whereas COLLAB selects policies by utility \citep{chakraborty2025collab} and RMOD optimizes against adversarial reward weights \citep{son2026rmod}.

\paragraph{Sampling globally normalized text distributions.}
SMC uses local-normalizer corrections for policy products \citep{lew2023smc}, including across tokenizers \citep{chan2026ensembling}. \citet{gareev2024local} and \citet{tan2024probability} use local proposals in independence MH for globally normalized text. Twisted SMC learns continuation information \citep{zhao2024twisted}, QUEST resamples suffixes \citep{faria2024quest}, and \citet{he2026recipes} study SMC and replica exchange for sequence products and tilts. We analyze independent complete-response proposals.

\paragraph{Finite-budget selection guarantees.}
\citet{deligiannidis2024importance} compare MH and SIR bias rates, while
\citet{yang2023convergence} study worst-start TV rates for multiple-try
and ordinary independence MH at matched cost. Our comparison instead
orders the marginal output distributions under every convex
$f$-divergence at each budget and extends to fixed schedules of their multiple-try
sampler (Appendix~\ref{app:block-dominance}). The coverage analysis connects this
comparison to the candidate-budget thresholds studied in importance
sampling and finite-pool selection
\citep{chatterjee2018sample,block2023sample,flamich2024notes}.

\section{Conclusion}\label{sec:conclusion}

We show that the independence-MH correction brings local product decoding at least as close to the sequence target as ordinary SIR under every convex $f$-divergence, at every candidate budget. Its consensus improvement can be certified from recorded rollout weights without estimating the sampler's output distribution. The GSM8K and HH-RLHF experiments show the magnitude of the positive consensus gains in practical settings. However, comparing correction methods empirically at LLM scale remains difficult: their marginal output probabilities are not readily evaluable, preventing reliable direct estimation of consensus differences, while task rewards alone do not capture the consensus objective. Theoretical comparison results such as given by Theorem~\ref{thm:dominance} are therefore particularly important. Our coverage analysis also shows why improvement need not imply
closeness to the target: neither method can return responses from regions absent from its candidate pool.

\clearpage
\subsection*{AI use statement}
The authors formulated the research questions, derived core theoretical
results, and wrote the initial manuscript. Generative AI tools assisted with readability,
literature searches, figures, LaTeX preparation, and the development,
extension, and checking of selected theoretical arguments and proofs.
The authors take responsibility for the scientific content and reported
results.

\subsection*{Ethics statement}
The method composes existing policies and inherits limitations of their
training data and reward signals. More accurate sampling from their
product does not by itself establish safety or improved external
performance. The HH-RLHF dataset contains potentially harmful content. We
report methodology and aggregate statistics without prompt or response examples.

\subsection*{Reproducibility statement}
Appendix~\ref{app:toy-table} gives the enumerable configurations,
deterministic calculations, matched-budget sampler schedules, and
reference controls. Appendices~\ref{app:gsm8k-protocol} and~\ref{app:hh-protocol} document the
GSM8K and HH-RLHF evaluations and their uncertainty calculations.
Appendix~\ref{app:llm-tree} gives the restricted continuation-tree protocol
and direct objective comparisons.
The associated code is available in the repository at
\url{https://github.com/akurennoy/mh-policy-composition.git}.

\bibliography{references}
\bibliographystyle{preprint}

\clearpage
\appendix
\section{Proofs of the main results}\label{app:proofs}

This appendix proves the results stated in the main text, in the order
in which they appear. We first establish how local normalizers provide
exact correction weights and 
then use the priority representation to compare MH with ordinary SIR.
The subsequent proofs establish the objective guarantees of
Proposition~\ref{prop:monotone}, the nearby-policy rates of
Theorem~\ref{thm:nearby}, and the critical-budget limit of
Theorem~\ref{thm:critical}.

Unless stated otherwise, we fix the prompt and weights and take all
sums over the common finite response support. We write $p=\widehat p_w$
for the target, $q=q_w$ for the local decoder, and $\rho=\rho_w$ for
the recorded weight, so $p=q\rho/Z$. Let $K=K_{p,q}$ denote the independence-MH transition kernel, whose transition probabilities are determined by the proposal $q$ and target $p$. Write $H_M=qK_{p,q}^{M-1}$ for the law of the MH correction that uses $M$ candidates (one for initialization and $n=M-1$ for the updates).

We write $\E_\mu f=\sum_y\mu(y)f(y)$ for expectation.
For KL divergence, we use
$\KL(\mu\|\nu)=\sum_y\mu(y)\log[\mu(y)/\nu(y)]$, with
terms with $\mu(y)=0$ set to $0$ and the divergence set to $+\infty$ if
$\mu(y)>0=\nu(y)$ for some $y$. For positive laws, the convex divergence is
$D_f(\mu\|p)=\sum_y p(y)f(\mu(y)/p(y))$.
The MH, SIR, and target laws in Theorem~\ref{thm:dominance} are all positive,
so both KL directions fall within this definition.
Choosing $f(u)=(u-1)^2$ gives $\chi^2(\mu\|p)$.

\subsection{Local normalizers and the sequence-level target}\label{app:proof-tilt}

\begin{proof}[Proof of Lemma~\ref{lem:tilt}]
Fix $x,w$. On the common support tree, every allowed transition has positive probability under every active policy. Thus $g_{w,t}$ is positive on allowed transitions and $0$ elsewhere. Each reachable prefix has an allowed continuation, so $z_{w,t}>0$ and the local decoder assigns positive probability to every response in $\support_x$. The generalized H\"older inequality gives
\[
    z_{w,t}(h)=\sum_v\prod_{i\in I_w}\pi_{i,t}(v\mid h)^{w_i}
    \le\prod_{i\in I_w}\left(\sum_v\pi_{i,t}(v\mid h)\right)^{w_i}=1.
\]
This includes the single-active-policy case, where equality holds. Consequently $0<\rho_w\le1$. Multiplying the local probabilities along a supported response gives
\begin{align*}
    q_w(y\mid x)
    &=\frac{\prod_{t=1}^T\prod_{i\in I_w}\pi_{i,t}(y_t\mid h_t)^{w_i}}
    {\prod_{t=1}^Tz_{w,t}(h_t)}\\
    &=\frac{\prod_{i\in I_w}\pi_i(y\mid x)^{w_i}}{\rho_w(y\mid x)}
    =\frac{\ptil_w(y\mid x)}{\rho_w(y\mid x)}.
\end{align*}
Summing $q_w\rho_w=\ptil_w$ over $\support_x$ yields $Z_w=\E_{q_w}\rho_w$. Since $q_w$ is a probability law on this support and $0<\rho_w\le1$, we obtain $0<Z_w\le1$. Dividing $q_w\rho_w=\ptil_w$ by $Z_w$ proves~\eqref{eq:tilt} and positivity of $\widehat p_w$.

If a response ends before the cap, every policy emits EOS deterministically at each padded transition. Thus $z_{w,t}(h_t)=1$, and $g_{w,t}(v\mid h_t)=q_{w,t}(v\mid h_t)$ is $1$ for $v=\mathrm{EOS}$ and $0$ otherwise. Padding therefore changes neither the response probability nor its recorded weight.
\end{proof}

\paragraph{The target's next-token conditional.}
For an allowed token $v$ at prefix $h_t$, factor each policy into its prefix, current-token, and suffix probabilities. The prefix factors cancel when normalizing over $v$, leaving
\begin{equation}\label{eq:conditional}
\begin{aligned}
    \widehat p_{w,t}(v\mid h_t)&\propto g_{w,t}(v\mid h_t)F_{w,t}(h_t,v),\\
    F_{w,t}(h_t,v)&=\sum_{y_{>t}}\prod_{i\in I_w}\pi_i(y_{>t}\mid h_t,v)^{w_i}.
\end{aligned}
\end{equation}
Here $\pi_i(y_{>t}\mid h_t,v)$ is the suffix probability, and the sum runs over supported suffixes. At $t=T$ the suffix is empty, so $F_{w,t}=1$. Earlier in the response, this continuation mass generally varies across tokens. The local rule omits $F_{w,t}$, which explains why its next-token probabilities can differ from the target's. For two exact specialists, $-\beta\log F_{w,t}(h_t,v)$ is the finite-horizon counterpart of the divergence correction of \citet[Theorem~3.2]{hunt2019composing}, with the shared reference absorbed into each reward.

\paragraph{Computing local normalizers from logits.}
Let $b_i(v,h)$ be the component logits after the temperature and mask
transformations that define $\pi_i$, and write
$\operatorname{LSE}_v b(v)=\log\sum_v e^{b(v)}$.
On allowed tokens, the normalized log probabilities are
$b_i(v,h)-\operatorname{LSE}_v b_i(v,h)$. Substituting into~\eqref{def:q}
gives
\begin{equation}\label{eq:logit-normalizer}
    \log z_w(h)=\operatorname{LSE}_v\!\left(\sum_{i\in I_w}w_i b_i(v,h)\right)
    -\sum_{i\in I_w}w_i\operatorname{LSE}_v b_i(v,h).
\end{equation}
Adding any prefix-dependent constant to a component's logits leaves
both its policy and this expression unchanged. The subtraction of
component log normalizers is essential: softmax of the combined logits
samples the correct $q_w$, but the normalizer of those logits alone is
not $z_w$.

Generation and weight calculation must use the same temperatures, masks,
EOS rules, and length caps. Temperature and masking enter the target
through the resulting normalized component policies. For example, a token
can be forbidden by masking it
in every active component and renormalizing each component before
computing $z_w$. The correction then targets the product of these masked
policies, which generally differs from the original sequence target
conditioned on avoiding that token. Applying top-$k$ or nucleus truncation
only to the combined proposal generally removes target support and
invalidates the original correction.

Subtracting a common constant from all candidate log weights leaves MH,
SIR, and $G_M$ unchanged and can improve numerical stability. Clipping
those weights or normalizing them by response length generally changes
the target. Clipping $G_M$ for the certificate instead leaves the sampler
unchanged.

\subsection{Priority representation and proof of Theorem~\ref{thm:dominance}}
\label{app:proof-dominance}

The following lemma represents the output of the MH correction on iid candidates using the classical one-item priority-sampling rule \citep{duffield2007priority}.

\begin{lemma}[Priority representation]\label{lem:priority}
Let $Y_1,\ldots,Y_M$ be iid candidates from $q$, and let $U_1,\ldots,U_M$ be iid uniform on $(0,1)$, independent of the candidates. 

Then the maximum of $\rho(Y_j)/U_j$ over $1\le j\le M$ is almost surely attained at a unique index and, denoting this index by
\[
    J=\argmax_{1\le j\le M}\frac{\rho(Y_j)}{U_j},
\]
we have that $Y_J$ has law $H_M=qK^{M-1}$, i.e., the output law of independence MH initialized at $Y_1$ and using $Y_2,\ldots,Y_M$ as successive proposals.
\end{lemma}
\begin{proof}
For $M=1$ the statement is trivial so in the remainder of the proof we assume $M \ge 2$.

Conditional on the candidates, the weights $\rho(Y_j)$ are fixed and positive. The priorities $\rho(Y_j)/U_j$ are therefore independent and continuously distributed, so each pair ties with probability zero. A union bound over the finitely many pairs shows that the maximizing index $J$ is almost surely unique.

An independent uniform permutation of iid candidates leaves their joint law unchanged. It therefore suffices to show that, for any fixed candidate pool, priority selection and MH processing the candidates in a uniformly random order return the same distribution over candidates.

Fix a pool and give its candidates distinct tags $1,\ldots,M$, even if their responses coincide. Write $r_i=\rho(Y_i)>0$ and $R_*=\max_jr_j$. We will show that both rules select tag $i$ with probability
\begin{equation}\label{eq:priority-pool}
    r_i\int_0^{1/R_*}\prod_{j\ne i}(1-r_jt)\,dt.
\end{equation}

For the randomly ordered MH rule, initialize at the first candidate and process each remaining candidate once with fresh independent acceptance randomness.

For priority selection, condition on $U_i=u$. Tag $i$ wins exactly when $U_j\ge r_ju/r_i$ for every $j\ne i$, an event of probability $\prod_{j\ne i}(1-r_ju/r_i)_+$. Its effective integration range is $0<u<\min\left\{1,r_i/\max_{j\ne i}r_j\right\}
  =r_i/R_*$. Substituting $t=u/r_i$ gives~\eqref{eq:priority-pool}.

We prove the same formula for random-order MH by induction on the pool size. For each nonempty tagged set $A$, let $R_A=\max_{j\in A}r_j$ and let $P_A(i)$ be the probability that the chain finishes at tag $i$. The induction claim is
\[
    P_A(i)=r_i\int_0^{1/R_A}\prod_{k\in A\setminus\{i\}}(1-r_kt)\,dt.
\]
It holds immediately for $|A|=1$. For $m=|A|\ge2$, condition on the uniformly chosen last tag:
\[
 mP_A(i)=\sum_{j\ne i}P_{A\setminus\{i\}}(j)\min\{1,r_i/r_j\}
       +\sum_{j\ne i}P_{A\setminus\{j\}}(i)(1-r_j/r_i)_+.
\]
The first sum covers the case in which $i$ is processed last and accepted. In the second, another tag $j$ is processed last, so the chain must already be at $i$ and reject $j$.

For fixed $A,i$, abbreviate
\[
 Q_i(t)=\prod_{k\ne i}(1-r_kt),\qquad
 B_i(t)=\sum_{j\ne i}\prod_{k\ne i,j}(1-r_kt),
\]
with all indices in $A$. Product differentiation gives
\[
 B_i-mQ_i=-(tQ_i)',\qquad
 Q_i'=-\sum_{j\ne i}r_j\prod_{k\ne i,j}(1-r_kt).
\]
We separate the two cases that determine the induction integrals' upper limits.

If $r_i<R_A$, removing $i$ leaves a maximal-weight tag. The second sum has nonzero terms only when $r_j<r_i$, so those removals also leave a maximal tag. Every nonzero induction integral therefore has the common endpoint $1/R_A$. For each $j$, the coefficients combine as
\[
 r_j\min\{1,r_i/r_j\}+r_i(1-r_j/r_i)_+
 =\min\{r_i,r_j\}+(r_i-r_j)_+=r_i.
\]
Consequently,
\[
 mP_A(i)=r_i\int_0^{1/R_A}B_i(t)\,dt
        =mr_i\int_0^{1/R_A}Q_i(t)\,dt.
\]
The last equality uses $B_i-mQ_i=-(tQ_i)'$ and $Q_i(1/R_A)=0$, since $r_i < R_A$ and hence some $k\in A\setminus\{i\}$ satisfies $r_k=R_A$, making the factor $1-r_kt$ in $Q_i(t)$ vanish at $t=1/R_A$.

If $r_i=R_A$, the first sum in the recursion is $1$: when $i$ is last, it is accepted from every preceding state. In every set $A\setminus\{j\}$, tag $i$ remains maximal. Substituting the induction hypothesis into the second sum gives
\[
 mP_A(i)=1+\int_0^{1/R_A}\bigl(R_AB_i(t)+Q_i'(t)\bigr)\,dt.
\]
Using the same polynomial identity, the endpoint terms are
\[
 -R_A[tQ_i(t)]_0^{1/R_A}+[Q_i(t)]_0^{1/R_A}=-1.
\]
They cancel the leading $1$, leaving $mP_A(i)=mR_A\int_0^{1/R_A}Q_i(t)\,dt$. This also covers tied maxima because no step assumes that weights are distinct. Thus, for every fixed pool, priority selection and MH with a uniformly random candidate order have the same output law. It therefore follows that the unconditional laws that arise when the two rules are applied to $M$ iid candidates drawn from $q$ are the same. Since the permutation step does not change the joint distribution of the $M$ candidates we conclude that $Y_J$ has law $qK^{M-1}$, completing the proof.
\end{proof}

The priority representation, in particular, gives a way to combine independent batches. If each batch keeps its winning response together with its key $r/U$, returning the batch winner with the largest key reproduces
selection from the combined pool. This requires neither $Z$ nor new model
evaluations.

\begin{proof}[Proof of Theorem~\ref{thm:dominance}]
For $M=1$, both selectors return their sole candidate, so $H_1=S_1=q$.

\paragraph{Step 1: Reduce responses to their recorded weights.}
Let $Q$ be the law of $R=\rho(Y)$ for $Y\sim q$, with finite support $\mathcal R$,
and let $\Pi(r)$ be its law under $Y\sim p$.
Lemma~\ref{lem:tilt} implies that
\[
  \Pi(r)
  =\sum_{y:\rho(y)=r}p(y)
  =\sum_{y:\rho(y)=r}\frac{q(y)\rho(y)}{Z}
  =\frac rZ\sum_{y:\rho(y)=r}q(y)
  =\frac{rQ(r)}Z
  \]
Let $\mathcal S = \mathcal S_x$ stand for the common support of $q$ and $p$.

Write $R_j=\rho(Y_j)$ and $\mathbf R=(R_1,\ldots,R_M)$, and let $J$ denote the index selected by either method. Conditional on $\mathbf R=\mathbf r$, the candidates are independent, with $Y_j\sim q(\cdot\mid R=r_j)$. Both MH and SIR choose $J$ using only the weights and independent randomness, so conditioning additionally on $J=j$ does not change this conditional law. In particular, every candidate whose weight is $r$ has the same conditional response law $q(\cdot\mid R=r)$, regardless of its probability of being selected. For any $y$ with $\rho(y)=r$, it follows that
\[
\begin{aligned}
 \Pr(Y_J=y,R_J=r\mid\mathbf R=\mathbf r)
 &=\sum_{j:r_j=r}\Pr(J=j\mid\mathbf R=\mathbf r)\,\Pr(Y_j=y\mid R_j=r)\\
 &=q(y\mid R=r)\Pr(R_J=r\mid\mathbf R=\mathbf r).
\end{aligned}
\]
Averaging over $\mathbf R$ and conditioning on $R_J=r$ gives
$\Pr(Y_J=y\mid R_J=r)=q(y\mid R=r)$. Thus selection changes the probabilities of the weight levels but not the relative probabilities of responses within a weight level. Furthermore, note that after conditioning on $R$ the target and the proposal coincide:
\[
\begin{aligned}
 p(y\mid R=r)
 &=\frac{p(y)\mathbf 1\{\rho(y)=r\}}{\Pi(r)}
 =\frac{q(y)r\,\mathbf 1\{\rho(y)=r\}/Z}{rQ(r)/Z}\\
 &=\frac{q(y)\mathbf 1\{\rho(y)=r\}}{Q(r)}
 =q(y\mid R=r)\qquad\forall\,y\in\mathcal S\quad\forall\,r\in\mathcal R.
\end{aligned}
\]
Thus, for either MH's or SIR's output law $\mu\in\{H_M,S_M\}$ and its selected-weight law $\mathsf U$,
\begin{equation}\label{eq:ratios}
    \frac{\mu(y)}{p(y)}
    =\frac{\mathsf U(\rho(y))q(y\mid R=\rho(y))}{\Pi(\rho(y))q(y\mid R=\rho(y))}
    =\frac{\mathsf U(\rho(y))}{\Pi(\rho(y))}\qquad\forall\,y\in\mathcal S.
\end{equation}
In addition, for either law $\mu\in\{H_M,S_M\}$ we have that
\begin{equation}\label{eq:positivity}
    \mu(y) \ge \Pr\{Y_1=y,\,\ldots,\,Y_M=y\} > 0\qquad\forall\,y\in\mathcal S.
\end{equation}
Using \eqref{eq:ratios} and \eqref{eq:positivity} we derive
\begin{equation}\label{app:weight-reduction}
    \begin{aligned}
        D_f(\mu\|p)
        &=\sum_{y\in\mathcal S} p(y)f\left(\frac{\mu(y)}{p(y)}\right)\\
        &=\sum_{r\in\mathcal R}\sum_{y:\rho(y)=r}
        p(y)f\left(\frac{\mathsf U(r)}{\Pi(r)}\right)\\
        &=\sum_{r\in\mathcal R}\Pi(r)f\left(\frac{\mathsf U(r)}{\Pi(r)}\right)\\
        &=D_f(\mathsf U\|\Pi).
        \end{aligned}
\end{equation}

\paragraph{Step 2: Couple the selected weights.}
Consider a fixed rollout pool tagged with $r_j=\rho(y_j) > 0$. Draw iid uniforms $U_j$ and put $E_j=-\log(1-U_j)$. The $E_j/r_j$ are independent exponential variables with rates $r_j$, so their minimum occurs at tag $j$ with probability
\[
    \int_0^\infty r_j e^{-r_jt}\prod_{k\ne j}e^{-r_kt}\,dt
    =\frac{r_j}{\sum_kr_k}.
\]
Minimizing $E_j/r_j$ implements ordinary SIR, while minimizing $U_j/r_j$ implements priority selection. By Lemma~\ref{lem:priority}, the latter has MH's output law when the candidates are iid from $q$.

Let $i$ be the priority winner and $j$ the SIR winner. Define $g(u)=-\log(1-u)/u$. If $i\ne j$, uniqueness of the continuous keys gives
\[
 \frac{U_i}{r_i}<\frac{U_j}{r_j},\qquad
 g(U_j)\frac{U_j}{r_j}<g(U_i)\frac{U_i}{r_i}.
\]
These inequalities imply $g(U_i)>g(U_j)$ and therefore, since the function $g$ is strictly increasing on $(0,1)$, we have that $U_i>U_j$ and $r_i/r_j>U_i/U_j>1$. Whenever the winners differ, priority selection returns the larger weight. Averaging over iid candidates and applying Lemma~\ref{lem:priority}, we conclude that the distribution of the importance weight selected by MH stochastically dominates that selected by SIR.\footnote{This step uses only the monotonicity of $u\mapsto-\log(1-u)/u$. The same
argument therefore applies to any race that selects the smallest value of
$\Psi^{-1}(U_j)/r_j$, where $\Psi^{-1}$ is the quantile function of an
atomless positive noise and $\Psi^{-1}(u)/u$ is nondecreasing. Noise with
a nonincreasing density is one example. For such races, priority
selection again picks at least as large a weight.}

\paragraph{Step 3: Convert the weight ordering to a divergence ordering.}
Let $\mathsf H_M$ and $\mathsf S_M$ be the distributions of the weight selected by MH and SIR, respectively. Define $B(u)=\E_Q(1-uR)_+$. Conditional on $Y_1=y$ and $U_1=u$, a competitor fails to beat candidate $1$ with probability $B(u/\rho(y))$. Independence of the competitors and exchangeability of the $M$ tags give
\[
 \frac{H_M(y)}{q(y)}=M\int_0^1 B(u/\rho(y))^{M-1}\,du.
\]
Changing variables and using $\Pi(r)=rQ(r)/Z$, we obtain
\begin{equation}\label{eq:priority-density}
    \frac{H_M(y)}{q(y)}
    =M\rho(y)\int_0^{1/\rho(y)}B(u)^{M-1}\,du,
    \qquad h_M(r):=\frac{\mathsf H_M(r)}{\Pi(r)}
    =MZ\int_0^{1/r}B(u)^{M-1}\,du.
\end{equation}
For $M=1$ the integrand is read as $1$. The upper limit decreases with $r$, so $h_M$ is nonincreasing: relative to the target, finite-budget MH favors smaller weights.

Equation~\eqref{eq:priority-density} also follows from the exact transition formula of \citet{smith1996exact}, restated by \citet[Theorem~4]{wang2022exact}. From a minimum-weight response $y_-$, MH accepts every proposal, so $K(y_-,\cdot)=q$ and $H_M=K^M(y_-,\cdot)$. Substituting this initialization into their formula gives the displayed density. Lemma~\ref{lem:priority} additionally identifies the conditional selection law on a randomly ordered tagged pool.

Let $s_M(r)=\mathsf S_M(r)/\Pi(r)$ and choose $b(r)=f'_+(h_M(r))$. The ratios $h_M(r)$ are positive, and a finite convex function has a finite right derivative at every interior point of its domain. Since that derivative is nondecreasing and $h_M$ is nonincreasing, $b$ is a nonincreasing subgradient. The supporting-line inequality gives
\begin{align}\label{app:divergence-order}
 D_f(\mathsf S_M\|\Pi)-D_f(\mathsf H_M\|\Pi)
 &\ge\sum_r\Pi(r)b(r)(s_M(r)-h_M(r))\nonumber\\
 &=\E_{\mathsf S_M}b-\E_{\mathsf H_M}b\ge0.
\end{align}
The final inequality follows from the stochastic domination established in Step 2: MH selects stochastically larger weights than SIR, so the expectation of the nonincreasing function $b$ is no larger under $\mathsf H_M$ than under $\mathsf S_M$. Equation~\eqref{app:weight-reduction} gives the desired response-law comparison. Taking $f(u)=u\log u$ and using~\eqref{eq:objective-gap} gives $J(H_M)\ge J(S_M)$.

\paragraph{Step 4: Establish strictness.}
Suppose $M\ge2$ and $\rho$ is nonconstant. At any nontrivial cutoff $r_0$, that is, $\min\mathcal R\le r_0<\max\mathcal R$, choose supported weights $l\le r_0<h$. A pool containing both weights has positive probability. We construct an open set of keys on which priority selects $h$ and SIR selects $l$.

Write $R_*$ for the pool's largest weight. Choose $0<v<1/(2R_*)$ and set $U_h=hv$ and $U_l=l(1+\xi)v$. At $\xi=0$ the priority keys $U/r$ tie, while the exponential keys satisfy $g(hv)v>g(lv)v$, where $g(u)=-\log(1-u)/u$ is the strictly increasing function defined in Step 2.
For sufficiently small $\xi>0$, priority selection therefore chooses $h$, while SIR chooses $l$. Taking all other uniforms sufficiently close to $1$ prevents either rule from selecting another tag. The strict inequalities persist on an open neighborhood, which has positive probability. Together with the almost-sure weight ordering, this proves
\[
    \mathsf S_M((0,r_0])>\mathsf H_M((0,r_0])
\]
at every nontrivial cutoff.

In particular $\mathsf S_M\ne\mathsf H_M$, so the supporting-line inequality in~\eqref{app:divergence-order} is strict whenever $f$ is strictly convex. Total variation needs a separate argument. Since $B(u)>0$ for $0\le u<1/\min\mathcal R$, formula~\eqref{eq:priority-density} makes $h_M$ strictly decreasing across distinct supported weights. Together with $\E_\Pi h_M=1$, this implies that $A=\{r:h_M(r)>1\}$ is a nonempty proper lower set. Therefore
\[
 \|\mathsf S_M-\Pi\|_{\rm TV}\ge\mathsf S_M(A)-\Pi(A)
 >\mathsf H_M(A)-\Pi(A)=\|\mathsf H_M-\Pi\|_{\rm TV}.
\]
The scalar reduction~\eqref{app:weight-reduction} transfers both strictness conclusions to the response laws.
\end{proof}

\subsection{Stationarity and convergence in total variation}\label{app:proof-convergence}

Let $\mu_n=\mu_0K^n$ for any initial distribution $\mu_0$ on the common finite support. The classical independence-sampler minorization gives the convergence rate used below \citep{tierney1994markov,mengersen1996rates,liu1996metropolized}.

\begin{proposition}[Convergence]\label{prop:convergence}
The target $p$ is stationary for $K$, and for $n\ge1$,
\begin{equation}\label{eq:geom}
    \|\mu_n-p\|_{\rm TV}\le(1-\varepsilon)^n\|\mu_0-p\|_{\rm TV},
    \qquad \varepsilon=\frac{Z}{\max_y\rho(y)}=\frac{1}{\max_y p(y)/q(y)}\in(0,1].
\end{equation}
\end{proposition}
Here $\|\mu-\nu\|_{\rm TV}=\tfrac12\sum_y|\mu(y)-\nu(y)|$ is the total variation distance.

\begin{proof}
Write $\rho_{\max}=\max_y\rho(y)$. Positivity on the finite support gives $0<Z=\E_q\rho\le\rho_{\max}$, so $\varepsilon\in(0,1]$. For distinct $y,y'$, the MH kernel satisfies
\[
    K(y,y')=q(y')\min\{1,\rho(y')/\rho(y)\},\qquad
    p(y)K(y,y')=\frac{q(y)q(y')}{Z}\min\{\rho(y),\rho(y')\}.
\]
The last expression is symmetric. Detailed balance therefore holds, and summing it over the initial state gives $pK=p$.

For every $y,y'$, including the diagonal, where rejected proposals only add mass,
\begin{equation}\label{eq:minorization}
    K(y,y')\ge q(y')\min\{1,\rho(y')/\rho(y)\}
    \ge\frac{q(y')\rho(y')}{\rho_{\max}}=\varepsilon p(y').
\end{equation}
If $\varepsilon=1$, every row of $K$ equals $p$, and one update reaches the target. Otherwise define the residual Markov kernel
\[
    R(y,y')=\frac{K(y,y')-\varepsilon p(y')}{1-\varepsilon}.
\]
Its entries are nonnegative, its rows sum to $1$, and $pR=p$. Thus each update is a draw from $p$ with probability $\varepsilon$ and an $R$-transition otherwise. The triangle inequality and the row sums of $R$ give
\[
    \sum_{y'}\left|\sum_y(\mu(y)-\nu(y))R(y,y')\right|
    \le\sum_y|\mu(y)-\nu(y)|.
\]
Consequently,
\[
    \|\mu K-p\|_{\rm TV}
    =(1-\varepsilon)\|\mu R-pR\|_{\rm TV}
    \le(1-\varepsilon)\|\mu-p\|_{\rm TV}.
\]
Iteration proves~\eqref{eq:geom}.
\end{proof}

\subsection{Proof of Proposition~\ref{prop:monotone}}
\label{app:proof-certificate}
\label{app:proof-monotone}
\label{app:budget-convexity}

We first establish three auxiliary facts: data processing under a stationary
kernel, a threshold representation of the MH output law, and convexity of
its divergence in the candidate budget. We then combine them with a
recorded-weight bound to prove Proposition~\ref{prop:monotone}.
The objective-gap identity~\eqref{eq:objective-gap} covers both consensus
and, under the exact-specialist assumptions, the external reward objective.
The consensus objective is also the zero-reward case of multi-reference
KL-regularized RLHF \citep{aminian2025multiple}.

\begin{lemma}[Data processing under a stationary kernel]\label{lem:stationary-data-processing}
Let $K$ be the independence-MH kernel with stationary law $p$. For every
positive response law $\mu$ and every convex generator $f$ considered in
Theorem~\ref{thm:dominance},
\[
    D_f(\mu K\|p)\le D_f(\mu\|p).
\]
For KL, the inequality also holds when $\mu$ assigns probability zero to
some responses.
\end{lemma}
\begin{proof}
We establish data processing for the convex generators used throughout the paper. For a positive law $\mu$, put $r(y)=\mu(y)/p(y)$ and let $(Y,Y')$ have the stationary joint law $p(y)K(y,y')$. Since $pK=p$,
\[
    \frac{(\mu K)(y')}{p(y')}
    =\E[r(Y)\mid Y'=y'].
\]
The conditional Jensen inequality therefore gives
\[
    D_f(\mu K\|p)
    \le\E[f(r(Y))]
    =D_f(\mu\|p).
\]
For KL, continuity extends this inequality to initial laws that assign probability $0$ to some responses.
\end{proof}

For the threshold representation, write
\[
    B(u)=\E_q(1-u\rho(Y))_+,
    \qquad F_\rho(t)=p\{\rho(Y)\le t\}.
\]
Here $B$ is the function in~\eqref{eq:priority-density}, and $F_\rho$ is the CDF of $\rho(Y)$ under $Y\sim p$.

\begin{lemma}[Threshold representation]\label{lem:threshold}
Let $Y_1,\ldots,Y_M$ be iid $q$, and let $U_1,\ldots,U_M$ be iid uniform on $(0,1)$, independent of the candidates. Set
\[
    J=\argmax_{j\le M}\frac{\rho(Y_j)}{U_j},
    \qquad \Lambda_M=\max_{j\le M}\frac{\rho(Y_j)}{U_j}.
\]
For $t>0$, $\Pr(\Lambda_M\le t)=B(1/t)^M$. For almost every $t$ in the support of $\Lambda_M$, the conditional law of $Y_J$ given $\Lambda_M=t$ is $p(\cdot\mid\rho\le t)$. In particular,
\begin{equation}\label{eq:threshold}
    H_M(y)=p(y)\,\E\!\left[
    \frac{\mathbf 1\{\rho(y)\le\Lambda_M\}}{F_\rho(\Lambda_M)}\right].
\end{equation}
\end{lemma}
\begin{proof}
A single priority key satisfies
\[
    \Pr\!\left\{\frac{\rho(Y)}{U}\le t\right\}
    =\E_q(1-\rho(Y)/t)_+=B(1/t).
\]
Independence gives the CDF of $\Lambda_M$. Given $Y_1=y$, the first key has density $\rho(y)t^{-2}$ on $t>\rho(y)$. Candidate $1$ is then selected with $\Lambda_M\in dt$ exactly when its key lies in $dt$ and every other key is at most $t$. By exchangeability, for $t>\min_y\rho(y)$,
\[
    \Pr(Y_J=y,\Lambda_M\in dt)
    =M q(y)\rho(y)t^{-2}\mathbf 1\{\rho(y)<t\}
    B(1/t)^{M-1}\,dt.
\]
Summing over $y$ gives the density $MZ F_\rho(t)t^{-2}B(1/t)^{M-1}$ of $\Lambda_M$, apart from the finitely many weight values, which have probability $0$. Dividing the joint density by this marginal density yields
\[
    \Pr(Y_J=y\mid\Lambda_M=t)
    =\frac{p(y)\mathbf 1\{\rho(y)\le t\}}{F_\rho(t)}.
\]
The denominator is positive almost surely. Integrating over $t$ and applying Lemma~\ref{lem:priority} proves~\eqref{eq:threshold}.
\end{proof}

\begin{lemma}[Convexity in the candidate budget]\label{lem:budget-convexity}
Under the assumptions of Theorem~\ref{thm:dominance}, put $d_M=D_f(H_M\|p)$. Then $d_M$ is nonincreasing and
\begin{equation}\label{eq:budget-convexity}
    2d_M\le d_{M-1}+d_{M+1},\qquad M\ge2.
\end{equation}
Moreover, let $\mathsf B\ge1$ be an integer-valued candidate budget independent of all proposals and selection randomness, with $\E\mathsf B\le M$. The marginal output law of MH run with this random budget, $\overline H=\sum_{b\ge1}\Pr(\mathsf B=b)H_b$, satisfies
\begin{equation}\label{eq:random-budget}
    D_f(H_M\|p)\le D_f(\overline H\|p).
\end{equation}
\end{lemma}
The right-hand side is the divergence of the mixture law, after averaging over the budget. The lemma concerns proposal initialization and a pool-independent budget. Stopping rules that depend on the observed candidates are not covered and can yield greater error than the proposal itself. For example, with $q=(1/2,1/2)$ and $\rho=(1/2,1)$, so that $p=(1/3,2/3)$, stopping after the first candidate when its weight is $1/2$ and otherwise applying one MH update returns $(5/8,3/8)$, which is farther from $p$ than $q$.

\begin{proof}
We first compare fixed and random budgets, then obtain discrete convexity from a two-point budget distribution.

\paragraph{Independent budget randomization.}
Let $\Lambda_{\mathsf B}$ be the largest priority key at budget $\mathsf B$. Jensen's inequality for the convex map $b\mapsto B(1/t)^b$, followed by $\E\mathsf B\le M$, gives
\[
    \Pr(\Lambda_{\mathsf B}\le t)
    =\E[B(1/t)^{\mathsf B}]
    \ge B(1/t)^{\E\mathsf B}
    \ge B(1/t)^M
    =\Pr(\Lambda_M\le t).
\]
Hence $\Lambda_{\mathsf B}$ is stochastically smaller than $\Lambda_M$. For any nondecreasing function $g$ on the weight support, the truncated mean
\[
    t\longmapsto\E_p[g(\rho)\mid\rho\le t]
\]
is nondecreasing wherever $F_\rho(t)>0$: enlarging the threshold adds only weights at least as large as those already included. The conditional law in Lemma~\ref{lem:threshold} depends on the threshold but not on the budget. Applying the lemma and averaging over $\mathsf B$ therefore gives
\[
    \E_{\overline H}g(\rho)
    =\E\,\E_p[g(\rho)\mid\rho\le\Lambda_{\mathsf B}]
    \le\E\,\E_p[g(\rho)\mid\rho\le\Lambda_M]
    =\E_{H_M}g(\rho).
\]

Every $H_b$ and their mixture preserve the target's conditional response law at each weight, as in~\eqref{app:weight-reduction}. The density ratio $h_M(r)$ in~\eqref{eq:priority-density} is nonincreasing in $r$, so $\psi(r)=f'_+(h_M(r))$ is also nonincreasing. The supporting-line argument of~\eqref{app:divergence-order} now gives
\[
    D_f(\overline H\|p)-D_f(H_M\|p)
    \ge\E_{\overline H}\psi(\rho)-\E_{H_M}\psi(\rho)
    \ge0.
\]
This proves~\eqref{eq:random-budget} for the marginal mixture law.

\paragraph{Discrete convexity.}
For $M\ge2$, choose $\mathsf B=M-1$ or $M+1$ with equal probabilities. Equation~\eqref{eq:random-budget} and convexity of $D_f$ in its first argument imply
\[
    d_M\le D_f\!\left(\frac{H_{M-1}+H_{M+1}}2\middle\|p\right)
    \le\frac{d_{M-1}+d_{M+1}}2.
\]
Monotonicity follows by applying Lemma~\ref{lem:stationary-data-processing} to $H_{M+1}=H_MK$.
\end{proof}

\begin{proof}[Proof of Proposition~\ref{prop:monotone}]
Fix either objective in~\eqref{eq:objective-gap}, with scale $\lambda>0$.
Write $d_M=\KL(H_M\|p)$.

\paragraph{Nonnegative, diminishing gains.}
Lemma~\ref{lem:stationary-data-processing} gives
$d_{M+1}=\KL(H_MK\|pK)\le d_M$.
Lemma~\ref{lem:budget-convexity} also gives
$2d_{M+1}\le d_M+d_{M+2}$. Hence
\[
 0\le J(H_{M+2})-J(H_{M+1})
 =\lambda(d_{M+1}-d_{M+2})
 \le\lambda(d_M-d_{M+1})
 =J(H_{M+1})-J(H_M).
\]

\paragraph{Recorded-weight lower bound.}
We use the augmented-space construction of \citet{domke2018importance}. Let $\mathcal Q$ be the joint law of a proposal pool and its SIR-selected index. Compare it with $\mathcal P$, which chooses an index uniformly, draws that candidate from $p$, and draws the others independently from $q$:
\begin{align*}
  \mathcal Q(y_{1:M},i)
  &=\left(\prod_{j=1}^M q(y_j)\right)
    \frac{\rho(y_i)}{\sum_{j=1}^M\rho(y_j)}, \\
  \mathcal P(y_{1:M},i)
  &=\frac{1}{M}p(y_i)\prod_{j\ne i}q(y_j).
\end{align*}
Both are probability laws. Under the selected-response map $(y_{1:M},i)\mapsto y_i$, their marginals are $S_M$ and $p$, respectively. The identity $p=q\rho/Z$ gives
\[
  \frac{\mathcal Q(y_{1:M},i)}{\mathcal P(y_{1:M},i)}
  =\frac{MZ}{\sum_{j=1}^M\rho(y_j)}.
\]
This ratio does not depend on the selected index, and the pool marginal of $\mathcal Q$ is $q^M$. Data processing therefore yields
\begin{align}
  \KL(S_M\|p)
  &\le \KL(\mathcal Q\|\mathcal P) \nonumber\\
  &=\log Z-
    \E_{q^M}\left[\log\left(\frac{1}{M}\sum_{j=1}^M\rho(Y_j)\right)\right].
  \label{app:sir-kl-bound}
\end{align}
This is the importance-weighted variational bound in its SIR form \citep{cremer2017reinterpreting,domke2018importance}. Theorem~\ref{thm:dominance} gives the same upper bound for $\KL(H_M\|p)$, while
\begin{equation}
  \KL(q\|p)=\log Z-\E_q\log\rho.
  \label{app:q-kl}
\end{equation}
Subtracting~\eqref{app:sir-kl-bound} from~\eqref{app:q-kl} cancels $\log Z$. By the objective-gap identity,
\begin{align}
  J(H_M)-J(q)
  &\ge \lambda\,\E_{q^M}\left[
    \log\left(\frac{1}{M}\sum_{j=1}^M\rho(Y_j)\right)
    -\frac{1}{M}\sum_{j=1}^M\log\rho(Y_j)
  \right] \nonumber\\
  &=\lambda\,\E_{q^M}G_M.
  \label{app:certificate-bound}
\end{align}
The same lower bound holds for $J(S_M)-J(q)$. Concavity of $\log$ gives $G_M\ge0$, with equality exactly when all weights in the pool are equal. For $M=1$, both selectors return a draw from $q$ and the bound equals $0$.
Finally, take $J=\widetilde J_x$ and $\lambda=1$ for consensus.
Under the exact-specialist assumptions,~\eqref{eq:product} and
\eqref{eq:objective-gap} give the same conclusions with
$J=J_{w,x}^{\rm ext}$ and $\lambda=\beta$.
\end{proof}

\paragraph{Certificate dependence on budget.}
Write $\bar\rho_M=M^{-1}\sum_{j=1}^M\rho(Y_j)$, so that $\E G_M=\E\log\bar\rho_M-\E_q\log\rho$. By the importance-weighted bound of \citet[Theorem~1]{burda2016importance}, $\E\log\bar\rho_M$ is nondecreasing in $M$ and converges to $\log Z$, because $\rho$ is bounded above and away from $0$ on the finite support. Hence $\E_{q^M}G_M$ is nondecreasing in $M$ and converges to $\log Z-\E_q\log\rho=\KL(q\|p)$. The certificate is therefore asymptotically tight: the improvement of either method is also at most $\lambda\KL(q\|p)$.

\paragraph{Confidence bounds for a fixed prompt list.}
Let $x_1,\ldots,x_L$ be fixed prompts, each satisfying the proposition's assumptions, with a common scale $\lambda>0$ in the objective-gap identity. For $A\in\{H,S\}$, define the equally weighted prompt averages of the improvement and of the expected pool gap by
\[
  \overline\Delta_M^A
  =\frac1L\sum_{\ell=1}^L
    \bigl[J_{x_\ell}(A_{M,x_\ell})-J_{x_\ell}(q_{x_\ell})\bigr],
  \qquad
  \overline g_M
  =\frac1L\sum_{\ell=1}^L\E_{q_{x_\ell}^M}G_{M,x_\ell}.
\]
Applying~\eqref{app:certificate-bound} at each prompt gives $\overline\Delta_M^A\ge\lambda\overline g_M\ge0$ for both selectors.

Fix a nonempty finite budget set $\mathcal M\subseteq\{1,2,\ldots\}$, a group count $K_0\ge1$, a clipping cap $\Gamma>0$, and an error level $\alpha\in(0,1)$, all chosen independently of the evaluation pools. At each prompt, generate $K_0$ independent pools of $\max\mathcal M$ proposals from $q_{x_\ell}$, independently across prompts. Evaluate each budget $M\in\mathcal M$ using the first $M$ candidates of each pool and their exact recorded weights. With $N=LK_0$, set
\[
  X_{\ell k}^{(M)}=\min\{G_{M,x_\ell}^{(k)},\Gamma\},
  \qquad
  \widehat g_{M,\Gamma}
  =\frac1N\sum_{\ell=1}^L\sum_{k=1}^{K_0}X_{\ell k}^{(M)}.
\]
For each $M$, these $N$ variables are independent and lie in $[0,\Gamma]$, although their distributions may differ across prompts. Equal group counts and clipping give $\E\widehat g_{M,\Gamma}\le\overline g_M$. Hoeffding's inequality for independent bounded variables implies
\[
  \Pr\!\left(\widehat g_{M,\Gamma}-\overline g_M\ge t\right)
  \le \exp\!\left(-\frac{2Nt^2}{\Gamma^2}\right),
  \qquad t>0.
\]
Use failure probability $\alpha/|\mathcal M|$ for each budget and take a union bound. Combining the resulting event with the pointwise certificate yields, with probability at least $1-\alpha$,
\begin{equation}
  \overline\Delta_M^A
  \ge \lambda\max\!\left\{0,\widehat g_{M,\Gamma}
    -\Gamma\sqrt{\frac{\log(|\mathcal M|/\alpha)}{2N}}\right\},
  \qquad M\in\mathcal M,\quad A\in\{H,S\}.
  \label{eq:certificate-confidence}
\end{equation}

Taking $L=1$ and $\mathcal M=\{M\}$ gives the bound for a single prompt and a single budget. Reusing candidate prefixes across budgets is valid because the union bound requires no cross-budget independence. The same event covers both selectors. Coverage is conditional on the fixed prompts and checkpoints: it concerns the objective improvement of the marginal output laws, averaged over these prompts and over fresh proposal pools, and not a realized response or a newly sampled prompt set.

\subsection{Proof of Theorem~\ref{thm:nearby}}
\label{app:proof-nearby}

We combine a quadratic expansion of local log normalizers with an operator form of the classical independence-MH transition formula \citep{smith1996exact,liu1996metropolized} and the SIR marginal formula \citep{cremer2017reinterpreting}.

\begin{lemma}[Exact MH residual]\label{lem:mh-residual}
For a fixed positive finite-support instance, put $\omega=q/p=Z/\rho$ and define
\[
 (\mathcal L_\omega f)(y)
 =\E_p[(\omega(X)-\omega(y))_+]f(y)
  -\E_p[f(X)(\omega(y)-\omega(X))_+].
\]
Then $H_M/p-1=-\mathcal L_\omega^M\mathbf 1$ for every $M\ge1$.
\end{lemma}
\begin{proof}
Let $Q$ have every row equal to $q$ and set $\Delta=K-Q$. For a signed
measure $\nu$ of total mass $0$, $\nu K=\nu\Delta$, and $\nu\Delta$
also has total mass $0$. Stationarity of $p$ therefore gives
$H_M-p=(q-p)\Delta^{M-1}$. A direct calculation gives
\[
 \Delta(x,y)=-p(y)(\omega(y)-\omega(x))_+\quad(x\ne y),
 \qquad \Delta(x,x)=\E_p[(\omega(X)-\omega(x))_+].
\]
Thus $(pf)\Delta=p\mathcal L_\omega f$. Since $\E_p\omega=1$,
$\mathcal L_\omega\mathbf 1=1-\omega$, and iteration proves the identity.
\end{proof}

\begin{proof}[Proof of Theorem~\ref{thm:nearby}]
Fix the family in Theorem~\ref{thm:nearby}: the finite support tree, horizon $T$, common policy $\pi_0$, weights, perturbations $a_i$, and candidate budget $M$ do not vary with $\delta$. Let $P_0$ be the complete-response law induced by $\pi_0$. All remainders below are uniform over responses on this fixed support. The finite exponential-tilt family satisfies $q_\delta=P_0+O(|\delta|)$ and $p_\delta=P_0+O(|\delta|)$.

\paragraph{Quadratic log correction.}
Write $\bar a(v,h)=\sum_iw_i a_i(v,h)$. The component normalizers cancel from the local proposal, giving
\[
  q_\delta(v\mid h)
  =\frac{\pi_0(v\mid h)e^{\delta\bar a(v,h)}}
  {\E_{V\sim\pi_0(\cdot\mid h)}e^{\delta\bar a(V,h)}}.
\]
At a fixed prefix $h$, define the log moment-generating functions
\[
  A_i(\delta;h)=\log\E_{\pi_0(\cdot\mid h)}e^{\delta a_i},
  \qquad
  A_{\bar a}(\delta;h)=\log\E_{\pi_0(\cdot\mid h)}e^{\delta\bar a}.
\]
The local normalizer therefore satisfies
\[
  \log z_\delta(h)=A_{\bar a}(\delta;h)-\sum_iw_iA_i(\delta;h).
\]
Taylor expansion on the finite support gives
\[
  A_i(\delta;h)=\delta\E a_i
  +\frac{\delta^2}{2}\var(a_i)+O(|\delta|^3),
\]
with the analogous expansion for $A_{\bar a}$. Here expectations and variances are under $\pi_0(\cdot\mid h)$. The linear terms cancel, and the quadratic terms satisfy
\[
  \sum_iw_i\var(a_i)-\var(\bar a)
  =\sum_iw_i\var(a_i-\bar a).
\]
For a complete response $y$ with prefixes $h_t$, set
\begin{equation}\label{eq:disagreement-coefficient}
  D(y)=\sum_{t=1}^T\sum_iw_i
  \var_{V\sim\pi_0(\cdot\mid h_t)}
  \left[a_i(V,h_t)-\bar a(V,h_t)\right],
  \qquad
  C=\frac18\var_{Y\sim P_0}[D(Y)].
\end{equation}
Thus $D(y)$ accumulates the weighted variance of the component perturbations about their average along the response. Deterministic EOS padding contributes $0$ to these variances. Summing the local expansions gives
\begin{equation}
  \ell_\delta(y):=\log\rho_\delta(y)
  =-\frac{\delta^2}{2}D(y)+O(|\delta|^3).
  \label{app:log-rho-expansion}
\end{equation}

\paragraph{Leading proposal bias.}
By Lemma~\ref{lem:tilt}, $p_\delta=q_\delta\rho_\delta/Z_\delta$, where $Z_\delta=\E_{q_\delta}\rho_\delta$. Hence
\begin{align}
  \KL(q_\delta\|p_\delta)
  &=\log\E_{q_\delta}e^{\ell_\delta}
    -\E_{q_\delta}\ell_\delta \nonumber\\
  &=\frac12\var_{q_\delta}(\ell_\delta)+O(|\delta|^6) \nonumber\\
  &=\frac{\delta^4}{8}\var_{P_0}(D)+O(|\delta|^5)
   =C\delta^4+O(|\delta|^5).
  \label{app:q-kl-expansion}
\end{align}
The second line follows from $\|\ell_\delta\|_\infty=O(\delta^2)$: the remainder in the log moment-generating expansion is cubic in $\ell_\delta$. The final remainder also accounts for the cubic term in~\eqref{app:log-rho-expansion} and the change of measure $q_\delta=P_0+O(|\delta|)$. The coefficient $C$ therefore measures variation of $D$ across responses. If $D$ is constant, the quartic term vanishes, although higher-order terms can still produce bias.

\paragraph{MH contraction.}
Write $W_\delta=p_\delta/q_\delta=\rho_\delta/Z_\delta$. The expansion
of $\ell_\delta$ gives $W_\delta=1+O(\delta^2)$ uniformly and
\begin{equation}\label{app:chi-bound}
 \chi^2(q_\delta\|p_\delta)
 =\E_{q_\delta}\frac{(1-W_\delta)^2}{W_\delta}
 =2C\delta^4+O(|\delta|^5).
\end{equation}
Also $\varepsilon_\delta=Z_\delta/\max_y\rho_\delta(y)
\ge e^{-\operatorname{osc}(\ell_\delta)}$, so
$1-\varepsilon_\delta=O(\delta^2)$. The minorization in
Appendix~\ref{app:proof-convergence} writes
$K_\delta=\varepsilon_\delta\mathbf P_\delta+(1-\varepsilon_\delta)R_\delta$,
where every row of $\mathbf P_\delta$ is $p_\delta$ and $p_\delta R_\delta=p_\delta$. Conditional Jensen under the stationary
joint law $p_\delta(y)R_\delta(y,y')$ gives
\begin{equation}\label{app:chi-contraction}
 \chi^2(\mu K_\delta\|p_\delta)
 =(1-\varepsilon_\delta)^2\chi^2(\mu R_\delta\|p_\delta)
 \le(1-\varepsilon_\delta)^2\chi^2(\mu\|p_\delta).
\end{equation}
The same inequality holds for any positive finite-support independence
sampler. If $\varepsilon_\delta=1$, then $K_\delta=\mathbf P_\delta$.
Iteration and $\KL\le\chi^2$ imply
\begin{equation}\label{app:mh-kl-bound}
 \KL(H_{M,\delta}\|p_\delta)
 \le(1-\varepsilon_\delta)^{2(M-1)}\chi^2(q_\delta\|p_\delta)
 =O(|\delta|^{4M}).
\end{equation}

\paragraph{Exact MH coefficient.}
Put $\overline D=\E_{P_0}D$ and define the disagreement operator
\begin{equation}\label{eq:disagreement-operator}
 (\mathcal T f)(y)
 =\E_{P_0}[(D(X)-D(y))_+]f(y)
  -\E_{P_0}[f(X)(D(y)-D(X))_+].
\end{equation}
Its range consists of $P_0$-mean-zero functions, and $\mathcal T\mathbf 1=\overline D-D$.
Normalizing~\eqref{app:log-rho-expansion} gives
$\omega_\delta=q_\delta/p_\delta
=1+\delta^2(D-\overline D)/2+O(|\delta|^3)$.
Since the positive-part function is $1$-Lipschitz and
$p_\delta=P_0+O(|\delta|)$,
\[
 \mathcal L_{\omega_\delta}
 =\frac{\delta^2}{2}\mathcal T+O(|\delta|^3)
\]
as operators on the fixed finite-dimensional function space. Applying
Lemma~\ref{lem:mh-residual} yields the uniform density expansion
\begin{equation}\label{eq:mh-residual-expansion}
 \frac{H_{M,\delta}}{p_\delta}-1
 =\frac{\delta^{2M}}{2^M}\mathcal T^{M-1}(D-\overline D)
  +O(|\delta|^{2M+1}).
\end{equation}
For $v=H_{M,\delta}/p_\delta-1$, we have $\E_{p_\delta}v=0$, so
$\KL(H_{M,\delta}\|p_\delta)=\tfrac12\E_{p_\delta}v^2
+O(\|v\|_\infty^3)$. Consequently,
\begin{equation}\label{app:mh-kl-expansion}
 \KL(H_{M,\delta}\|p_\delta)=C_M\delta^{4M}+O(|\delta|^{4M+1}),
 \qquad
 C_M=\frac{\E_{P_0}[(\mathcal T^{M-1}(D-\overline D))^2]}{2\cdot4^M}.
\end{equation}
In particular $C_1=\var_{P_0}(D)/8=C$. If $D$ is constant, every $C_M$
vanishes. Otherwise, let $D_{\min}=\min_yD(y)$ and $c=(\overline D-D_{\min})/2>0$.
At a response with $D=D_{\min}$, the operator multiplies every function
by $2c$, so $\mathcal T^{M-1}(D-\overline D)=-(2c)^M$ there.
The set $\{D=D_{\min}\}$ has positive $P_0$-mass, proving $C_M>0$.

\paragraph{SIR's leading coefficient.}
Write
\[
  W_\delta=1+u_\delta,
  \qquad \E_{q_\delta}u_\delta=0.
\]
Normalizing the expansion in~\eqref{app:log-rho-expansion} gives
\begin{equation}
  u_\delta(y)
  =-\frac{\delta^2}{2}
    \left(D(y)-\E_{P_0}D\right)+O(|\delta|^3).
  \label{app:sir-weight-expansion}
\end{equation}
Conditioning on one candidate and using exchangeability of the candidate indices yields the exact identity
\begin{equation}
  \frac{S_{M,\delta}(y)}{p_\delta(y)}
  =M\,\E\left[
    \frac{1}{M+u_\delta(y)+\sum_{j=2}^Mu_\delta(Y_j)}
  \right],
  \qquad Y_j\overset{\mathrm{iid}}\sim q_\delta.
  \label{app:sir-density-identity}
\end{equation}
The perturbation of the denominator is uniformly $O(\delta^2)$. In its reciprocal expansion, the terms linear in $u_\delta(Y_j)$ vanish in expectation. The quadratic term is $[u_\delta(y)^2+(M-1)\E_{q_\delta}u_\delta^2]/M^2=O(\delta^4)$, so
\begin{equation}
  v_\delta(y):=\frac{S_{M,\delta}(y)}{p_\delta(y)}-1
  =-\frac{u_\delta(y)}{M}+O(|\delta|^4).
  \label{app:sir-ratio-expansion}
\end{equation}
Because $\E_{p_\delta}v_\delta=0$, the linear term in the expansion of $(1+v_\delta)\log(1+v_\delta)$ also vanishes. Thus
\begin{align}
  \KL(S_{M,\delta}\|p_\delta)
  &=\frac12\E_{p_\delta}v_\delta^2+O(|\delta|^6) \nonumber\\
  &=\frac{1}{2M^2}\E_{p_\delta}u_\delta^2+O(|\delta|^6) \nonumber\\
  &=\frac{\delta^4}{8M^2}\var_{P_0}(D)+O(|\delta|^5) \nonumber\\
  &=\frac{C}{M^2}\delta^4+O(|\delta|^5).
  \label{app:sir-kl-expansion}
\end{align}
The third line uses~\eqref{app:sir-weight-expansion} and $p_\delta=P_0+O(|\delta|)$, accounting again for the changing measure. Together, \eqref{app:q-kl-expansion}, \eqref{app:mh-kl-expansion}, and~\eqref{app:sir-kl-expansion} prove~\eqref{eq:nearby}.
\end{proof}

The operator~\eqref{eq:disagreement-operator} describes how rejections
propagate the leading residual. Its first term retains residual mass
at a response, while the second accounts for rejected moves from responses
with less accumulated disagreement. The common proposal part cancels
on every mass-zero residual. At minimum-disagreement responses, only
the first term remains, which prevents the leading error from vanishing
when $C>0$.

\paragraph{Dependence on the tree.}
The constants in Theorem~\ref{thm:nearby} depend on the tree, and the
expansion is asserted only for a fixed tree. For example, concatenate $N$ independent two-token modules. Both
policies choose $j\in\{0,1\}$ uniformly. Conditional on $j$, their second-token laws
are proportional to $e^{\delta js}$ and $e^{-\delta js}$ for
$s\in\{-1,1\}$. This is an instance of the fixed-score family in
Theorem~\ref{thm:nearby}, with $D=\sum_{k=1}^N j_k$ under the common
uniform policy, so $C=N/32$. Although $C$ grows linearly in $N$, the
remainder is not uniform in $N$, so the expansion does not describe long
responses at a fixed separation $\delta$.
Appendix~\ref{app:horizon-budget} treats growing horizons directly.

\paragraph{A remark on soft best-of-\texorpdfstring{$M$}{M}.}
With weights $\rho=e^{\reward/\lambda}$, ordinary SIR is soft best-of-$M$
\citep{verdun2025soft}. Theorem~2 of \citet{verdun2025soft}
(arXiv:2505.03156v1) states that $\|S_M-p\|_{\rm TV}\ge C'/\sqrt{M-1}$
for a constant $C'>0$ that does not depend on $M$, and concludes that
$\KL(p\|S_M)=\Omega(1/M)$. This lower bound does not hold. Take their
binary example with $q(0)=q(1)=1/2$, $\reward(y)=y$, and
$\lambda=1/\log2$, so that $\rho=(1,2)$ and $p(1)=2/3$. With
$B\sim\operatorname{Binomial}(M,1/2)$, the exact marginal is
$S_M(1)=\E[f(B/M)]$ for $f(t)=2t/(1+t)$. A Taylor expansion at $t=1/2$,
with $f''(1/2)=-32/27$ and $\E(B/M-1/2)^2=1/(4M)$, gives
$S_M(1)=2/3-4/(27M)+O(M^{-2})$. Consequently,
\[
  \KL(p\|S_M)=\frac{4}{81M^2}+O(M^{-3}),\qquad
  \|S_M-p\|_{\rm TV}=\frac{4}{27M}+O(M^{-2}).
\]
Their Appendix~D instead reports a leading KL term $1/(9M)$ for this
example, because a variance term of order $1/M$ in the expansion of
$S_M(1)$ is treated as $O(M^{-2})$. Their Theorem~1 upper bound, here
$\KL(p\|S_M)\le\log(1+1/(9M))$, remains valid.

\subsection{Candidate coverage and proof of Theorem~\ref{thm:critical}}\label{app:horizon-budget}

This subsection proves Theorem~\ref{thm:critical} and specializes it to
independent copies of the running example. The sample-size scale $e^{\KL(p\|q)}$ is
classical in importance sampling \citep{chatterjee2018sample}.
Approximate rejection sampling and channel simulation also relate
finite-pool accuracy to likelihood tails
\citep{block2023sample,flamich2024notes}. Our result identifies the exact
limiting TV error at the threshold and shows that unmodified MH and
ordinary SIR both attain it.

\begin{lemma}[Candidate coverage]\label{lem:coverage}
Let a selection rule return one of $M\ge1$ iid candidates drawn from $q$, and let $\nu$ be its output law. The rule may randomize and may know the full instance. For every set $A$ of responses,
\[
    q(A)^M\le\nu(A)\le1-(1-q(A))^M.
\]
\end{lemma}
\begin{proof}
The output lies in $A$ whenever all $M$ candidates do, and only if at least one does.
\end{proof}

\paragraph{A likelihood-tail converse.}
Put $W=p/q$ and $L=\log W$. Every candidate selector satisfies
$\nu(A)\le Mq(A)$. On $A=\{L\ge\tau\}$,
$q(A)=\E_p[e^{-L}\mathbf 1_A]\le e^{-\tau}$, so
\begin{equation}\label{eq:likelihood-converse}
    \|\nu-p\|_{\rm TV}\ge p\{L\ge\tau\}-Me^{-\tau}.
\end{equation}
In particular, if $\tau$ is a median of $L$ under $p$, the right-hand
side is at least $1/2-Me^{-\tau}$. This coverage bound holds even for
selectors with full instance knowledge.

The converse identifies target mass that the pool is unlikely to reach.
For the matching upper bound, we temporarily cap the largest importance
weights, making convergence to the capped target easier to control.
The next lemma transfers the resulting bound back to MH and SIR with
their original weights.

\begin{lemma}[Clipping and distance to the original target]\label{lem:clipping-order}
Let $W=p/q>0$ on a common finite support. For $B>0$, let $A_M$ and
$A_M^B$ be the output laws of either proposal-initialized MH or ordinary
SIR using weights $W$ and $W_B=\min(W,B)$, respectively. Then
$D_f(A_M\|p)\le D_f(A_M^B\|p)$ for every convex generator
in Theorem~\ref{thm:dominance}.
\end{lemma}
\begin{proof}
Use the same candidates and race variables $R_j$ in both selectors:
uniform variables for MH's priority representation, exponential
variables for SIR. If the original winner $i$ had $W_i<W_j$ for the
clipped winner $j$, their winning inequalities would imply
\[
    \frac{W_j}{W_i}<\frac{R_j}{R_i}
    <\frac{\min(W_j,B)}{\min(W_i,B)}.
\]
This is impossible because $W/\min(W,B)$ is nondecreasing. Continuous
keys are almost surely distinct, including when candidate weights tie.
Thus the original selected weight stochastically dominates the clipped
selected weight.

Both outputs preserve $q(\cdot\mid W)$, and the original ratio
$A_M/p$ is nonincreasing in $W$. For MH this follows from
\eqref{eq:priority-density}. For SIR it follows from
\[
    \frac{S_M(y)}{p(y)}
    =M\E\!\left[\frac{1}{W(y)+\sum_{j=2}^MW(Y_j)}\right],
    \qquad Y_j\overset{\mathrm{iid}}{\sim}q.
\]
Consequently $\psi(W)=f'_+(A_M/p)$ is nonincreasing in $W$.
The supporting-line argument in~\eqref{app:divergence-order} gives
\[
    D_f(A_M^B\|p)-D_f(A_M\|p)
    \ge\E_{A_M^B}\psi(W)-\E_{A_M}\psi(W)\ge0.
\]
\end{proof}

\begin{proposition}[Finite-budget likelihood-tail bounds]\label{prop:smooth-envelope}
For $B>0$, define
\[
    \alpha_B=\E_q\min(W,B),\qquad
    t_B=1-\alpha_B=\E_q(W-B)_+.
\]
Then
\begin{equation}\label{eq:selector-envelope}
    t_M\le\inf_{\nu\in\mathcal S_M(q)}\|\nu-p\|_{\rm TV}
    \le\|H_M-p\|_{\rm TV}
    \le\inf_{B>0}\left\{t_B+(1-\alpha_B/B)^M\right\}.
\end{equation}
The upper bound also holds with the last power replaced by
$e^{-M\alpha_B/B}$. Clipping is used only in the proof. The MH law $H_M$ uses
the original weights.
\end{proposition}
\begin{proof}
Every selector satisfies $\nu(y)\le Mq(y)$. Summing $p-Mq$ over
$\{W>M\}$ gives the lower bound $t_M$.
For the upper bound, set $p_B=qW_B/\alpha_B$. The clipped target obeys
$p_B/q\le B/\alpha_B$, so independence MH with clipped weights has
minorization constant at least $\alpha_B/B$. Since
$q\ge(\alpha_B/B)p_B$, proposal initialization gives
$\|q-p_B\|_{\rm TV}\le1-\alpha_B/B$. Proposition~\ref{prop:convergence}
therefore yields
\[
    \|H_M^B-p_B\|_{\rm TV}\le(1-\alpha_B/B)^M.
\]
Also $p\ge\alpha_Bp_B$, so $\|p-p_B\|_{\rm TV}\le1-\alpha_B=t_B$.
Apply Lemma~\ref{lem:clipping-order}, the triangle inequality, and
$1-u\le e^{-u}$.
\end{proof}

\begin{proof}[Proof of Theorem~\ref{thm:critical}]
Write $W_N=p_N/q_N$, $L_N=\log W_N$, $\ell_N=\log M_N$, and $c=\Psi(b)$.
Set $A_N=\{L_N>\ell_N\}$. The assumed convergence gives
$p_N(A_N)\to1-c$. For any $\epsilon>0$ such that $\Psi$ is
continuous at $b+\epsilon$,
\begin{align*}
    M_Nq_N(A_N)
    &=\E_{p_N}[e^{\ell_N-L_N}\mathbf 1_{A_N}]\\
    &\le p_N\{\ell_N<L_N\le\ell_N+\epsilon s_N\}
       +e^{-\epsilon s_N}.
\end{align*}
Taking the limit superior and then sending $\epsilon\downarrow0$
through continuity points gives $M_Nq_N(A_N)\to0$.
By candidate coverage, every selector's mass on $A_N$ tends to
$0$, uniformly over selection rules. Hence
\[
    \liminf_N\inf_{\nu\in\mathcal S_{M_N}(q_N)}
    \|\nu-p_N\|_{\rm TV}\ge1-c.
\]
Binary KL data processing also gives the general converse
\begin{equation}\label{eq:critical-kl-converse}
    \liminf_N\inf_{\nu\in\mathcal S_{M_N}(q_N)}
    \KL(\nu\|p_N)\ge-\log c.
\end{equation}
Indeed, $\nu(A_N)\to0$ uniformly while $p_N(A_N)\to1-c\in(0,1)$.
We claim attainment only for TV under the theorem's assumptions.

For the upper bound, choose
\[
    k_N=\exp\!\left(\min\{\sqrt{s_N},\tfrac12\log M_N\}\right),
    \qquad B_N=M_N/k_N.
\]
Then $k_N\to\infty$, $k_N=o(M_N)$, and $\log k_N=o(s_N)$.
The preceding tail argument with $\log B_N$ in place of $\ell_N$ gives
\[
    \alpha_N:=\E_{q_N}\min(W_N,B_N)
    =p_N(W_N\le B_N)+B_Nq_N(W_N>B_N)\longrightarrow c.
\]
Proposition~\ref{prop:smooth-envelope} now implies
\[
    \|H_{M_N}-p_N\|_{\rm TV}
    \le1-\alpha_N+e^{-\alpha_Nk_N}\longrightarrow1-c
\]
in the sense of a limit superior.

To handle SIR, let $p_{B_N}=q_N\min(W_N,B_N)/\alpha_N$ and
$U=\min(W_N,B_N)/\alpha_N$. Under $q_N$, $\E U=1$ and
$U\le K_N:=B_N/\alpha_N$, where $K_N\ge1$. The clipped SIR
output then satisfies, by Jensen's inequality,
\[
    \frac{S_{M_N}^{B_N}(y)}{p_{B_N}(y)}
    =M_N\E\!\left[\frac{1}{U(y)+\sum_{j=2}^{M_N}U(Y_j)}\right]
    \ge\frac{M_N}{K_N+M_N-1}.
\]
This density lower bound yields
\[
    \|S_{M_N}^{B_N}-p_{B_N}\|_{\rm TV}
    \le\frac{K_N-1}{K_N+M_N-1}
    \le\frac{1}{\alpha_N k_N}\longrightarrow0.
\]
Since $\|p_N-p_{B_N}\|_{\rm TV}\le1-\alpha_N$,
Lemma~\ref{lem:clipping-order} and the triangle inequality give the
same upper bound $1-c$ on the limit superior for unmodified SIR.
Together with the uniform converse, this proves that both methods
attain the limit in~\eqref{eq:critical}.
\end{proof}

\paragraph{Independent copies of Example~\ref{ex:local}.}
For $\delta>0$, one module has two equally weighted specialists
\[
    \pi_{\pm,\delta}(j,s)
    =\frac{e^{\pm\delta js}}{2(1+\cosh\delta)},
    \qquad (j,s)\in\{0,1\}\times\{-1,1\}.
\]
These are exact optima for rewards $\pm\beta\delta js$ under a
uniform reference. Their sequence product is uniform. Write
\[
    c_\delta=\log\cosh\delta,\qquad
    \gamma_\delta=\frac{e^{c_\delta}}{1+e^{c_\delta}},\qquad
    a_\delta=\frac{2}{1+e^{c_\delta}}.
\]
The local decoder chooses $j=1$ with probability $\gamma_\delta$
and $s$ uniformly, recording weight $e^{-c_\delta j}$.
At $\delta=\log3$, a single module is exactly the example in Figure~\ref{fig:overview}.
For $N$ independent modules, let $N_1=\sum_{k=1}^Nj_k$ be the number of modules with $j_k=1$. Then
\begin{equation}\label{eq:module-laws}
\begin{aligned}
    N_1&\sim\operatorname{Bin}(N,\gamma_\delta)\quad\text{under }q_N,
    &N_1&\sim\operatorname{Bin}(N,1/2)\quad\text{under }p_N,\\
    \rho_N&=e^{-c_\delta N_1},
    &Z_N&=a_\delta^N.
\end{aligned}
\end{equation}
Thus $L_N=-c_\delta N_1-N\log a_\delta$, with
\begin{equation}\label{eq:module-critical-scale}
    d_N=\KL(p_N\|q_N)=N\log\cosh(c_\delta/2),
    \qquad s_N^2=\frac{Nc_\delta^2}{4}.
\end{equation}
For fixed $\delta>0$, the binomial central limit theorem verifies the hypotheses of
Theorem~\ref{thm:critical} with $\Psi=\Phi$, the standard normal CDF.
The same conclusion holds for $\delta=\delta_N\to0$ when
$N\delta_N^4\to\infty$, since then
$d_N\sim N\delta_N^4/32$ and $s_N^2\sim N\delta_N^4/16\to\infty$.
The budget must use the exact center in~\eqref{eq:module-critical-scale}
because its leading approximation need not be accurate to $o(s_N)$.

There is also an elementary finite-horizon bound. Put
$\Delta_\delta=\gamma_\delta-1/2=\tfrac12\tanh^2(\delta/2)$ and
$E_N=\{N_1/N\le1/2+\Delta_\delta/2\}$.
Hoeffding's inequality gives
\[
    p_N(E_N)\ge1-e^{-N\Delta_\delta^2/2},
    \qquad q_N(E_N)\le e^{-N\Delta_\delta^2/2}.
\]
By candidate coverage, every selector of $M$ complete proposals obeys
\begin{equation}\label{eq:module-coverage}
    \|\nu-p_N\|_{\rm TV}
    \ge1-(M+1)e^{-N\Delta_\delta^2/2}.
\end{equation}
At fixed TV accuracy below $1$, this requires a budget exponential in
$N\Delta_\delta^2$, which is of order $N\delta^4$ for small $\delta$.
In contrast, querying the two child normalizers in each module gives
the exact first-token conditional by multiplying the common root
probabilities by those normalizers. This conditional is uniform, as is the second-token
conditional. This sequential sampler uses $O(N)$ conditional policy-vector
evaluations on this structured family. The comparison counts returned
probability vectors, not transformer FLOPs, and does not assert an
efficient exact sampler on arbitrary response trees.

\section{Extensions and limits}\label{app:refinements}

This appendix develops extensions and limits of the main results.
We first extend the comparison of Theorem~\ref{thm:dominance} to
other schedules with the same budget. We then establish optimality guarantees
under different information constraints, show that no rollout-only
statistic can certify a small remaining error,
and treat continuation-guided proposals and approximate specialists.

The optimality results concern rules that return one of $M$ iid complete
candidates from a fixed proposal, so they exclude adaptive proposal
construction and prefix-level search. They differ in the information
available to the selector. The exact minimax bound allows complete
knowledge of the target and proposal, because knowing where target mass
lies cannot supply a response missing from the candidate pool. The
nearby-policy lower bound instead hides unobserved weights and the
normalizer. For weight-only admissibility, the rule uses only recorded
weights and must be shared across instances. These are worst-case or
uniform guarantees, not claims that MH is best on each instance.

\subsection{Multiple-try and i-SIR schedules}\label{app:block-dominance}

The multiple-try comparison uses the recycled-trial independence
sampler of \citet{yang2023convergence}, a specialization of the multiple-try Metropolis method of
\citet{liu2000mtm}. From the current state $y$, draw $Y_{1:k}\sim q^k$, select $J$
with probability proportional to $\rho(Y_J)$, and accept $Y_J$ with probability
\[
 \min\!\left\{1,\frac{\sum_{i=1}^k\rho(Y_i)}
 {\rho(y)+\sum_{i\ne J}\rho(Y_i)}\right\}.
\]
Write $A_k$ for this kernel. No additional reverse trials are drawn.
The iterated SIR (i-SIR) kernel $P_{k+1}$ keeps the current state $y$, adds $k$ fresh $q$ candidates, and
resamples from this pool proportionally to the weights
\citep{andrieu2018uniform}. Both blocks cost $k$ fresh candidates.
The single-proposal multiple-try kernel satisfies $A_1=K$, and
proposal-initialized i-SIR satisfies $qP_{k+1}=S_{k+1}$.

\begin{theorem}[Dominance over block schedules]\label{thm:block-dominance}
Under the assumptions of Theorem~\ref{thm:dominance}, let $\nu_M$ be the output law of a chain that starts
from $q$ and applies a fixed schedule of these blocks, using $M$ candidates
including initialization. Allow i-SIR as the first block and in at most
one later block. Then $D_f(H_M\|p)\le D_f(\nu_M\|p)$ for every
convex generator $f$.
\end{theorem}

The theorem orders the marginal output laws themselves, which goes beyond the worst-case convergence rates of
\citet{yang2023convergence}. The schedule and initialization restrictions
are discussed after the proof.

We compare blocks by the probabilities they assign to recorded-weight
levels. We first compare one $k$-proposal block with $k$ consecutive
MH updates, then replace blocks within the schedule. The order of
replacement matters because i-SIR need not preserve the decreasing
density ratios used in the comparison.

For iid proposal weights $R_i$, put
\[
 \psi_k(t)=kZ\,\E\!\left[\frac1{t+R_1+\cdots+R_{k-1}}\right].
\]
For $y\ne y'$, conditioning on a tagged proposal and summing over its
$k$ possible positions gives
\begin{equation}\label{eq:block-kernels}
 A_k(y,y')=p(y')\psi_k(\rho(y)\vee\rho(y')),\qquad
 P_{k+1}(y,y')=p(y')\psi_k(\rho(y)+\rho(y')).
\end{equation}
For the first identity, selection followed by acceptance contributes
$\rho(y')/[\rho(y)\vee\rho(y')+\sum_{i=1}^{k-1}R_i]$. For the second,
selection contributes $\rho(y')/[\rho(y)+\rho(y')+\sum_iR_i]$.
Both kernels are $p$-reversible. Ordinary SIR satisfies
$S_k(y)/p(y)=\psi_k(\rho(y))$, a nonincreasing function of the weight.

\paragraph{Kernels ordered by weight.}
For laws $A,B$ on the weight space, write $A\preceq_{\rm st}B$ if
$\E_A g\le\E_B g$ for every nondecreasing function $g$. When $A,B$ are
response laws, this notation compares their induced weight laws, meaning
$\E_{Y\sim A}g(\rho(Y))\le\E_{Y\sim B}g(\rho(Y))$ for every such $g$.
Aggregate equal weights into $r_1<\cdots<r_d$, with proposal masses
$Q_i$ and target masses $\Pi_i=r_iQ_i/Z$. All blocks preserve the
conditional response law within each level, so the reduction~\eqref{app:weight-reduction} applies. Let $\mathcal C$ be the set of laws
whose density with respect to $\Pi$ is nonincreasing in the level. If $B\in\mathcal C$
is positive and $b_i=B_i/\Pi_i$, independence MH with proposal $B$ and
target $\Pi$ has off-diagonal entries
\[
 K_B(i,j)=\Pi_j\min\{b_i,b_j\}=\Pi_j b_{\max\{i,j\}}.
\]
Let $F_l=\sum_{i\le l}\Pi_i$, and let $E_l$ be the kernel that redraws a state in the first
$l$ levels from the target conditioned on those levels and leaves higher levels fixed.
Then
\begin{equation}\label{eq:nested-refresh}
 K_B=\sum_{l=1}^{d-1}F_l(b_l-b_{l+1})E_l+b_dE_d.
\end{equation}
The coefficients are nonnegative and sum to $\sum_i\Pi_i b_i=1$.
Telescoping verifies the off-diagonal entries. Each $E_l$ preserves
stochastic order and maps decreasing functions to decreasing functions.
Consequently $K_B$ is stochastically monotone and, by reversibility,
preserves $\mathcal C$.

Moreover $E_iE_j=E_{\max\{i,j\}}$, so products of these kernels again
have this form. The first row of $K_B$ is $B$. Every mixture of the
$E_l$ is determined by its first row, because its off-diagonal ratio at
$(i,j)$ is the first row's density at $\max\{i,j\}$. It follows that
\begin{equation}\label{eq:mh-block-composition}
 K_{B_1}K_{B_2}=K_{B_1K_{B_2}},\qquad K^k=K_{H_k}.
\end{equation}
These identities also follow from the exact transition formula of
\citet{smith1996exact}. Equation~\eqref{eq:block-kernels} identifies
$A_k=K_{S_k}$.

\begin{lemma}[Monotone-proposal comparison]\label{lem:monotone-proposal}
If $B,B'\in\mathcal C$ and $B\preceq_{\rm st}B'$, then
$\nu K_B\preceq_{\rm st}\nu K_{B'}$ for every $\nu\in\mathcal C$.
\end{lemma}
\begin{proof}
For a reversible kernel $L$ and $g_i=\nu_i/\Pi_i$, net flow across a
cut satisfies
\[
 (\nu L)(\{k,\ldots,d\})-\nu(\{k,\ldots,d\})
 =\sum_{i<k\le j}\Pi_iL(i,j)(g_i-g_j).
\]
Using the max-type form, the difference between the two kernels is
\[
 \sum_{j\ge k}(B'_j-B_j)c_j,\qquad
 c_j=\sum_{i<k}\Pi_i(g_i-g_j).
\]
Because $g$ is nonincreasing, $c_j\ge0$ and $c_j$ is nondecreasing in
$j$. Writing $T_j=\sum_{l\ge j}(B'_l-B_l)\ge0$, summation by parts
gives $c_kT_k+\sum_{j>k}(c_j-c_{j-1})T_j\ge0$.
\end{proof}

\paragraph{One-block comparison.}
The coupling in the proof of Theorem~\ref{thm:dominance} gives
$S_k\preceq_{\rm st}H_k$, and both laws lie in $\mathcal C$.
Lemma~\ref{lem:monotone-proposal} and~\eqref{eq:mh-block-composition}
therefore imply $\nu A_k\preceq_{\rm st}\nu K^k$ for
$\nu\in\mathcal C$. Also $P_{k+1}(i,j)\le A_k(i,j)$ off the diagonal,
since $\psi_k$ is decreasing. The reversible flow identity then yields
\begin{equation}\label{eq:one-block-order}
 \nu P_{k+1}\preceq_{\rm st}\nu A_k
 \preceq_{\rm st}\nu K^k,\qquad \nu\in\mathcal C.
\end{equation}

\begin{proof}[Proof of Theorem~\ref{thm:block-dominance}]
The proposal lies in $\mathcal C$, as does an initial i-SIR output
$qP_{k+1}=S_{k+1}$. Multiple-try kernels preserve this class.
If an i-SIR block occurs later in the schedule, its input thus lies in $\mathcal C$.
Replace it by the same number of MH updates using~\eqref{eq:one-block-order}.
Every subsequent block is multiple-try and stochastically monotone, so
this replacement can only increase the final output weight in stochastic order.
If the first block is i-SIR, replace it using
$S_{k+1}\preceq_{\rm st}H_{k+1}$. The remaining suffix is now monotone.
Finally replace each multiple-try block by single-proposal MH updates,
from left to right. The input at each replacement remains in
$\mathcal C$, and the suffix preserves the comparison. Thus
$\nu_M\preceq_{\rm st}H_M$ in weight. Since $H_M/p$ is nonincreasing
in weight, the supporting-line argument~\eqref{app:divergence-order}
proves the divergence inequality, strictly for strictly convex $f$
whenever the output laws differ.
\end{proof}

\subsection{Exact minimax total variation}\label{app:minimax-tv}

Selecting one of $M$ iid proposals to approximate a target is also studied as approximate rejection sampling \citep{block2023sample}. Here a uniform bound on the target-to-proposal density ratio gives an exact minimax TV value. Competing selectors may know the entire target and proposal, so the lower bound isolates the limitation imposed by candidate coverage. MH reaches this smallest achievable worst-case error using only the proposed responses and their weights.

For $\varepsilon\in(0,1]$, let $\mathcal C_\varepsilon$ be the class of positive $q,\rho$ on arbitrary common finite supports with $p=q\rho/Z$ and $\kappa:=Z/\max_y\rho(y)\ge\varepsilon$. Let $\mathcal A_M^{\rm full}$ contain all randomized rules that, given the instance $(q,\rho)$ and $M$ iid $q$ candidates, return one of those candidates. Write $\nu_{\mathsf A;q,\rho}$ for the output law of such a rule.

\begin{proposition}[Exact minimax total variation]\label{prop:minimax-tv}
For every integer $M\ge1$ and $\varepsilon\in(0,1]$,
\begin{equation}\label{eq:minimax-tv}
 \inf_{\mathsf A\in\mathcal A_M^{\rm full}}
 \sup_{(q,\rho)\in\mathcal C_\varepsilon}
 \|\nu_{\mathsf A;q,\rho}-p\|_{\rm TV}
 =\sup_{(q,\rho)\in\mathcal C_\varepsilon}\|H_M-p\|_{\rm TV}
 =(1-\varepsilon)^M.
\end{equation}
Thus proposal-initialized MH attains the minimax value without knowing $\varepsilon$, $Z$, or $\max_y\rho(y)$. The same conclusions hold when the instance class is restricted to products of two equally weighted specialists on two-token trees.
\end{proposition}
\begin{proof}
Pointwise, $q\ge\kappa p$, so the initial TV error satisfies
\[
    \|q-p\|_{\rm TV}
    =1-\sum_y\min\{q(y),p(y)\}\le1-\kappa.
\]
By Proposition~\ref{prop:convergence}, the remaining $M-1$ updates contract this error:
\[
 \|H_M-p\|_{\rm TV}
 \le(1-\kappa)^{M-1}\|q-p\|_{\rm TV}
 \le(1-\kappa)^M\le(1-\varepsilon)^M.
\]
For $M=1$, only the initial-distance bound is needed. When $\varepsilon=1$, $p=q$ and MH is exact at every budget.

For a matching lower bound when $\varepsilon<1$, take a set $A$ with $q(A)=\varepsilon$ and weights $1$ on $A$ and $c\in(0,1)$ on its complement. Then $\kappa=Z=\varepsilon+(1-\varepsilon)c\ge\varepsilon$. By Lemma~\ref{lem:coverage}, every candidate selector satisfies
\[
 \|\nu_{\mathsf A;q,\rho}-p\|_{\rm TV}
 \ge p(A)-\nu_{\mathsf A;q,\rho}(A)
 \ge\frac{\varepsilon}{\varepsilon+(1-\varepsilon)c}
       -1+(1-\varepsilon)^M.
\]
As $c\downarrow0$, the target mass of $A$ tends to $1$, but the candidate pool misses $A$ with probability $(1-\varepsilon)^M$. This proves the lower bound uniformly over selection rules. Every instance with $c>0$ has positive weights. The limit is used only to approach the supremum.

To realize the family within policy composition, let the first token be $j\in\{\mathrm{hi},\mathrm{lo}\}$ and the second be a binary label. Both specialists choose $\mathrm{hi}$ with probability $\varepsilon$. After $\mathrm{hi}$, both label distributions are uniform. After $\mathrm{lo}$, use $(b,1-b)$ and $(1-b,b)$ with $b\in(0,1/2)$ chosen so that $2\sqrt{b(1-b)}=c$. The root normalizer is $1$, the local label proposal is uniform in both branches, and the recorded weight is $1$ after $\mathrm{hi}$ and $c$ after $\mathrm{lo}$. Thus $A=\{j=\mathrm{hi}\}$ has the required law on a shared four-response support.
\end{proof}

The same coverage obstruction prevents a uniform bound on the forward KL in the objective gap. Every selector puts mass at least $u=(1-\varepsilon)^M$ on $A^c$, whereas $v_c=p(A^c)\to0$. Binary data processing and the entropy bound give
\[
 \KL(\nu_{\mathsf A;q,\rho}\|p)
 \ge u\log(1/v_c)-\log2\longrightarrow\infty.
\]
Consequently, for $\varepsilon<1$, no finite candidate budget gives a uniform forward-KL bound over $\mathcal C_\varepsilon$, even with full target knowledge.

\subsection{The error exponent with recorded-candidate information}\label{app:selector-lower}

Theorem~\ref{thm:nearby} shows that MH attains the exponent $4M$ whenever the disagreement coefficient $D$ is nonconstant. We now ask whether another selector can improve this exponent uniformly when it knows the proposal and sees the sampled responses and their recorded weights. The following construction hides the weights of unobserved responses. Averaging over the hidden instances restricts every selector to a low-degree polynomial, from which a finite-difference argument gives the lower bound.

Fix $M\ge1$ and put $L=4M$. Let $q$ be uniform on responses $(j,s)\in\{1,\ldots,L\}\times\{-1,1\}$. For a hidden subset $A\subseteq\{1,\ldots,L\}$ and $\delta\ne0$, define
\[
 b=\operatorname{sech}\delta,\qquad
 \rho_{A,\delta}(j,s)=\begin{cases}1,&j\in A,\\b,&j\notin A,\end{cases}
 \qquad p_{A,\delta}=q\rho_{A,\delta}/Z_{A,\delta}.
\]
A selector $\mathsf A$ returns one of $M$ iid $q$ candidates. It may use their full identities, their raw recorded weights, the complete proposal law $q$, and independent randomness. Its rule may depend on $M$, $\delta$, and the construction, but must be the same for every hidden subset $A$. It receives no $Z_{A,\delta}$, unobserved weights, extra samples, or off-candidate model queries. Let $\nu_{\mathsf A,A,\delta}$ be its output law, and let $\mathcal A_M$ denote these rules.

\begin{proposition}[Optimal worst-case exponent for recorded-candidate selectors]\label{prop:selector-lower}
With $c_M=M!/[4(4M)^M]$,
\begin{equation}\label{eq:selector-minimax}
 2c_M^2(1-b)^{2M}
 \le\inf_{\mathsf A\in\mathcal A_M}\sup_A
       \KL(\nu_{\mathsf A,A,\delta}\|p_{A,\delta})
 \le\sup_A\KL(H_{M,A,\delta}\|p_{A,\delta})
 \le\frac{(1-b)^{2M}}{4b}.
\end{equation}
For each fixed $M$, the family is realized by two equally weighted nearby policies on a fixed two-token tree. Since $1-b=\delta^2/2+O(\delta^4)$, the minimax error is $\Theta_M(|\delta|^{4M})$ as $\delta\to0$.
\end{proposition}
\begin{proof}
Both specialists sample $j$ uniformly. After $j\in A$, both sample $s$ uniformly. After $j\notin A$, their sign probabilities are
\[
 \pi_{1,\delta}(s\mid j)=\frac{e^{\delta s}}{2\cosh\delta},\qquad
 \pi_{2,\delta}(s\mid j)=\frac{e^{-\delta s}}{2\cosh\delta}.
\]
Their local sign product is uniform, with normalizer $b$ outside $A$ and $1$ inside $A$. The root normalizer is $1$. These policies are fixed bounded perturbations of the common uniform policy, as in Theorem~\ref{thm:nearby}. The support and horizon do not change with $\delta$. They are also exact optima for reference $q$ and rewards $\reward_i=\beta\log(\pi_{i,\delta}/q)$.

\paragraph{A polynomial restriction on every selector.}
Write $\mathcal B_A=\{(j,s):j\in A\}$ and average the selector's output mass on this set over all subsets of size $k$:
\[
 m_{\mathsf A}(k)=\binom Lk^{-1}\sum_{|A|=k}\nu_{\mathsf A,A,\delta}(\mathcal B_A).
\]
Condition on the complete candidate tuple, and suppose it contains $d\le M$ distinct first-token IDs. A particular high/low labeling of those IDs, with $a$ high labels, occurs under a uniformly chosen subset $A$ of size $k$ with probability
\[
 \frac{\binom{L-d}{k-a}}{\binom Lk}
 =\frac{(k)_a(L-k)_{d-a}}{(L)_d},\qquad
 (v)_a=v(v-1)\cdots(v-a+1),\quad (v)_0=1.
\]
This labeling probability is a polynomial in $k$ of degree at most $d$,
valid at every integer $0\le k\le L$. The candidate tuple itself has law
$q^{\otimes M}$, independently of $A$. Once the tuple is fixed, the
selector learns about $A$ only through the observed labels. Its
conditional probability of selecting a high weight therefore depends
on the tuple and labels, but not otherwise on $A$.

Averaging this selection probability against the labeling probabilities
makes $m_{\mathsf A}(k)$ a polynomial of degree at most $M$. Since
$m_{\mathsf A}(0)=0$, we can write $m_{\mathsf A}(k)=tP(t)$ with $t=k/L$
and $\deg P\le M-1$. The same polynomial must describe the common
selector across all hidden instances. Its coefficients may depend on
the rule, $M$, and $\delta$, but not on the hidden subset or its size.

The target mass depends only on $k$:
\[
 p_{A,\delta}(\mathcal B_A)=t\,g(t),\qquad
 g(t)=\frac1{b+(1-b)t}.
\]
Taking the absolute difference on $\mathcal B_A$ and averaging over subsets gives
\[
 \max_{|A|=k}\|\nu_{\mathsf A,A,\delta}-p_{A,\delta}\|_{\rm TV}
 \ge t\,|P(t)-g(t)|.
\]
A uniform sampling guarantee would therefore require a polynomial of degree at most $M-1$ to approximate this rational function at every admissible $t=k/L$.

\paragraph{A finite-difference lower bound.}
Use the $M+1$ grid points $t_j=1/4+j/(2M)$, corresponding to integer subset sizes $k_j=M+2j$. With spacing $h=1/(2M)$, the $M$th forward difference of $P$ vanishes. Repeated application of the fundamental theorem of calculus gives
\[
 \Delta_h^M g(1/4)
 =\int_{[0,h]^M}g^{(M)}(1/4+u_1+\cdots+u_M)\,du_1\cdots du_M.
\]
The derivative
\[
 g^{(M)}(t)=\frac{(-1)^MM!(1-b)^M}{[b+(1-b)t]^{M+1}}
\]
has constant sign and magnitude at least $M!(1-b)^M$ on $[0,1]$. Consequently,
\[
 h^MM!(1-b)^M
 \le\left|\Delta_h^M g(1/4)\right|
 =\left|\Delta_h^M(P-g)(1/4)\right|
 \le 2^M\max_{0\le j\le M}|P(t_j)-g(t_j)|.
\]
Because $t_j\ge1/4$, some subset $A$ has TV error at least $c_M(1-b)^M$. Pinsker's inequality gives the lower KL bound in~\eqref{eq:selector-minimax}. The argument applies to every allowed rule, including one chosen separately for each $\delta$.

\paragraph{Matching MH order.}
Let $t=|A|/L$. Direct calculation gives
\[
 \chi^2(q\|p_{A,\delta})=\frac{t(1-t)(1-b)^2}{b}
 \le\frac{(1-b)^2}{4b}.
\]
The minorization constant is at least $b$. Applying the chi-square contraction in~\eqref{app:chi-contraction} over $M-1$ updates therefore yields
\[
 \KL(H_{M,A,\delta}\|p_{A,\delta})
 \le\chi^2(H_{M,A,\delta}\|p_{A,\delta})
 \le\frac{(1-b)^{2M}}{4b}.
\]
MH belongs to $\mathcal A_M$, which completes the bounds in~\eqref{eq:selector-minimax}.
\end{proof}

The supremum over hidden subsets in~\eqref{eq:selector-minimax} makes
the optimality of MH's error exponent a worst-case guarantee.
The lower bound need not hold for every fixed pair of policies.
The information restriction matters because knowing $Z$ would reveal
$t=|A|/L$. A selector could then use different coefficients at different
subset sizes, invalidating the polynomial restriction.

To see the effect, suppose $b\ge1/2$ and $t$ is known. Two candidates
then suffice for exact sampling: on a mixed-weight pair, select the
high candidate with probability
\[
 \frac{1+(1-b)t}{2(b+(1-b)t)},
\]
and on equal-weight pairs select uniformly. A direct calculation shows
that the resulting high-weight mass is $t/(b+(1-b)t)$, the target value.
Calibrated weights $p(Y)/q(Y)$ together with raw $\rho(Y)$ also reveal
$Z$ and are excluded from $\mathcal A_M$.

Exponent optimality leaves room for better leading constants and better performance on particular instances. For example, take proposal masses $1/2,1/2$ on two weight classes with weights $1/2,1$. With two candidates, the high-weight class has target mass $2/3$, MH mass $5/8$, and mass $13/20$ under selection proportional to squared weights. The latter is strictly closer to the target in both KL directions and TV. Theorem~\ref{thm:dominance} compares MH with ordinary SIR. Bias-reduced selectors such as those studied by \citet{skare2003improved} lie outside that comparison.

\subsection{Admissibility among weight-only selectors}\label{app:selector-admissibility}

Fix $M\ge1$. Let $\mathcal W_M$ contain rules that return one of $M$ iid
candidates using only their ordered recorded weights in $(0,1]^M$ and
independent randomness. The rule is common to all instances: it receives
neither $q$, $Z$, unobserved weights, nor candidate identities. These
rules preserve $q(\cdot\mid\rho=r)$, so their response divergences equal
their weight-law divergences by~\eqref{app:weight-reduction}. Equality of
selectors below means equality of their marginal output laws.

\begin{theorem}[Rigidity of the coverage rate]\label{thm:selector-rigidity}
If $\mathsf A\in\mathcal W_M$ and a constant $L<\infty$ satisfy
\[
 \|\nu_{\mathsf A}-p\|_{\rm TV}\le L(1-Z/\rho_{\max})^M
\]
on every finite positive instance with $\rho\le1$, then $\nu_{\mathsf A}=H_M$
on every such instance. In particular, MH is the unique weight-only
selector attaining the minimax TV value in Proposition~\ref{prop:minimax-tv}
at every overlap level simultaneously.
\end{theorem}
\begin{proof}
For a weight multiset, average the rule over its orderings and write
$\Phi_i(\boldsymbol k)$ for its probability of selecting weight $r_i$,
where $\boldsymbol k$ records the multiplicities. We show that these
probabilities equal those of priority selection by induction on the number
$n$ of distinct weights in the multiset. The case $n=1$ is immediate.
Fix $r_1<\cdots<r_n$ and assume agreement on every multiset with fewer
values. Give the lower weights proposal masses $s\alpha_i$, with
$\alpha_i>0$, $\sum_{i<n}\alpha_i=1$, and give $r_n$ mass $1-s$.
Then $1-Z/r_n\le s$. The assumed bound and the MH bound in
Proposition~\ref{prop:minimax-tv} imply
$\|\nu_{\mathsf A}-H_M\|_{\rm TV}\le(L+1)s^M$.

By induction, a difference can arise only on multisets containing all
$n$ values. If any such multiset has a nonzero difference, let $k_0$ be
the smallest number of nonmaximal candidates among them. Since the
maximum occurs, $k_0\le M-1$. For some selected level $i$, the output-mass
difference has leading term
\[
 s^{k_0}\sum_{\substack{k_1+\cdots+k_{n-1}=k_0\\k_j\ge1}}
 \frac{M!\,\prod_{j<n}\alpha_j^{k_j}}
      {k_1!\cdots k_{n-1}!(M-k_0)!}
 \bigl[\Phi_i(\boldsymbol k)-\Phi_i^{\rm MH}(\boldsymbol k)\bigr].
\]
Distinct multisets give distinct monomials, so this is a nonzero
homogeneous polynomial. It cannot vanish on the entire positive simplex:
homogeneity would then make it vanish on the open positive orthant,
and hence identically. Thus it is nonzero for some positive
$\boldsymbol\alpha$ summing to $1$. The difference is then of order
$s^{k_0}$, contradicting the $O(s^M)$ bound. This proves the induction and equality on every instance.
For the final claim, a fixed selector attaining the minimax value at
every level satisfies the pointwise bound with $L=1$ by taking the
level to be $Z/\rho_{\max}$. At overlap $1$, all weights are equal.
\end{proof}

\begin{corollary}[Weight-only admissibility]\label{cor:selector-admissibility}
For TV or either direction of KL or chi-square divergence,
if $\mathsf A\in\mathcal W_M$ has divergence no larger than MH's on every
instance, then $\nu_{\mathsf A}=H_M$ everywhere. Thus any weight-only selector
that improves on MH on one instance must be worse on another.
\end{corollary}
\begin{proof}
For TV, apply Theorem~\ref{thm:selector-rigidity}. For the other divergences,
use the same induction and proposal-mass rays. MH has coordinate errors
$O(s^M)$. Each nonmaximal target mass is of order $s$, and its MH mass
is also of order $s$ (for $M=1$, both statements follow directly from
$H_1=q$). Thus both chi-square divergences, and hence both KL divergences,
are $O(s^{2M-1})$. Domination and Pinsker's inequality or the chi-square
bound on TV imply $\|\nu_{\mathsf A}-p\|_{\rm TV}=O(s^{M-1/2})$.
This again excludes any leading difference of order $s^{k_0}$ with
$k_0\le M-1$, completing the induction.
\end{proof}

Admissibility leaves room for better worst-case constants on restricted
families, but any weight-only rule that improves on MH somewhere must be
worse than MH on another instance. The proof shows this by letting the
proposal mass on the largest weight approach $1$. These results establish neither uniqueness at a single overlap level
nor admissibility for every convex divergence.

Note that with $q$ and candidate identities known, a selector can detect when the
pool covers the full support and improve on MH without losses elsewhere. For example,
let $M=3$, $q=(1/2,1/2)$, and $\rho=(1/2,1)$. The high-weight target
mass is $2/3$, whereas MH's is $21/32$. On pools containing both states,
priority selection returns the high weight with probability $7/12$ or $5/6$,
depending on the multiplicities. Increasing this probability by $1/72$ is valid in
both cases. These pools have probability $3/4$, so the modified output
has high-weight mass $21/32+(3/4)(1/72)=2/3$ and is exact. A universal
$q$-aware rule can make this change only when the known proposal and
full-coverage pool identify this instance, and otherwise use MH. It is
therefore no worse on any instance and better on this one.

\subsection{A barrier to certifying the remaining error}\label{app:certificate-barrier}

Proposition~\ref{prop:monotone} certifies improvement over the proposal.
Certifying a small remaining distance to the target is harder because
unobserved responses can carry substantial target mass. Below, a
\emph{rollout-only certifier} is a randomized function, common to all instances, of
$N$ iid proposal responses and their recorded weights. It uses no other
instance information and makes no off-candidate weight queries.

\begin{theorem}[Rollout-only certification barrier]\label{thm:certificate-barrier}
Fix $M,N\ge1$ and $\alpha\in(0,1/2)$. Suppose a rollout-only certifier
$U\in[0,1]$ satisfies
$\Pr\{\|H_M-p\|_{\rm TV}\le U\}\ge1-\alpha$ on every positive
finite-support instance with $\rho\le1$. On any such instance, for every
$s\in(0,1)$,
\begin{equation}\label{eq:certificate-barrier}
 \Pr\{U\ge v_s\}\ge(1-s)^N-\alpha,
 \qquad
 v_s=\frac{s}{s+(1-s)Z}-1+(1-s)^M.
\end{equation}
In particular, if $N\le1/(2Z)$,
\[
 \Pr\!\left\{U\ge\frac12-\frac{M}{2N}\right\}\ge\frac12-\alpha.
\]
The bound applies even to constant-weight instances, where $H_M=p$ exactly.
\end{theorem}
\begin{proof}
Adjoin new responses in a set $A$ with proposal mass $s$ and weight $1$.
On $A^c$, keep the original weights and conditional proposal law. The
new instance has target mass
$p'(A)=s/[s+(1-s)Z]$. Candidate coverage gives
$H'_M(A)\le1-(1-s)^M$, hence $\|H'_M-p'\|_{\rm TV}\ge v_s$.
With probability $e=(1-s)^N$, the certification pool misses $A$.
Conditional on this event, its responses and weights have exactly the
original pool law. Validity on the new instance therefore implies
\[
 e\,\Pr\{U<v_s\}\le\alpha,
 \qquad
 \Pr\{U\ge v_s\}\ge1-\alpha/e\ge e-\alpha
\]
when $e>\alpha$. If $e\le\alpha$, the claimed lower bound is automatic.
For $s=1/(2N)$, Bernoulli's inequality gives $e\ge1/2$ and
$1-(1-s)^M\le M/(2N)$. If $Z\le1/(2N)=s$, then
$p'(A)\ge1/2$, proving the final bound.
\end{proof}

For example, with $Z=0.01$, $N=50$, $M=17$, and $\alpha=0.05$, a valid
upper certificate must be at least $0.33$ with probability at least
$0.45$, even when the actual remaining TV is $0$. By Pinsker's
inequality, an upper certificate for either KL direction obeys the
same probability bound with threshold $2\max\{v_s,0\}^2$ in place of
$v_s$.

\paragraph{Realization with known proposal and support.}
The zero-error obstruction also holds for some instances when the
certifier knows $q$ and the full response support. Use two-token responses
$(j,b)$, with $q(j=\mathrm{hi})=s$, $q(j=\mathrm{lo})=1-s$, and a uniform
binary second token. In the first instance, let both branches have
recorded weight $c\in(0,1)$. In the second, change only the high branch's
weight to $1$. Both are products of two equally weighted specialists:
use the same root law, opposite Bernoulli conditionals with affinity
$c$ on a low-weight branch, and matching uniform conditionals on a
weight-$1$ branch. Both instances have the same $q$ and support.
On the event that the certification pool misses the high branch, every
observed response and weight is the same in both instances. Validity on the second instance
therefore forces $\Pr\{U\ge v_s\}\ge(1-s)^N-\alpha$ on the first, where
$H_M=p=q$ and $Z=c$. Taking $s=1/(2N)\ge c$ gives the same zero-error
bound. Knowledge of the complete weight map would distinguish the
instances and lies outside the claim.

\subsection{Continuation-guided proposals}\label{app:twist}

The local decoder omits the continuation mass $F_{w,t}$
in~\eqref{eq:conditional} when choosing each token. To incorporate an
estimate of this quantity, we introduce a fixed positive function
$\widehat F$ on supported prefixes, with $\widehat F=1$ at complete
responses. This function, called a \emph{twist}, is meant to approximate
the continuation mass, so that $\widehat F(h_tv)\approx F_{w,t}(h_t,v)$.
Suppressing $x,w$, define
\begin{equation}\label{eq:twisted-proposal}
\begin{aligned}
    q^{\widehat F}(v\mid h)
    &=\frac{g(v\mid h)\widehat F(hv)}
    {\sum_u g(u\mid h)\widehat F(hu)},\\
    z^{\widehat F}(h)
    &=\frac{\sum_vg(v\mid h)\widehat F(hv)}{\widehat F(h)},
    \qquad
    \rho^{\widehat F}(y)=\prod_tz^{\widehat F}(h_t).
\end{aligned}
\end{equation}
These are the usual twist-induced proposal and potentials
\citep{guarniero2017iterated,zhao2024twisted}. The twist changes candidate
generation, so the finite-budget comparison must use the same twisted
proposal for both selectors.

\begin{lemma}[Twisted local-normalizer identity]\label{lem:twisted}
For every such twist,
\begin{equation}\label{eq:twisted-identity}
    \ptil_w(y)=\widehat F(x)q^{\widehat F}(y)\rho^{\widehat F}(y),
    \qquad
    p=\frac{q^{\widehat F}\rho^{\widehat F}}
    {\E_{q^{\widehat F}}\rho^{\widehat F}}.
\end{equation}
Consequently, Theorems~\ref{thm:dominance} and~\ref{thm:block-dominance}
and Proposition~\ref{prop:monotone}
apply with the twisted proposal and weights.
\end{lemma}
\begin{proof}
Along a response, the product of the factors in~\eqref{eq:twisted-proposal} telescopes:
\[
    q^{\widehat F}(y)\rho^{\widehat F}(y)
    =\prod_t g(y_t\mid h_t)
      \frac{\widehat F(h_{t+1})}{\widehat F(h_t)}
    =\frac{\ptil_w(y)}{\widehat F(x)}.
\]
The terminal twist equals $1$. Summing over responses gives
$\E_{q^{\widehat F}}\rho^{\widehat F}=Z_w/\widehat F(x)$.
The cited finite-budget results use only positivity and the normalized
importance identity, which have now been verified.
\end{proof}

The certificate formed from twisted weights bounds improvement over
$q^{\widehat F}$, not over the original $q$. A poor twist can make the
twisted proposal farther from the target than $q$, even though MH still
corrects it toward the target.

\subsection{Approximate specialists}\label{app:approx}

Fix the prompt, weights, positive reference, and $\beta>0$. The
supplied policies remain $\pi_i$, while $\pi_i^*$ are the exact optima
of~\eqref{eq:problem}. Write $p=\widehat p_w$, $p^*=\pi_w^*$,
and define the log residuals and specialist regrets
\[
    b_i=\log(\pi_i/\pi_i^*),\qquad
    \bar b=\sum_iw_i b_i,\qquad
    \varepsilon_i=\beta\KL(\pi_i\|\pi_i^*).
\]
The product formula immediately gives
\begin{equation}\label{eq:approximate-tilt}
    p=\frac{p^*e^{\bar b}}{\E_{p^*}e^{\bar b}},
    \qquad
    \KL(\mu\|p^*)=\KL(\mu\|p)+\E_\mu\bar b
                  -\log\E_{p^*}e^{\bar b}.
\end{equation}
All laws are on the common finite support, and each active policy is
positive there. For an output law $\mu$, we call $\beta\KL(\mu\|p^*)=J_{w,x}^{\rm ext}(p^*)-J_{w,x}^{\rm ext}(\mu)$ its \emph{external regret}.

\subsubsection{Stability under coverage}

\begin{proposition}[Coverage controls target mismatch]\label{prop:approximate-stability}
With $c_i=\max_y p(y)/\pi_i(y)$,
\[
    \KL(p\|p^*)\le\sum_iw_i c_i\KL(\pi_i\|\pi_i^*)
    =\frac1\beta\sum_iw_i c_i\varepsilon_i.
\]
No bounded-residual assumption is required.
\end{proposition}
\begin{proof}
Let $\vartheta(u)=e^{-u}-1+u$, a nonnegative convex function.
Using $p^*=pe^{-\bar b}/\E_pe^{-\bar b}$ and
$\log v\le v-1$ gives
\begin{align*}
    \KL(p\|p^*)
    &=\E_p\bar b+\log\E_pe^{-\bar b}
      \le\E_p\vartheta(\bar b)\\
    &\le\sum_iw_i\E_p\vartheta(b_i)
      \le\sum_iw_i c_i\E_{\pi_i}\vartheta(b_i).
\end{align*}
The last inequality uses $p\le c_i\pi_i$ and $\vartheta\ge0$.
Since $\E_{\pi_i}e^{-b_i}=1$,
$\E_{\pi_i}\vartheta(b_i)=\E_{\pi_i}b_i=\KL(\pi_i\|\pi_i^*)$.
\end{proof}

The coverage constants depend on the supplied policies, but because they
are maxima over the entire support, the largest ratios in a finite rollout
sample do not certify them. The bound controls the mismatch between the two targets. For
a finite-budget output $\mu$, the exact decomposition is
\begin{equation}\label{eq:approximate-error-decomposition}
    \KL(\mu\|p^*)=\KL(\mu\|p)+\KL(p\|p^*)
                       +(\E_\mu-\E_p)\bar b.
\end{equation}
Thus the two KL errors cannot simply be added, because the residual shift
also matters. Appendix~\ref{app:hijack} shows that this shift can reverse
the effect of correction.

\subsubsection{Insufficiency of specialist regret alone}

\begin{theorem}[No uniform regret-only guarantee]\label{thm:approximate-impossibility}
Fix $\beta>0$, a positive reference on three responses, and weights
$(1/2,1/2)$. For every $\varepsilon>0$, every composition rule that uses only the supplied
policies, weights, reference, and $\beta$ has unbounded external regret
over reward tuples and supplied
policies satisfying $\max_i\varepsilon_i\le\varepsilon$.
This includes randomized rules and rules using any number of supplied-policy
rollouts, even if given the common regret upper bound $\varepsilon$.
\end{theorem}
\begin{proof}
For $0<\eta\le1/10$, set $s=\sqrt{1-\eta-\eta^3}$ and supply
\[
    \pi_1=(s^2,\eta^3,\eta),\qquad
    \pi_2=(\eta^3,s^2,\eta).
\]
Consider two reward tuples. In the first, the exact optima equal the
supplied policies. In the second, they are
\[
    \pi_1^*=(s^2,\eta,\eta^3),\qquad
    \pi_2^*=(\eta,s^2,\eta^3).
\]
For either tuple, rewards
$\reward_i=\beta\log(\pi_i^*/\piref)$ realize the specified optima
with the same reference and $\beta$. The supplied-policy regrets are
$0$ in the first case and
\[
    \varepsilon_i=2\beta(\eta-\eta^3)\log(1/\eta)\longrightarrow0
\]
in the second. The two scalarized optima are respectively
\[
    P=\frac{(s\sqrt\eta,s\sqrt\eta,1)}{1+2s\sqrt\eta},
    \qquad
    Q=\frac{(s,s,\eta^{5/2})}{2s+\eta^{5/2}}.
\]
All observable inputs to the rule are identical, so it has the same
output law $\mu$ in the two cases.
Let $A_\eta=\sum_yP(y)^{5/6}Q(y)^{1/6}$ and let
$R_\eta=P^{5/6}Q^{1/6}/A_\eta$. Nonnegativity of KL implies
\begin{align*}
    \max\{\KL(\mu\|P),\KL(\mu\|Q)\}
    &\ge\tfrac56\KL(\mu\|P)+\tfrac16\KL(\mu\|Q)\\
    &=\KL(\mu\|R_\eta)-\log A_\eta
      \ge-\log A_\eta.
\end{align*}
The first two entries of $P$ are at most $\sqrt\eta$. The third
entry of $Q$ is at most $\eta^{5/2}$ since $2s\ge1$.
Bounding the remaining probabilities by $1$ gives
$A_\eta\le3\eta^{5/12}$, and hence worst-case external regret at least
\[
    \beta\left[\frac5{12}\log(1/\eta)-\log3\right]
    \longrightarrow\infty.
\]
Both reward tuples satisfy the stated regret upper bound for all
sufficiently small $\eta$, proving the claim.
\end{proof}

For the product rule itself, substitution gives
$\KL(P\|Q)=\tfrac52\log(1/\eta)+\log2+o(1)$.
Composition concentrates on a response that has small probability
under each specialist, so each specialist's approximation error on that
response has little effect on its individual regret. The reward ranges grow in this construction. The theorem does not impose
a uniform reward-range or probability-floor constraint. The coverage condition of Proposition~\ref{prop:approximate-stability} prevents this amplification.

\subsubsection{Possible deterioration of the external objective}\label{app:hijack}

Gains in the consensus and external objectives can have opposite signs at every
candidate budget, even when approximation leaves the local decoder unchanged. Consider binary
responses of length $T\ge2$, with first token $A$ or $B$ chosen
uniformly by both exact specialists. At every later position in branch
$A$, the two specialists' probabilities for token 1 are $a$ and $1-a$,
respectively. In branch $B$ they are $a_1$ and $1-a_1$, where $0<a<a_1\le1/2$.
Use weights $(1/2,1/2)$. The supplied policies swap the assignments
of $a$ and $a_1$ between branches, keeping root probabilities unchanged.
Rewards $\reward_i=\beta\log(\pi_i^*/\piref)$ make the original pair
exact specialists under any positive reference.

Set
\[
    z_A=2\sqrt{a(1-a)},\quad z_B=2\sqrt{a_1(1-a_1)},\quad
    x=(T-1)\log(z_B/z_A)>0,\quad \sigma(x)=\frac1{1+e^{-x}}.
\]
Both local decoders are uniform over all responses. Their recorded
weights, however, are swapped between branches. Both product targets
are uniform conditional on the first token, with
$p(A)=\sigma(x)$ and $p^*(A)=\sigma(-x)$. Consequently,
\[
    \KL(q\|p^*)=\log\cosh(x/2),\qquad
    \KL(p\|p^*)=x\tanh(x/2).
\]
MH using the supplied weights has branch transitions
$K(A,B)=e^{-x}/2$ and $K(B,A)=1/2$. Solving this two-state recursion gives
\[
    H_M(A)=\sigma(x)+(1/2-\sigma(x))\lambda^{M-1},
    \qquad \lambda=(1-e^{-x})/2.
\]
This increases strictly with $M$, from $1/2$ toward $\sigma(x)$.
The binary KL to $\sigma(-x)<1/2$ is strictly increasing for branch
mass at least $1/2$, so every additional candidate increases external
regret $\beta\KL(H_M\|p^*)$, while consensus KL decreases to $0$.
The certificate is nevertheless positive. For two candidates,
equal-branch pools have $G_2=0$, and mixed pools, which occur with
probability $1/2$, have $G_2=\log\cosh(x/2)$. Thus
\[
    \E G_2=\tfrac12\log\cosh(x/2)>0.
\]
The individual regrets are
\[
    \varepsilon_i=\frac{\beta(T-1)}2
    [\operatorname{kl}(a_1,a)+\operatorname{kl}(a,a_1)],
\]
where $\operatorname{kl}$ denotes binary KL.
The same supplied policies could instead be exact for their own implied
rewards, giving positive external gains from the identical rollouts.
Thus supplied policies and recorded weights alone cannot determine
whether correction improves an unspecified external objective.

\section{Enumerable results and protocols}\label{app:toy-table}

We report deterministic calculations of response-law errors and a
separate Monte Carlo comparison of matched-budget samplers. The latter
estimates total variation after grouping responses into recorded-weight
bins (Appendix~\ref{app:matched-protocol}). These estimates are distinct
from the exact full-response calculations below.

\subsection{Exact computation in weight space}\label{app:scalar-law}

The deterministic calculations below use the classical weight-space reduction for the independence sampler \citep{smith1996exact,wang2022exact}.
Let $\mu_n=qK^n=H_{n+1}$. Here $n$ counts MH updates after proposal
initialization, giving a total candidate budget of $M=n+1$.
The reduction works because, conditional on an exact weight
$\rho(y)=r$, every iterate retains the response law
$q(\cdot\mid\rho=r)$. The target and ordinary SIR share this conditional
law (Appendix~\ref{app:proof-dominance}).

Thus response probabilities within each weight level remain in the
same proportions as sampling proceeds. Only the total probability mass
assigned to that level changes. We can therefore describe each iterate
by a function $f_n$ such that
\[
    \mu_n(y)=q(y)f_n(\rho(y)),\qquad f_0\equiv1.
\]

\paragraph{A recursion on the weights.}
Write $Q$ for the law of $R=\rho(Y)$ when $Y\sim q$. From a response of weight $r$, the probability of accepting a fresh proposal is
\[
    a(r)=\E_{R\sim Q}\min\{1,R/r\}.
\]
An MH update retains mass when a proposal is rejected and transfers mass to accepted proposals. Substituting $\mu_n(y)=q(y)f_n(\rho(y))$ into this update and dividing by $q(y)$ gives
\begin{equation}\label{eq:scalar-recursion}
    f_{n+1}(r)=f_n(r)(1-a(r))
    +\E_{R\sim Q}\!\left[f_n(R)\min\{1,r/R\}\right].
\end{equation}
The second term also accounts for proposals that reproduce the current response. For numerical evaluation, accumulate the rejection probability directly as
\[
    1-a(r)=\E_{R\sim Q}(1-R/r)_+.
\]
This form avoids the loss of precision from subtracting an acceptance
probability close to $1$, because it sums only nonnegative contributions
from weights below $r$.

\paragraph{Divergence and objective.}
Write $\rho_\#\mu$ for the law of $\rho(Y)$ when $Y\sim\mu$. At each attained weight $r$,
\[
    (\rho_\#\mu_n)(r)=Q(r)f_n(r),
    \qquad (\rho_\#p)(r)=\frac{rQ(r)}{Z}.
\]
Consequently, for any response with $\rho(y)=r$,
\[
    \frac{\mu_n(y)}{p(y)}
    =\frac{Zf_n(r)}{r}
    =\frac{(\rho_\#\mu_n)(r)}{(\rho_\#p)(r)}.
\]
Grouping the divergence sum by exact weight now gives
\begin{equation}\label{eq:scalar-divergence}
    D_f(\mu_n\|p)=D_f(\rho_\#\mu_n\|\rho_\#p).
\end{equation}
This identity covers exact sequence-level KL and TV, even with tied
weights. A single score mean or a histogram that merges different weights
does not generally determine these divergences. To apply the reduction
to another sampler, we would need the same conditional response law given
$\rho$. We do not assume this property for SMC or replica exchange, so
their objective evaluation requires separate justification.

The consensus objective can also be evaluated from these scalar quantities. Substituting $\ptil_w=q\rho$ and $\mu_n=qf_n(\rho)$ into~\eqref{eq:consensus} gives
\begin{equation}\label{eq:scalar-objective}
    \widetilde J_x(\mu_n)=\E_{\mu_n}[\log\rho-\log f_n(\rho)],
    \qquad \widetilde J_x(q)=\E_q\log\rho,
    \qquad \widetilde J_x(p)=\log Z.
\end{equation}
For ordinary SIR, conditioning on a tagged candidate and using exchangeability gives the corresponding marginal ratio
\[
    \frac{S_M(y)}{q(y)}=M\,\E\!\left[
    \frac{\rho(y)}{\rho(y)+R_2+\cdots+R_M}\right],
    \qquad R_j\overset{\mathrm{iid}}{\sim}Q,
\]
with an empty sum when $M=1$. This ratio depends only on $\rho(y)$, so the same divergence and objective identities apply with the SIR ratio in place of $f_n$.

\paragraph{Empirical evaluation.}
Replacing $Q$ by an empirical distribution evaluates the sampler under
that empirical proposal. How well this approximates the population law
depends on the estimated weight distribution, including its coverage of
rare weights. Estimation error also enters the plug-in objective through
$\log f_n$. Using independent samples for fitting and evaluation does
not remove the resulting bias. The GSM8K study avoids fitting the scalar law by
using the rollout-only improvement certificate
(Appendix~\ref{app:gsm8k-protocol}).

\subsection{Exact laws and nearby-policy rates}

The enumerable study uses a vocabulary of 4 tokens and 3 autoregressive policies.
At every prefix, each policy's conditional probabilities are drawn
independently from a symmetric Dirichlet distribution with concentration
$\alpha$. Each replication draws $w\sim\mathrm{Dirichlet}(1,1,1)$.
The base, peaky, and flat configurations use $T=6$, with concentrations
$\alpha=0.5$, $0.2$, and $1.0$, respectively. Short and long use $\alpha=0.5$
at $T=3$ and $T=10$, respectively. We use 20 replications per configuration,
except for the long configuration ($T=10$), which uses 10. With a uniform reference and rewards
$\reward_i=\log(\pi_i/\piref)$, these policies are exact optima at
$\beta=1$, so $\widehat p_w=\pi_w^*$ and each KL reduction equals a
regularized-objective gain.

For each instance, enumeration gives $q_w$, $\widehat p_w$, and $\rho_w$
on the complete response space. We verify $\ptil_w=q_w\rho_w$ and
propagate the MH law deterministically. SIR's marginal density ratio is
evaluated by quadrature of the scalar identity in
Appendix~\ref{app:scalar-law}, checked against direct candidate-pool
enumeration on small supports. These calculations introduce no sampling
or density-fitting error. Figure~\ref{fig:toy}, left, plots the mean KL
at candidate budgets $M\in\{1,2,3,4,6,9,13,21,33,51\}$ on logarithmic
axes, omitting values at or below $10^{-12}$. The MH calculation uses
stable rejection accumulation as in Appendix~\ref{app:scalar-law}.
Table~\ref{tab:toy} gives selected values and their variability across
policy configurations.

\begin{table}[t]
\caption{Exact $\KL(\mu_n\|\widehat p_w)$ in nats after $n$ MH updates,
initialized from $q_w$, with total candidate budget $M=n+1$. Entries are
means over 20 policy configurations (10 at $T=10$); standard deviations
appear beneath.}
\label{tab:toy}
\centering
\small
\setlength{\tabcolsep}{4pt}
\begin{tabular}{lcc@{\hspace{1em}}cccccc}
\toprule
& & & \multicolumn{6}{c}{MH updates $n$ (budget $M=n+1$)}\\
\cmidrule(l){4-9}
Configuration & $\alpha$ & $T$ & $0$ & $1$ & $2$ & $3$ & $5$ & $8$ \\
\midrule
Base  & 0.5 & 6  & $0.154$ & $0.022$ & $0.0055$ & $0.0018$ & $0.0003$ & $<\!10^{-4}$ \\
      &     &    & \scriptsize$\pm0.056$ & \scriptsize$\pm0.015$ & \scriptsize$\pm0.005$ & \scriptsize$\pm0.002$ & \scriptsize$\pm0.0004$ & \\
\addlinespace[2pt]
Peaky & 0.2 & 6  & $0.603$ & $0.181$ & $0.084$ & $0.048$ & $0.021$ & $0.0089$ \\
      &     &    & \scriptsize$\pm0.405$ & \scriptsize$\pm0.181$ & \scriptsize$\pm0.098$ & \scriptsize$\pm0.061$ & \scriptsize$\pm0.032$ & \scriptsize$\pm0.018$ \\
\addlinespace[2pt]
Flat  & 1.0 & 6  & $0.034$ & $0.0017$ & $0.0002$ & $<\!10^{-4}$ & $<\!10^{-4}$ & $<\!10^{-4}$ \\
      &     &    & \scriptsize$\pm0.012$ & \scriptsize$\pm0.0012$ & \scriptsize$\pm0.0002$ & & & \\
\addlinespace[2pt]
Short & 0.5 & 3  & $0.042$ & $0.0024$ & $0.0003$ & $<\!10^{-4}$ & $<\!10^{-4}$ & $<\!10^{-4}$ \\
      &     &    & \scriptsize$\pm0.034$ & \scriptsize$\pm0.0038$ & \scriptsize$\pm0.0006$ & & & \\
\addlinespace[2pt]
Long  & 0.5 & 10 & $0.270$ & $0.060$ & $0.023$ & $0.011$ & $0.0032$ & $0.0008$ \\
      &     &    & \scriptsize$\pm0.125$ & \scriptsize$\pm0.044$ & \scriptsize$\pm0.020$ & \scriptsize$\pm0.010$ & \scriptsize$\pm0.0034$ & \scriptsize$\pm0.0009$ \\
\bottomrule
\end{tabular}
\end{table}

One update reduces mean KL by $85.9\%$ in base, $70.0\%$ in peaky,
$95.0\%$ in flat, $94.4\%$ in short, and $77.6\%$ in long.

\paragraph{Nearby-policy illustration.}
The right panel of Figure~\ref{fig:toy} uses a separate fixed instance
of Theorem~\ref{thm:nearby}: vocabulary size 3, horizon 3, and 2 equally
weighted policies, giving 27 complete responses. The common policy has
independent $\mathrm{Dirichlet}(1,1,1)$ conditionals at each prefix. Its two
Gaussian logit-perturbation arrays are fixed while $\delta$ varies.
We evaluate the local decoder and both selectors for $M=2$ and $M=3$ over
12 logarithmically spaced perturbation sizes from $0.01$ to $1$.
We re-evaluate the fixed instance with 90-digit arithmetic to resolve
the small MH divergences without floating-point cancellation, keeping
the policies and perturbation sizes unchanged. MH uses its dense
transition kernel, and SIR uses direct candidate-pool enumeration.
The guide slopes show asymptotic orders rather than fitted rates.
This experiment varies policy separation at fixed support, horizon,
and budget, as the theorem requires.

\paragraph{Greedy initialization.}
In Figure~\ref{fig:toy-greedy}, MH starts from the response obtained by
choosing the most probable local-product token at each step. This chain begins
farther from the target than the proposal-initialized chain but converges to the
same target, as Proposition~\ref{prop:convergence} guarantees. After 8
updates, mean KL is below $0.012$ in every configuration and below
$10^{-4}$ in base, flat, and short. The comparison with ordinary SIR in
Theorem~\ref{thm:dominance} requires proposal initialization and does
not cover this chain.

\begin{figure}[t]
\centering
\includegraphics{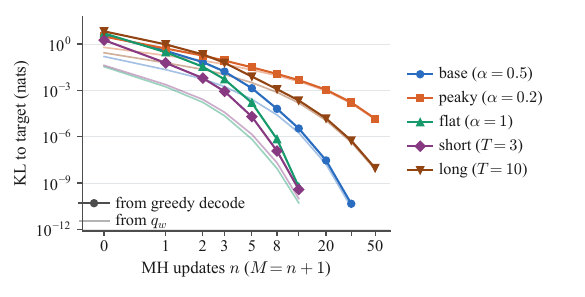}
\caption{Exact KL of MH started from the greedy local-product decode, averaged
within each configuration of Table~\ref{tab:toy}. The budget $M=n+1$ counts
the greedy decode and $n$ independent MH proposals. Faint lines repeat the
proposal-initialized MH curves of Figure~\ref{fig:toy} for comparison. Values
at or below $10^{-12}$ are omitted.}
\label{fig:toy-greedy}
\end{figure}

\subsection{Reward--divergence view}\label{app:frontier}

At $T=10$ and equal weights, take a uniform reference and the synthetic
rewards $\reward_i=\log(\pi_i/\piref)$ introduced at the start of
this appendix. The scalarized reward is then the implicit
checkpoint score $r_w=\sum_iw_i\log(\pi_i/\piref)$. Varying the
regularization strength in
$\pi_\lambda\propto\piref\exp(\reward_w/\lambda)$, $\lambda>0$,
traces the optimal reward--divergence trade-off at fixed reward weights.
Its point at $\lambda=1$ is the correction target.
Figure~\ref{fig:toy-frontier} shows the trade-off and both MH trajectories.
Both initializations approach the target, but expected reward itself
need not increase: the greedy-initialized chain can lower expected reward
while improving its regularized objective.

\begin{figure}[t]
\centering
\includegraphics[width=\textwidth]{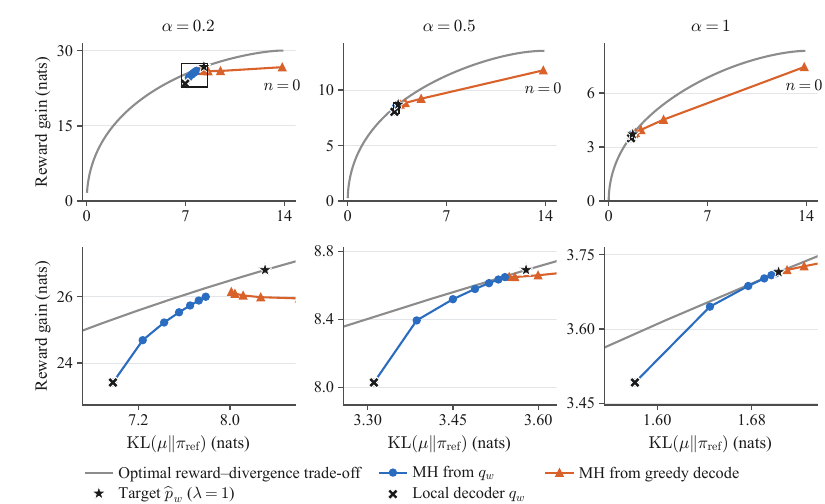}
\caption{Optimal reward--divergence trade-offs and MH trajectories at
$T=10$, with equal reward weights and a uniform reference. Columns vary
the Dirichlet concentration; the bottom row enlarges the boxed region
around the target. Reward is centered at its reference expectation,
$\E_\mu\reward_w-\E_{\piref}\reward_w$. Curves average 10 policy draws;
trajectory markers show initialization ($n=0$) and 6 MH updates.}
\label{fig:toy-frontier}
\end{figure}

\subsection{Matched-budget samplers and metric}\label{app:matched-protocol}

Table~\ref{tab:matched-detail} compares selection and structured-sampling rules
at a budget of 9 full-sequence equivalents of conditional evaluations. For each primary
instance, we use 8 independent batches of $10^6$ complete
algorithm runs per method. All entries, including MH and $q_w$, use the
same empirical metric. Exact controls are given separately below.
The uncorrected proposal is shown at its natural one-rollout cost,
since generating unused candidates would not change its law.

\paragraph{Sampling-importance-resampling.}
Ordinary SIR draws $M$ independent responses from $q_w$ and selects one
with probability proportional to its recorded $\rho_w$ weight
\citep{owen2013mc}. The primary comparison uses $M=9$.

\paragraph{Sequential Monte Carlo.}
Bootstrap SMC maintains $M$ partial responses. At each step, every
particle is extended by one token from the local product decoder, and its
weight is multiplied by $z_{w,t}$. Multinomial resampling draws particles in proportion to their
weights and resets those weights to equal values. We resample either
after every nonterminal step or whenever the effective sample size (ESS)
falls below half the particle count. For normalized weights $a_j$,
$\mathrm{ESS}=1/\sum_ja_j^2$. The final response is selected using the
remaining weights.

\paragraph{Same-space replica exchange.}
Three replicas target distributions proportional to $q_w\rho_w^b$
for $b\in\{1,1/2,0\}$. Each starts from an independent $q_w$ rollout.
At budget $M$ divisible by 3, we run $M/3-1$ sweeps. Each sweep applies
an independence-MH update to every replica and then proposes adjacent
swaps in increasing replica-index order. The swap acceptance ratio
preserves the joint target. We return the $b=1$ replica. At $M=9$,
3 initial rollouts and 2 update sweeps cost 9 complete decodes.
The swaps add no conditional evaluations.

\paragraph{Length-space replica exchange.}
Following \citet{he2026recipes}, the intermediate targets are products
of policy prefix probabilities. Replicas have lengths $0,2,\ldots,T$.
An adjacent exchange truncates the longer prefix and extends the shorter
one with a fresh 2-token chunk. Its acceptance ratio uses the local
normalizer products on the new and removed chunks. Initialization costs
$0+2+\cdots+T$ token evaluations. Each sweep pair visits nonoverlapping
even edges and then odd edges, and the longest response is returned.
At $T=6$, 7 pairs cost $12+7\times6=54=9T$ evaluations. At $T=10$,
6 pairs cost $30+6\times10=90=9T$. The implementation includes both
initialization and extensions in the charged budget.

\paragraph{Maximum-log-normalizer selection.}
The $\maxlr$ baseline returns the response with the largest $\lrho$ among
9 independent $q_w$ candidates. It selects on the correction weight,
not on reward. In the base and long
instances, it is farther from the target than $q_w$, which shows that maximizing a correction weight need not
approximate sampling from the target. The corrected samplers are matched
by conditional-evaluation counts, although batching and implementation can
still affect elapsed time.

\paragraph{Binned total variation.}
Responses are sorted by $\lrho$, with bin boundaries at cumulative
target-probability thresholds $j/32$, $j=1,\ldots,31$. Discreteness
means that bins need not have equal mass, and the saved ordering can
split tied weights. Every method uses the same partition and exact
target bin masses. TV on these bins lower-bounds full-response TV.
The exact scalar-law equality in Appendix~\ref{app:scalar-law} applies
to weight level sets. Combining different weights into bins can lose
information.

The MH--SIR ordering nevertheless holds for these consecutive bins.
To see this, let $C_k$ be the union of the first $k$ bins and put
$a_k(r)=q(C_k\mid\rho=r)$. This function is nonincreasing, including
when a bin boundary splits a tied level. Both selectors preserve the
proposal conditional law within each level, so the stochastic ordering
in Appendix~\ref{app:proof-dominance} gives
\[
    S_M(C_k)=\E_{\mathsf S_M}a_k
    \ge\E_{\mathsf H_M}a_k=H_M(C_k).
\]
The binned MH-to-target ratios are also nonincreasing: each is a
target-weighted average of the decreasing pointwise ratio over a
consecutive bin. The same supporting-line argument therefore proves
the binned divergence ordering for every convex generator. This
population statement does not imply the same ordering for individual
Monte Carlo estimates.

\begin{table}[t]
\centering
\small
\caption{Matched-budget comparison in units of $10^{-3}$ TV.
Sampled rows give the mean $\pm$ across-batch standard deviation of the binned
TV over 8 batches of $10^6$ runs. The corrected samplers use 9 full-sequence
equivalents of conditional policy evaluations. The last sampled row draws from the target itself and shows the
sampling variation of the estimate. Exact rows propagate each law
deterministically. Each configuration uses the same policy instance across methods.}
\label{tab:matched-detail}
\begin{tabular}{lr@{\,$\pm$\,}lr@{\,$\pm$\,}lr@{\,$\pm$\,}l}
\toprule
Method & \multicolumn{2}{c}{Base} & \multicolumn{2}{c}{Peaky} & \multicolumn{2}{c}{Long}\\
\midrule
\multicolumn{7}{l}{\textit{Estimated binned TV}}\\
$q_w$ (1 rollout) & 237.765 & 0.509 & 515.610 & 0.438 & 302.427 & 0.483\\
MH & 2.506 & 0.254 & 37.762 & 0.234 & 7.596 & 0.217\\
Ordinary SIR & 30.097 & 0.662 & 86.999 & 0.639 & 42.918 & 0.450\\
SMC, adaptive resampling & 30.064 & 0.435 & 73.574 & 0.481 & 43.051 & 0.349\\
SMC, resampling every step & 52.170 & 0.214 & 106.350 & 0.501 & 68.462 & 0.512\\
Same-space replica exchange & 14.634 & 0.438 & 109.666 & 0.441 & 33.324 & 0.480\\
Length-space replica exchange & 10.027 & 0.607 & 116.224 & 0.608 & 53.493 & 0.530\\
\maxlr{} & 472.451 & 0.281 & 215.438 & 0.708 & 412.216 & 0.333\\
Target $\widehat p_w$ (direct draws) & 2.161 & 0.292 & 1.524 & 0.201 & 2.184 & 0.249\\
\midrule
\multicolumn{7}{l}{\textit{Exact TV}}\\
$q_w$, binned & \multicolumn{2}{c}{237.719} & \multicolumn{2}{c}{515.909} & \multicolumn{2}{c}{302.413}\\
$q_w$, full response & \multicolumn{2}{c}{237.773} & \multicolumn{2}{c}{516.016} & \multicolumn{2}{c}{302.615}\\
MH, binned & \multicolumn{2}{c}{1.354} & \multicolumn{2}{c}{37.619} & \multicolumn{2}{c}{7.455}\\
MH, full response & \multicolumn{2}{c}{1.355} & \multicolumn{2}{c}{37.678} & \multicolumn{2}{c}{7.456}\\
\bottomrule
\end{tabular}
\end{table}

\paragraph{Exact controls and sampling variation.}
The exact rows of Table~\ref{tab:matched-detail} give the binned and
full-response TV of $q_w$ and MH on the same primary instances. The target row applies the
empirical metric to independent draws from $\widehat p_w$ under the same
batch design. It describes sampling variation at this sample size. It is
neither a universal significance threshold nor a bias correction to subtract
from other entries.

\paragraph{Sensitivity to instance and budget.}
At budget 9, the sensitivity sweep uses 10 policy seeds each for base
and peaky and 3 for long, with $10^5$ runs per batch for the additional
instances. MH has lower binned TV than ordinary SIR on all 23 instances, and all stored
pointwise 95\% Monte Carlo intervals for their differences exclude $0$.

The comparison with structured samplers varies across instances.
Against all structured comparators, MH has the smallest point estimate
on 7 of 10 base, 7 of 10 peaky, and all 3 long instances. Adaptive SMC
is $0.0145$ TV below MH on one peaky instance, with interval
$[0.0128,0.0163]$ for MH minus SMC. On one base instance, however, extending
the evaluation to $10^6$ runs per batch does not resolve the apparent
same-space-exchange advantage. MH minus exchange is $+0.0000223$,
with interval $[-0.0003776,+0.0004221]$.

For each comparison, we subtract the empirical binned TV values at matching
batch indices to obtain 8 differences $d_b$. The pointwise 95\%
Student-$t$ interval is $\bar d\pm t_{7,0.975}s_d/\sqrt8$, where $s_d$
is the sample standard deviation of those differences and
$t_{7,0.975}\simeq2.365$. The methods use independent random streams.

Budgets 18 and 36 were evaluated on the same primary instances.
MH remains below ordinary SIR at both budgets. At budget 36, the MH
and replica-exchange discrepancies are close to the exact-target
sampling diagnostic, so the empirical metric provides little separation
among them.

\paragraph{Objective checks.}
Across 51 enumerable policy instances, direct regularized-objective
increments were cross-checked against forward-KL differences and the
scalar-weight calculation. Identity residuals are at floating-point precision.

\paragraph{Recorded runs.}
Each configuration, policy instance, budget, method, and batch uses an
independent random stream. The archived outputs include seeds, bin maps,
counts, charged costs, source snapshots, and aggregate tables. Nested 10,000-, 100,000-, and
1,000,000-sample prefixes share draws and are correlated. The values in
Table~\ref{tab:matched-detail} come from this matched-protocol run collection.

\section{GSM8K experiment}

\subsection{GSM8K certificate evaluation}\label{app:gsm8k-protocol}

\paragraph{Models and data.}
The experiment composes 2 DPO adapters of
\href{https://huggingface.co/google/gemma-4-E2B-it}{\nolinkurl{google/gemma-4-E2B-it}}, trained for answer accuracy and
brevity, respectively. Both use a DPO coefficient of $0.1$ and a candidate
pool of 4,000 distinct GSM8K training problems with 4 responses per
problem. Brevity preferences rank responses by the continuous score $-L/1024$, where $L$
counts generated tokens, including EOS but excluding the prompt and
padding. DPO uses this ordering rather than the numeric magnitudes. The supplied checkpoints are fixed throughout evaluation.
We report consensus improvement ($\lambda=1$
in~\eqref{eq:objective-gap}), independently of the DPO training coefficient.

The final cohort contains all 1,319 official GSM8K test problems.
An exact-question comparison found no overlap with the specialist-training
pool. The protocol, checkpoints, weights, and sampling settings were
fixed before final generation. The dataset revision, original test
bytes, and prompt manifest are archived under
\nolinkurl{data/certificates/gsm8k}.

\paragraph{Sampling and candidate budgets.}
We evaluate accuracy weights $w_1\in\{0.25,0.5,0.75\}$ and brevity weight
$1-w_1$ on the same prompt cohort. Generation uses temperature 1, the
full vocabulary without top-$k$ or nucleus truncation, and a 1,024-token
cap. Reasoning is disabled by passing \texttt{enable\_thinking=False} to
the tokenizer's chat template. Policy inference uses FP32, and normalization
and probability accumulation use FP64. EOS contributes
its probability and counts toward response length. Reaching the cap
adds no forced terminal token. Empty, invalid, and capped responses
remain in the evaluation.
The prompt consists of the following text, followed by a blank line and the question:
\begin{quote}\ttfamily\raggedright
Solve the following problem. Show your solution and end with exactly one
line of the form \#\#\#\# <numerical answer>.
\end{quote}

For each prompt and weight, 4 replications each draw 2 independent
groups of 17 fresh local-product responses. This gives 8 groups per
prompt and $N=10{,}552$ groups per weight. The complete study uses
$3\times10{,}552\times17=538{,}152$ responses, with no fitting pool.
Each group supplies all budgets $M=1,\ldots,17$ through its prefixes.
The first candidate initializes MH, and the remaining candidates are
processed once in chronological order with independently generated
acceptance uniforms. Content-derived seeds give each response its own
sampling stream. The generation settings and code used for sampling
are archived with the results.

Replicated groups estimate the samplers' expected behavior over fresh
candidate pools. They are evaluation overhead, not the deployment budget.
At each rollout budget $M$, both MH and ordinary SIR use $M$ candidate
rollouts per output, while the local-decoding baseline uses 1.
We use no burn-in and exclude initialization from the acceptance rates.

\paragraph{Simultaneous confidence bounds.}
For every group and budget, we compute $G_M$ in~\eqref{eq:pool-gap} from
the first $M$ original proposal weights, rather than the retained MH
states. With equal numbers of groups per prompt, the pooled mean estimates
the equally weighted average over the fixed test cohort. With clipping
cap $\Gamma=10$ nats and familywise error level $\alpha=0.05$, the
lower confidence bound is
\[
    L_M=\max\!\left\{0,\;
    \frac1N\sum_{j=1}^N\min(G_M^{(j)},10)
    -10\sqrt{\frac{\log(144/0.05)}{2\cdot10{,}552}}\right\}.
\]
We set $L_1=0$ exactly. The common subtraction is $0.194279$ nats.
The factor 144 retains the original prespecified family of 3 tasks,
3 weights, and 16 nonzero update budgets. This protocol planned GSM8K,
HH-RLHF, and Reddit summarization studies. We retain its correction for
GSM8K, although only 48 of those comparisons are reported here. The final
HH-RLHF study uses its own family of 48 comparisons
(Appendix~\ref{app:hh-protocol}). Each study therefore has a simultaneous
95\% confidence statement within its reported family, not a joint 95\%
statement across the two studies.

The bound assumes independent groups conditional on the fixed prompts
and checkpoints. Group distributions may differ across prompts.
The union bound does not require independence across nested budgets or
composition weights. On the same confidence event, both MH and SIR have
average consensus improvement at least $L_M$ over fresh proposal pools
for these prompts. No additional comparison factor is needed for the
two methods, because Proposition~\ref{prop:monotone} lower-bounds both
gains by the same expectation. This is not a confidence statement over
a population of prompts or about the MH--SIR difference.

\begin{table}[t]
\centering
\caption{GSM8K certificate results, in nats. Means are unclipped; the
simultaneous 95\% lower bounds use certificates clipped at 10 nats, and the
clipped columns give the percentage of groups whose certificate was clipped. Cumulative MH acceptance
over the first $n$ updates excludes initialization.}
\label{tab:gsm8k-certificate}
\small
\setlength{\tabcolsep}{4pt}
\begin{tabular}{ccccccccc}
\toprule
& \multicolumn{3}{c}{$M=9$} & \multicolumn{3}{c}{$M=17$} & \multicolumn{2}{c}{MH acceptance (\%)}\\
\cmidrule(lr){2-4}\cmidrule(lr){5-7}\cmidrule(l){8-9}
$w_1$ & Mean $G_M$ & Lower bound & Clipped & Mean $G_M$ & Lower bound & Clipped & $n=8$ & $n=16$\\
\midrule
0.25 & 0.868 & 0.674 & 0.038 & 0.937 & 0.742 & 0.038 & 53.92 & 50.53\\
0.50 & 1.700 & 1.494 & 0.597 & 1.848 & 1.644 & 0.663 & 46.10 & 41.14\\
0.75 & 1.413 & 1.216 & 0.265 & 1.528 & 1.330 & 0.256 & 49.61 & 45.09\\
\bottomrule
\end{tabular}
\end{table}

Table~\ref{tab:gsm8k-certificate} gives the certificate and acceptance
endpoints. At $M=9$, clipping affects $0.038\%$, $0.597\%$, and $0.265\%$
of groups in increasing order of $w_1$. The solid curves in the top-left panel of
Figure~\ref{fig:llm-results} estimate $\E[G_M]$, not the full consensus gain.
Positive lower bounds establish improvement over local decoding without
estimating either corrected sampler's marginal output probabilities.
Acceptance is a separate sampling diagnostic and does not measure
remaining divergence from the target.

\paragraph{Choice of a fitting-free endpoint.}
The scalar identities in Appendix~\ref{app:scalar-law} also suggest
estimating each sampler's objective by fitting its output density ratio.
Independent fitting does not remove nonlinear plug-in bias, and the
resulting intervals require separate coverage validation. Controlled
finite-state tests did not pass the prespecified bias and coverage
thresholds, so fitted objective estimates do not enter this evaluation.
Those tests and the earlier fitted estimates remain archived under
\nolinkurl{results/llm-regularised-paired-validation} and
\nolinkurl{results/gsm8k-draft-update-20260918}, respectively.

\paragraph{Separate specialist evaluation.}
A separate evaluation draws 4 responses per adapter on 1,000 validation
problems. The accuracy adapter achieves $86.875\%$ accuracy with a
mean length of $390.77$ tokens, compared with $85.425\%$ and $250.95$ tokens for
the brevity adapter. Paired 95\% prompt-bootstrap intervals from 2,000
draws are $[132.68,146.88]$ tokens saved and $[0.325,2.65]$ percentage
points of accuracy lost. This evaluation was expanded after an initial
128-problem comparison. It uses BF16, top-$p=0.95$, and top-$k=64$ on
a different prompt set. These measurements characterize the supplied
specialists and are not task-performance results for the full-vocabulary
FP32 composition study. They are archived in
\nolinkurl{results/gsm8k-draft-update-20260918/specialists.json}.

\paragraph{Evaluation artifacts and replay.}
The completed certificate curves and clipping statistics are archived
under \nolinkurl{results/certificates/gsm8k}, together with per-step
acceptance counts and per-prompt acceptance rates. Record hashes and
frozen launch sources accompany the results. The statistical design is
specified in \nolinkurl{experiments/certificates/PROTOCOL.md}.

Within the manuscript directory, \nolinkurl{figures/SOURCES.json}
records hashes of the unchanged plotting inputs. The script
\nolinkurl{scripts/verify_figure_data.py} checks the saved certificate
arithmetic and acceptance aggregates. These files support local figure
regeneration without model inference.

Recomputing certificates from observations requires the original grouped
proposal weights. Acceptance replay additionally requires the recorded
uniforms. The full raw batches are external, checksum-addressed artifacts
listed in the companion archive, rather than bundled in the manuscript
export. The original analysis entry point also checks hashes of model
and adapter files, so reproducing that command requires those assets to
be restored. Checks of the saved aggregates and provenance therefore
remain distinct from a full-record or checkpoint-level rerun.

\subsection{GSM8K output-level replay and cost}\label{app:output-evaluation}
The analysis reuses all 538,152 archived responses and verifies the 7,914 raw
batch checksums, record identities, saved seeds and acceptance uniforms, the probability identity,
and GSM8K answer extraction. The bootstrap resamples whole prompts, retaining all 8
groups of each resampled prompt, and uses common draws across methods, weights, and budgets (10,000
draws). We report descriptive pointwise 95\% percentile
intervals, not the simultaneous fixed-cohort certificate bounds of Appendix~\ref{app:gsm8k-protocol}.
The local decoder returns the first proposal (proposal 0). MH replays the
fresh proposals in generation order. SIR results are exact expectations conditional on each pool.
Independently seeded sampled selections are also archived. Majority voting counts normalized valid
answers, breaks ties by earliest appearance, returns the first response with the
winning answer, and falls back to proposal 0 when all answers are invalid.
Invalid candidates still count toward generation cost, and no selection rule
uses the gold answer. Token costs include EOS tokens and all discarded candidates.
Table~\ref{tab:gsm8k-output} reports output-level
results for each weight, and Table~\ref{tab:gsm8k-mh-sir} reports the paired MH--SIR differences.
\begin{table}[p]
\centering
\caption{Output-level GSM8K results at each candidate budget. Length counts the
tokens of the returned response; total tokens include all generated candidates.
At $M=1$, MH, SIR, and majority voting all return the local decoder's single
proposal. Brackets give pointwise 95\% prompt-bootstrap intervals.}
\label{tab:gsm8k-output}
\footnotesize
\begin{tabular}{rrlrcrc}
\toprule
$M$ & Total tokens & Method & Accuracy (\%) & 95\% CI & Length & 95\% CI\\
\midrule
\multicolumn{7}{l}{\textit{Accuracy weight $w_1=0.25$}}\\
1 & 259.6 & Local decoder ($q_w$) & 88.50 & [86.99, 89.93] & 259.6 & [254.1, 265.3]\\
\addlinespace[2pt]
2 & 519.8 & MH & 89.16 & [87.65, 90.58] & 251.2 & [246.2, 256.3]\\
 &  & SIR & 89.13 & [87.63, 90.54] & 251.6 & [246.6, 256.8]\\
 &  & Majority vote & 88.98 & [87.50, 90.40] & 258.8 & [253.4, 264.3]\\
\addlinespace[2pt]
3 & 779.8 & MH & 89.26 & [87.77, 90.68] & 248.4 & [243.7, 253.3]\\
 &  & SIR & 89.30 & [87.81, 90.71] & 249.0 & [244.2, 253.9]\\
 &  & Majority vote & 90.14 & [88.65, 91.57] & 257.5 & [252.2, 262.9]\\
\addlinespace[2pt]
5 & 1299.3 & MH & 89.48 & [87.97, 90.88] & 246.5 & [241.9, 251.2]\\
 &  & SIR & 89.48 & [88.01, 90.87] & 246.7 & [242.2, 251.4]\\
 &  & Majority vote & 90.52 & [89.05, 91.94] & 257.0 & [251.7, 262.3]\\
\addlinespace[2pt]
9 & 2339.6 & MH & 89.45 & [87.94, 90.87] & 244.6 & [240.1, 249.2]\\
 &  & SIR & 89.58 & [88.10, 90.98] & 245.1 & [240.6, 249.7]\\
 &  & Majority vote & 90.84 & [89.36, 92.26] & 256.5 & [251.4, 261.7]\\
\addlinespace[2pt]
17 & 4422.8 & MH & 89.68 & [88.17, 91.08] & 243.6 & [239.2, 248.0]\\
 &  & SIR & 89.64 & [88.17, 91.04] & 244.1 & [239.7, 248.6]\\
 &  & Majority vote & 91.24 & [89.77, 92.66] & 255.7 & [250.7, 260.8]\\
\midrule
\multicolumn{7}{l}{\textit{Accuracy weight $w_1=0.5$}}\\
1 & 291.6 & Local decoder ($q_w$) & 88.88 & [87.40, 90.29] & 291.6 & [285.0, 298.5]\\
\addlinespace[2pt]
2 & 583.5 & MH & 89.83 & [88.36, 91.21] & 277.5 & [271.7, 283.5]\\
 &  & SIR & 89.74 & [88.28, 91.11] & 277.9 & [272.1, 283.9]\\
 &  & Majority vote & 89.39 & [87.93, 90.77] & 289.5 & [283.1, 295.9]\\
\addlinespace[2pt]
3 & 874.0 & MH & 90.09 & [88.65, 91.47] & 272.3 & [266.8, 277.9]\\
 &  & SIR & 90.03 & [88.58, 91.40] & 272.8 & [267.3, 278.4]\\
 &  & Majority vote & 90.37 & [88.89, 91.76] & 288.5 & [282.2, 294.9]\\
\addlinespace[2pt]
5 & 1456.7 & MH & 90.16 & [88.70, 91.54] & 268.2 & [263.0, 273.5]\\
 &  & SIR & 90.13 & [88.69, 91.52] & 268.7 & [263.5, 274.0]\\
 &  & Majority vote & 91.03 & [89.59, 92.40] & 287.8 & [281.6, 294.1]\\
\addlinespace[2pt]
9 & 2624.2 & MH & 90.09 & [88.61, 91.48] & 265.4 & [260.3, 270.5]\\
 &  & SIR & 90.18 & [88.72, 91.56] & 266.1 & [261.0, 271.2]\\
 &  & Majority vote & 91.04 & [89.58, 92.45] & 287.4 & [281.3, 293.7]\\
\addlinespace[2pt]
17 & 4955.8 & MH & 90.15 & [88.68, 91.54] & 263.5 & [258.5, 268.6]\\
 &  & SIR & 90.17 & [88.70, 91.54] & 263.9 & [258.9, 268.9]\\
 &  & Majority vote & 91.31 & [89.84, 92.69] & 286.6 & [280.4, 292.8]\\
\midrule
\multicolumn{7}{l}{\textit{Accuracy weight $w_1=0.75$}}\\
1 & 330.1 & Local decoder ($q_w$) & 89.16 & [87.65, 90.53] & 330.1 & [322.2, 338.0]\\
\addlinespace[2pt]
2 & 661.8 & MH & 89.92 & [88.46, 91.28] & 312.2 & [305.1, 319.3]\\
 &  & SIR & 89.85 & [88.39, 91.22] & 312.7 & [305.6, 319.7]\\
 &  & Majority vote & 89.73 & [88.26, 91.08] & 327.0 & [319.5, 334.6]\\
\addlinespace[2pt]
3 & 992.3 & MH & 90.16 & [88.69, 91.54] & 304.3 & [297.7, 310.9]\\
 &  & SIR & 90.11 & [88.67, 91.48] & 305.1 & [298.5, 311.7]\\
 &  & Majority vote & 90.58 & [89.12, 91.94] & 325.3 & [318.0, 332.7]\\
\addlinespace[2pt]
5 & 1653.7 & MH & 90.15 & [88.68, 91.55] & 297.6 & [291.4, 303.9]\\
 &  & SIR & 90.24 & [88.79, 91.61] & 298.2 & [292.0, 304.5]\\
 &  & Majority vote & 91.06 & [89.60, 92.42] & 324.2 & [317.0, 331.5]\\
\addlinespace[2pt]
9 & 2976.6 & MH & 90.47 & [89.00, 91.83] & 292.3 & [286.4, 298.3]\\
 &  & SIR & 90.35 & [88.89, 91.70] & 293.6 & [287.6, 299.6]\\
 &  & Majority vote & 91.22 & [89.76, 92.60] & 323.9 & [316.8, 331.2]\\
\addlinespace[2pt]
17 & 5619.2 & MH & 90.29 & [88.81, 91.66] & 288.4 & [282.7, 294.1]\\
 &  & SIR & 90.42 & [88.97, 91.78] & 289.5 & [283.8, 295.1]\\
 &  & Majority vote & 91.48 & [90.03, 92.85] & 323.9 & [316.7, 331.1]\\
\bottomrule
\end{tabular}
\end{table}

\makeatletter\setlength{\@fptop}{0pt}\makeatother
\begin{table}[tp]
\centering
\caption{Paired MH-minus-SIR differences on identical fresh pools, with pointwise
95\% prompt-bootstrap intervals. At $M=1$ both methods return the first proposal,
so both differences are $0$. The difference in selected $\log\rho$ describes the
selection mechanism and is not a KL improvement.}
\label{tab:gsm8k-mh-sir}
\footnotesize
\begin{tabular}{rccc}
\toprule
& \multicolumn{3}{c}{Accuracy weight $w_1$}\\
\cmidrule(l){2-4}
$M$ & 0.25 & 0.5 & 0.75\\
\midrule
\multicolumn{4}{l}{\textit{MH $-$ SIR accuracy (percentage points)}}\\
2 & 0.03 [$-$0.12, 0.17] & 0.09 [$-$0.03, 0.21] & 0.07 [$-$0.06, 0.20]\\
3 & $-$0.04 [$-$0.20, 0.12] & 0.06 [$-$0.08, 0.18] & 0.05 [$-$0.08, 0.18]\\
5 & 0.00 [$-$0.19, 0.20] & 0.03 [$-$0.12, 0.18] & $-$0.08 [$-$0.23, 0.07]\\
9 & $-$0.13 [$-$0.34, 0.08] & $-$0.09 [$-$0.27, 0.08] & 0.12 [$-$0.07, 0.30]\\
17 & 0.03 [$-$0.22, 0.28] & $-$0.01 [$-$0.20, 0.17] & $-$0.14 [$-$0.32, 0.04]\\
\midrule
\multicolumn{4}{l}{\textit{MH $-$ SIR selected $\log\rho$}}\\
2 & 0.0594 [0.0524, 0.0662] & 0.0513 [0.0445, 0.0581] & 0.0570 [0.0507, 0.0634]\\
3 & 0.0853 [0.0767, 0.0938] & 0.0782 [0.0692, 0.0870] & 0.0723 [0.0635, 0.0810]\\
5 & 0.0802 [0.0681, 0.0923] & 0.1084 [0.0972, 0.1198] & 0.0858 [0.0746, 0.0971]\\
9 & 0.0883 [0.0753, 0.1018] & 0.1155 [0.1013, 0.1296] & 0.0981 [0.0843, 0.1116]\\
17 & 0.0826 [0.0675, 0.0979] & 0.1019 [0.0850, 0.1190] & 0.0961 [0.0801, 0.1123]\\
\bottomrule
\end{tabular}
\end{table}
Fresh-pool effective sample sizes and maximum normalized weights are recorded
in \nolinkurl{pool-concentration.csv}.
The companion code provides separate CPU replay commands for the raw
selections and for a compact cache.

\paragraph{Serving cost.}
Table~\ref{tab:gsm8k-serving-cost} compares local decoding with batched MH correction on one H200. Each timed batch contains one prompt and $M$ candidates. Each weighting and budget uses 48 timed trials, comprising 16 prompts with 3 repetitions each. Both specialist policies run in FP32, and the measurements exclude model loading and warm-up. Relative to $M=1$, mean latency increases by 16--20\% at $M=3$ and 29--35\% at $M=9$. These are latency measurements for this batch configuration, not throughput estimates. Total generated tokens include every candidate, including discarded responses. Full timing results for MH, SIR, and majority voting are archived under \nolinkurl{results/output-evaluation/gsm8k/profile}.

\begin{table}[!htbp]
\centering
\small
\caption{Measured GSM8K serving cost on one H200. At $M=1$, MH returns the local decoder's single proposal. Total tokens count all generated candidates and are averaged over the timed trials. Peak allocated memory is the maximum over 48 trials in each stratum.}
\label{tab:gsm8k-serving-cost}
\begin{tabular}{rrrrrr}
\toprule
$w_1$ & $M$ & Mean latency (s) & p95 (s) & Total tokens & Peak allocated (GB)\\
\midrule
0.25 & 1 & 17.47 & 26.11 & 248.7 & 41.12\\
0.25 & 3 & 20.25 & 29.54 & 746.0 & 41.24\\
0.25 & 9 & 22.53 & 33.33 & 2209.0 & 41.57\\
0.5 & 1 & 21.11 & 31.76 & 272.8 & 41.13\\
0.5 & 3 & 25.30 & 39.23 & 831.6 & 41.25\\
0.5 & 9 & 27.99 & 39.65 & 2442.0 & 41.61\\
0.75 & 1 & 21.45 & 30.34 & 305.0 & 41.13\\
0.75 & 3 & 25.10 & 37.12 & 898.4 & 41.25\\
0.75 & 9 & 28.90 & 46.37 & 2767.6 & 41.65\\
\bottomrule
\end{tabular}
\end{table}

\section{HH-RLHF experiment}\label{app:hh-experiment}
\subsection{HH-RLHF certificate evaluation}\label{app:hh-protocol}
\paragraph{Models, data, and specialist check.}
We compose helpfulness- and harmlessness-specialized LoRA adapters of
Gemma-4-E2B-it. Both were trained by DPO ($\beta=0.1$, rank 16, 1 epoch,
10,000 preference pairs) on the same 4,000 deduplicated conversations from the
Anthropic HH-RLHF training partition. Preferences were induced separately by
the pinned Ray2333 GPT-2-large helpfulness and harmlessness reward models used
by Rewards-in-Context \citep{yang2024ric}, using their raw scalar logits. On
a separate 500-prompt confirmation split, the helpfulness specialist's mean
helpfulness score exceeded the harmlessness specialist's by 0.607 (95\% CI
[0.500, 0.713]). The harmlessness specialist's mean harmlessness score exceeded
the helpfulness specialist's by 0.629 (95\% CI [0.531, 0.728]). These reward-model
scores are proxies for response quality and safety, not human judgments.

The fixed certificate cohort contains 1,319 distinct, hash-selected conversations
from the original test partition, with exact conversation overlap excluded from
training. The chat template preserves alternating multi-turn roles and
disables Gemma thinking. Generation uses temperature 1, the full vocabulary,
FP32 policy inference, FP64 normalization and accumulation, and a 512-token cap.
For each prompt and weight, 4 replications each contain 2 independent groups
of 17 fresh responses. Thus the complete run has 10,552 groups per weight and
538,152 responses overall.

\paragraph{Certificate.}
The certificate calculation follows Appendix~\ref{app:gsm8k-protocol}, except
that it uses a separately scoped family of 48 comparisons (3 weights by 16
nonzero update budgets). With clipping cap 10 nats and $\alpha=0.05$, the common
Hoeffding--Bonferroni subtraction is 0.180384 nats. No certificate value was
clipped. Table~\ref{tab:hh-certificate} reports the two endpoints.

\begin{table}[t]
\centering
\caption{HH-RLHF GPT-2-specialist certificate results, in nats. Cumulative MH
acceptance excludes initialization.}
\label{tab:hh-certificate}
\small
\setlength{\tabcolsep}{4pt}
\begin{tabular}{ccccccccc}
\toprule
& \multicolumn{3}{c}{$M=9$} & \multicolumn{3}{c}{$M=17$} & \multicolumn{2}{c}{MH acceptance (\%)}\\
\cmidrule(lr){2-4}\cmidrule(lr){5-7}\cmidrule(l){8-9}
$w_1$ & Mean $G_M$ & Lower bound & Clipped & Mean $G_M$ & Lower bound & Clipped & $n=8$ & $n=16$\\
\midrule
0.25 & 0.228 & 0.048 & 0.000 & 0.244 & 0.063 & 0.000 & 77.19 & 75.86\\
0.50 & 0.427 & 0.247 & 0.000 & 0.460 & 0.280 & 0.000 & 71.18 & 69.27\\
0.75 & 0.321 & 0.141 & 0.000 & 0.338 & 0.158 & 0.000 & 74.71 & 73.27\\
\bottomrule
\end{tabular}
\end{table}

\subsection{HH-RLHF output-level replay and cost}\label{app:hh-output-evaluation}
\paragraph{Replay protocol.}
The replay verifies all 7,914 raw batch checksums, identities, seeds, acceptance
uniforms, and probability identities. Every response is scored by both pinned
GPT-2 reward models. If a prompt--response pair exceeds the 1,024-token GPT-2
context, we trim only the response tokens for scoring. This affects 118 responses
(0.022\% of the total), with counts recorded in the scoring artifacts. The local
decoder returns proposal 0, and MH replays the saved acceptance uniforms. SIR
metrics average over its exact selection probabilities within each pool.
The 10,000-draw paired bootstrap resamples whole prompts while retaining
all 8 groups. Intervals are descriptive pointwise 95\%
percentile intervals, separate from the simultaneous certificate bounds.

At $M=17$, MH responses are 6--10 tokens shorter on average than local decoding
(Table~\ref{tab:hh-output}), but this comparison does not establish that shorter
length causes the reward-score changes.

\begin{table}[p]
\centering
\caption{Output-level HH-RLHF results at every candidate budget. H and S are
the raw helpfulness and harmlessness reward logits, and weighted reward is
$w_1\mathrm{H}+(1-w_1)\mathrm{S}$. Length counts returned-response tokens; total
tokens include all generated candidates.}
\label{tab:hh-output}
\footnotesize
\setlength{\tabcolsep}{3.5pt}
\begin{tabular}{rrlrrrr}
\toprule
$M$ & Total tokens & Method & H & S & Weighted & Length\\
\midrule
\multicolumn{7}{l}{\textit{Helpfulness weight $w_1=0.25$}}\\
1 & 244.5 & Local decoder ($q_w$) & 1.611 & $-$0.832 & $-$0.221 & 244.5\\
\addlinespace[2pt]
2 & 488.7 & MH & 1.573 & $-$0.808 & $-$0.213 & 241.3\\
 &  & SIR & 1.581 & $-$0.815 & $-$0.216 & 242.1\\
\addlinespace[2pt]
3 & 733.8 & MH & 1.568 & $-$0.820 & $-$0.223 & 241.9\\
 &  & SIR & 1.569 & $-$0.813 & $-$0.218 & 241.4\\
\addlinespace[2pt]
5 & 1223.5 & MH & 1.562 & $-$0.806 & $-$0.214 & 240.1\\
 &  & SIR & 1.560 & $-$0.809 & $-$0.217 & 240.8\\
\addlinespace[2pt]
9 & 2202.3 & MH & 1.538 & $-$0.814 & $-$0.226 & 240.3\\
 &  & SIR & 1.549 & $-$0.808 & $-$0.218 & 240.0\\
\addlinespace[2pt]
17 & 4158.3 & MH & 1.534 & $-$0.790 & $-$0.209 & 238.5\\
 &  & SIR & 1.541 & $-$0.803 & $-$0.217 & 239.2\\
\midrule
\multicolumn{7}{l}{\textit{Helpfulness weight $w_1=0.5$}}\\
1 & 263.9 & Local decoder ($q_w$) & 1.802 & $-$1.043 & 0.379 & 263.9\\
\addlinespace[2pt]
2 & 526.2 & MH & 1.761 & $-$1.013 & 0.374 & 259.9\\
 &  & SIR & 1.765 & $-$1.017 & 0.374 & 260.6\\
\addlinespace[2pt]
3 & 790.2 & MH & 1.742 & $-$0.995 & 0.373 & 258.9\\
 &  & SIR & 1.746 & $-$1.003 & 0.371 & 259.5\\
\addlinespace[2pt]
5 & 1315.9 & MH & 1.725 & $-$0.995 & 0.365 & 257.3\\
 &  & SIR & 1.724 & $-$0.991 & 0.366 & 257.7\\
\addlinespace[2pt]
9 & 2369.2 & MH & 1.709 & $-$0.978 & 0.366 & 255.8\\
 &  & SIR & 1.707 & $-$0.979 & 0.364 & 256.5\\
\addlinespace[2pt]
17 & 4473.7 & MH & 1.690 & $-$0.953 & 0.368 & 254.5\\
 &  & SIR & 1.697 & $-$0.970 & 0.363 & 255.2\\
\midrule
\multicolumn{7}{l}{\textit{Helpfulness weight $w_1=0.75$}}\\
1 & 284.2 & Local decoder ($q_w$) & 1.989 & $-$1.231 & 1.184 & 284.2\\
\addlinespace[2pt]
2 & 567.6 & MH & 1.947 & $-$1.205 & 1.159 & 279.6\\
 &  & SIR & 1.946 & $-$1.203 & 1.159 & 280.2\\
\addlinespace[2pt]
3 & 851.0 & MH & 1.929 & $-$1.194 & 1.148 & 277.8\\
 &  & SIR & 1.927 & $-$1.193 & 1.147 & 278.4\\
\addlinespace[2pt]
5 & 1418.1 & MH & 1.894 & $-$1.175 & 1.127 & 276.5\\
 &  & SIR & 1.898 & $-$1.177 & 1.130 & 276.8\\
\addlinespace[2pt]
9 & 2552.7 & MH & 1.877 & $-$1.157 & 1.118 & 274.7\\
 &  & SIR & 1.881 & $-$1.165 & 1.120 & 275.6\\
\addlinespace[2pt]
17 & 4820.6 & MH & 1.875 & $-$1.164 & 1.115 & 274.2\\
 &  & SIR & 1.874 & $-$1.160 & 1.116 & 274.8\\
\bottomrule
\end{tabular}
\end{table}

\begin{table}[p]
\centering
\caption{Paired MH-minus-SIR differences on identical HH-RLHF pools, with
pointwise 95\% prompt-bootstrap intervals. The selected $\log\rho$ difference
is a selection diagnostic and is not a KL improvement.}
\label{tab:hh-mh-sir}
\footnotesize
\begin{tabular}{rccc}
\toprule
& \multicolumn{3}{c}{Helpfulness weight $w_1$}\\
\cmidrule(l){2-4}
$M$ & 0.25 & 0.5 & 0.75\\
\midrule
\multicolumn{4}{l}{\textit{MH $-$ SIR helpfulness reward}}\\
2 & $-$0.0073 [$-$0.0154, 0.0004] & $-$0.0044 [$-$0.0126, 0.0040] & 0.0015 [$-$0.0071, 0.0102]\\
3 & $-$0.0011 [$-$0.0116, 0.0094] & $-$0.0046 [$-$0.0144, 0.0050] & 0.0019 [$-$0.0084, 0.0122]\\
5 & 0.0011 [$-$0.0103, 0.0127] & 0.0010 [$-$0.0099, 0.0121] & $-$0.0048 [$-$0.0168, 0.0073]\\
9 & $-$0.0110 [$-$0.0235, 0.0017] & 0.0016 [$-$0.0100, 0.0135] & $-$0.0046 [$-$0.0173, 0.0079]\\
17 & $-$0.0066 [$-$0.0199, 0.0063] & $-$0.0072 [$-$0.0198, 0.0055] & 0.0010 [$-$0.0119, 0.0138]\\
\midrule
\multicolumn{4}{l}{\textit{MH $-$ SIR harmlessness reward}}\\
2 & 0.0070 [$-$0.0017, 0.0155] & 0.0048 [$-$0.0035, 0.0132] & $-$0.0025 [$-$0.0117, 0.0069]\\
3 & $-$0.0072 [$-$0.0177, 0.0037] & 0.0081 [$-$0.0022, 0.0186] & $-$0.0007 [$-$0.0118, 0.0105]\\
5 & 0.0027 [$-$0.0092, 0.0146] & $-$0.0047 [$-$0.0164, 0.0070] & 0.0023 [$-$0.0102, 0.0148]\\
9 & $-$0.0067 [$-$0.0198, 0.0062] & 0.0019 [$-$0.0107, 0.0142] & 0.0081 [$-$0.0052, 0.0213]\\
17 & 0.0128 [$-$0.0010, 0.0263] & 0.0171 [0.0045, 0.0299] & $-$0.0047 [$-$0.0181, 0.0091]\\
\midrule
\multicolumn{4}{l}{\textit{MH $-$ SIR weighted reward}}\\
2 & 0.0034 [$-$0.0021, 0.0088] & 0.0002 [$-$0.0032, 0.0037] & 0.0005 [$-$0.0047, 0.0059]\\
3 & $-$0.0057 [$-$0.0121, 0.0009] & 0.0018 [$-$0.0023, 0.0059] & 0.0013 [$-$0.0050, 0.0076]\\
5 & 0.0023 [$-$0.0050, 0.0096] & $-$0.0019 [$-$0.0067, 0.0028] & $-$0.0030 [$-$0.0104, 0.0042]\\
9 & $-$0.0078 [$-$0.0159, 0.0004] & 0.0018 [$-$0.0033, 0.0068] & $-$0.0014 [$-$0.0092, 0.0062]\\
17 & 0.0079 [$-$0.0004, 0.0162] & 0.0049 [$-$0.0003, 0.0102] & $-$0.0005 [$-$0.0084, 0.0074]\\
\midrule
\multicolumn{4}{l}{\textit{MH $-$ SIR selected $\log\rho$}}\\
2 & 0.0324 [0.0275, 0.0372] & 0.0372 [0.0315, 0.0427] & 0.0301 [0.0248, 0.0352]\\
3 & 0.0395 [0.0327, 0.0463] & 0.0467 [0.0397, 0.0537] & 0.0390 [0.0322, 0.0457]\\
5 & 0.0385 [0.0310, 0.0461] & 0.0446 [0.0359, 0.0534] & 0.0458 [0.0370, 0.0544]\\
9 & 0.0322 [0.0236, 0.0409] & 0.0453 [0.0356, 0.0552] & 0.0484 [0.0389, 0.0578]\\
17 & 0.0273 [0.0171, 0.0373] & 0.0394 [0.0283, 0.0505] & 0.0348 [0.0251, 0.0449]\\
\bottomrule
\end{tabular}
\end{table}

Fresh-pool effective sample sizes and maximum normalized weights are recorded in
\nolinkurl{pool-concentration.csv}. Token costs include EOS and all discarded
candidates. The reward-score sidecars are cryptographically tied to the raw
certificate files and pinned reward revisions.
\paragraph{Serving cost.}
Table~\ref{tab:hh-serving-cost} reports end-to-end generation and MH selection
on one H200. Each timed batch contains one prompt and $M$ responses. The 16
prompts are hash-selected, and each prompt, budget, and weight has 3 timed
repetitions after a separate warm-up, giving 432 trials in total. The measurements
include both FP32 Gemma component policies and exclude model loading and warm-up.

\begin{table}[t]
\centering
\caption{Measured HH-RLHF serving cost on one H200. Peak allocated memory is
the maximum over 48 timed trials in each stratum.}
\label{tab:hh-serving-cost}
\small
\begin{tabular}{rrrrr}
\toprule
$w_1$ & $M$ & MH mean latency (s) & p95 (s) & Peak allocated (GB)\\
\midrule
0.25 & 1 & 20.92 & 36.07 & 41.23\\
0.25 & 3 & 23.20 & 38.35 & 41.56\\
0.25 & 9 & 24.54 & 39.55 & 42.54\\
0.5 & 1 & 21.74 & 37.17 & 41.23\\
0.5 & 3 & 24.24 & 39.51 & 41.56\\
0.5 & 9 & 25.67 & 40.59 & 42.54\\
0.75 & 1 & 23.48 & 40.62 & 41.23\\
0.75 & 3 & 25.94 & 42.97 & 41.56\\
0.75 & 9 & 28.19 & 43.98 & 42.54\\
\bottomrule
\end{tabular}
\end{table}

\clearpage
\section{Trained LLM continuation-tree experiment}\label{app:llm-tree}

This experiment directly compares the consensus objectives of MH and ordinary
SIR on enumerable distributions derived from the trained GSM8K specialists.
It complements the full-response certificates in Section~\ref{subsec:llm}
by making the target and both output laws available for direct evaluation.
All quantities below refer to the restricted continuation policies defined
here, rather than the full-response policies.

\subsection{Policies, questions, and continuation trees}\label{app:llm-tree-protocol}

We reuse the Gemma-4-E2B-it accuracy and brevity DPO adapters and their shared
reference model from Appendix~\ref{app:gsm8k-protocol}, with no additional
training. A seeded permutation of the GSM8K test set selects 5 pilot
questions and 50 separate evaluation questions.
These questions are held out from specialist training but belong to the
same 1,319-question test cohort as the full-response study. The pilot checks
implementation, numerical precision, memory, and runtime. The final protocol
is frozen before evaluating the 50 questions. The original prompt template
and tokenizer are reused. No questions are selected or excluded on the basis
of reference correctness, specialist disagreement, or the observed MH--SIR gap.

Each question supplies an \emph{original} context at the normal assistant
response boundary and a \emph{midpoint} context. For the latter, we generate
one greedy reference-model response with a 1,024-token cap, remove EOS from
the response-token count, and retain the first half of the non-EOS tokens,
rounding down. We append these exact token IDs to the original context,
without decoding and re-tokenizing or adding another assistant header.
The protocol specifies the original context as a fallback for an empty
midpoint. None of the 50 evaluation questions requires this fallback.

For each context, we expand the reference model's top $K=8$ tokens at every
reachable prefix through a horizon of $T=4$ additional tokens. Ties in reference
logits are broken by ascending token ID. The same mask is applied to each
adapter separately, at temperature 1. An EOS token ends a branch only if it
is in the mask. At the horizon, the branch ends without an additional EOS
factor. Without early EOS, this gives 585 nonterminal prefixes and 4,096
leaves per context. We cache the allowed token IDs in reference order,
selected logits, full-vocabulary log normalizers, prefix token IDs, and EOS
IDs. All weights, budgets, and sensitivity checks reuse these evaluations.

\paragraph{The restricted target.}
Let $A(h)$ be the shared allowed-token set at prefix $h$. For
$i\in\{a,b\}$, write $\pi_i^{\mathrm{full}}$ for the original accuracy or
brevity policy and define its restricted conditional by
\begin{equation}\label{eq:tree-target}
 \pi_i(v\mid h)=
 \frac{\pi_i^{\mathrm{full}}(v\mid h)}
      {\sum_{u\in A(h)}\pi_i^{\mathrm{full}}(u\mid h)},
 \qquad v\in A(h),
\end{equation}
with probability $0$ outside $A(h)$. Within this appendix, $q_w$ and
$\widehat p_w$ denote the proposal and target on this tree. With accuracy
weight $w=w_1\in\{0.25,0.5,0.75\}$, define
\[
 z_w(h)=\sum_{v\in A(h)}\pi_a(v\mid h)^w\pi_b(v\mid h)^{1-w},
 \qquad
 q_w(v\mid h)=\frac{\pi_a(v\mid h)^w\pi_b(v\mid h)^{1-w}}{z_w(h)}.
\]
For each complete continuation $y$, multiply these factors along its path:
\[
 \rho_w(y)=\prod_t z_w(h_t),\qquad
 \widehat p_w(y)=\frac{q_w(y)\rho_w(y)}{Z_w},\qquad
 Z_w=\sum_y q_w(y)\rho_w(y).
\]
We independently verify $q_w(y)\rho_w(y)=\pi_a(y)^w\pi_b(y)^{1-w}$.
Masking and renormalizing the adapters changes their sequence distributions:
this target is generally different from the full-vocabulary product target
conditioned on remaining within the tree.

\subsection{Full output laws and numerical validation}\label{app:llm-tree-numerics}

The candidate budgets are $M\in\{1,2,3,5,9\}$, counting MH initialization.
We compute $H_M=q_wK^{M-1}$ by sorting the leaf weights and accumulating
prefix and suffix sums, using the recursion in
Appendix~\ref{app:scalar-law}. Ordinary SIR draws $M$ IID candidates from
$q_w$ and resamples proportionally to $\rho_w$. Its marginal law is computed
by the scalar Laplace identity
\[
 S_M(y)=M q_w(y)\rho_w(y)\int_0^\infty
 e^{-t\rho_w(y)}
 \left(\sum_v q_w(v)e^{-t\rho_w(v)}\right)^{M-1}\,dt,
\]
where the sum is over complete continuations. We evaluate this integral
adaptively in log time, after an invariant common rescaling of the weights.
Neither calculation samples an output law or constructs a dense transition
matrix. The directly computed divergences, all in nats, are
\[
 D_q=\KL(q_w\|\widehat p_w),\qquad
 D_H(M)=\KL(H_M\|\widehat p_w),\qquad
 D_S(M)=\KL(S_M\|\widehat p_w).
\]
The consensus identity gives the actual gains
$\Delta_H(M)=D_q-D_H(M)$ and $\Delta_S(M)=D_q-D_S(M)$, and the objective gap
\[
 \widetilde J_w(H_M)-\widetilde J_w(S_M)=D_S(M)-D_H(M),\qquad
 \widetilde J_w(\mu)=-w\KL(\mu\|\pi_a)-(1-w)\KL(\mu\|\pi_b).
\]
We also verify this identity by evaluating the two component KL divergences
directly. DPO training alone does not establish that this restricted-policy
consensus objective equals the original task-reward objective.

\paragraph{Precision and implementation checks.}
Model inference uses FP32 with TF32 disabled, and probability arithmetic
uses FP64. We check normalization before any numerical renormalization,
the two target constructions, the $M=1$ identities, and the KL ordering.
Diagnostic tests cover horizon 1, identical adapters, a single nonzero
policy weight, and agreement with direct MH and SIR calculations on tiny
supports, including repeated candidate draws. The pilot also compares
cached and direct model evaluation at 30 prefixes: the reference top-8
rankings agree, and the largest change in a restricted log probability is
$4.80\times10^{-5}$. This model-inference variation is separate from the
numerical precision checks below, which are conditional on the cached logits.

The nominal quadrature tolerance is $10^{-12}$. Every evaluation is repeated
at $10^{-13}$ and with extended-precision probability arithmetic. Pilot
amendments, recorded before the final run, allow quadrature termination due
to roundoff only when both its estimated $L^1$ error and its normalization
residual are below $10^{-10}$, with independent law and KL stability checks.
Across all 1,200 final context--weight--tree cases, the largest normalization
residual is $1.93\times10^{-14}$ and the largest KL change under the precision
checks is $1.91\times10^{-14}$ nats. All cases pass these checks.
The prespecified reporting resolution is $10^{-10}$ nats, and smaller
comparisons are reported as unresolved. These checks do not provide rigorous
error bounds. Enumeration removes sampling and density-estimation error from
the conditional KL calculations, while floating-point and quadrature error
remain.

\subsection{Results, certificates, and uncertainty}\label{app:llm-tree-results}

Figure~\ref{fig:llm-tree-kl} shows mean divergences over the 50 questions
for original and midpoint contexts at $w_1=0.5$. The tables below report
results at all three weights. For the common certificate $G_M$
in~\eqref{eq:certificate}, we draw 100,000 IID pools of 9 leaves from each
enumerated proposal and use pool prefixes for smaller budgets. This requires
no further model inference. We estimate $\E[G_M]$ and its Monte Carlo
standard error separately for every question, context, weight, and tree
setting. Its expectation is a theoretical lower bound on both gains.
The reported sample mean is not a certified lower confidence bound.

Arithmetic means over questions are the primary summaries. Pointwise 95\%
intervals use 10,000 bootstrap samples of whole questions,
keeping both contexts, all weights, all budgets, and all sensitivity settings
together within each resampled question. These intervals describe variation
across questions and do not include model-inference uncertainty. For the
mean certificate, the separate Monte Carlo standard error is the root sum
of squared per-question standard errors divided by 50.

\begin{figure}[t]
\centering
\includegraphics[width=\linewidth]{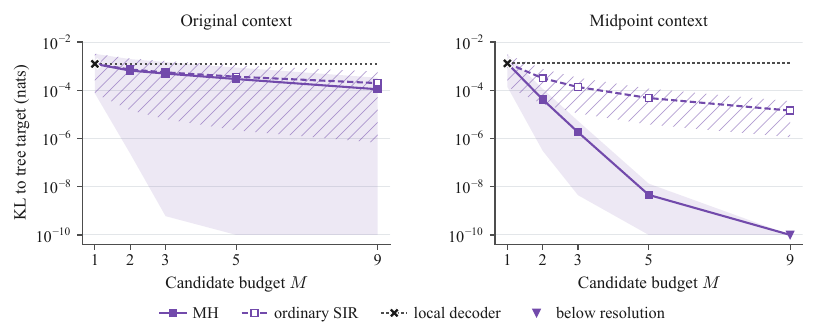}
\caption{\textbf{KL to the target on trained-policy continuation trees.}
Curves average 50 questions at accuracy weight $w_1=0.5$. Left: original
prompts. Right: midpoint contexts, taken halfway through a greedy
reference-model response. Each tree uses the shared reference top-8
tokens at each prefix and a 4-token horizon. The target is the product
of the restricted, separately renormalized specialists. Shading gives
pointwise 95\% question-bootstrap intervals, hatched for SIR.
The dotted baseline is the local decoder. Downward triangles mark means
below the $10^{-10}$-nat reporting resolution; interval endpoints below
this value are displayed at the same floor. Both panels use the same scale.}
\label{fig:llm-tree-kl}
\end{figure}

\begin{table}[!htbp]
\centering
\caption{\textbf{Direct comparisons at $M=9$ on top-8, horizon-4 trees.}
All numerical entries other than $w_1$ are in $10^{-6}$ nats and are means
over 50 questions. Parentheses give the separate Monte Carlo standard error
of the mean certificate. $D_S-D_H$ is the MH consensus advantage.
$<0.0001$ marks a mean KL below the $10^{-10}$-nat reporting resolution.}
\label{tab:llm-tree-main}
\small
\setlength{\tabcolsep}{4pt}
\begin{tabular}{lrrrrrr}
\toprule
Context & $w_1$ & $D_q$ & $D_H(9)$ & $D_S(9)$ & $D_S-D_H$ & Est. $\E[G_9]$ (SE)\\
\midrule
Original & 0.25 & 2272 & 105.7 & 271.8 & 166.1 & 1639 (8.27)\\
Original & 0.5 & 1229 & 109.7 & 197.2 & 87.49 & 843.5 (6.81)\\
Original & 0.75 & 136.5 & 0.153 & 4.752 & 4.599 & 115.8 (1.00)\\
\midrule
Midpoint & 0.25 & 884.3 & $<0.0001$ & 10.36 & 10.36 & 788.7 (0.542)\\
Midpoint & 0.5 & 1308 & $<0.0001$ & 14.33 & 14.33 & 1171 (0.834)\\
Midpoint & 0.75 & 860.2 & $<0.0001$ & 9.362 & 9.362 & 770.1 (0.479)\\
\bottomrule
\end{tabular}
\end{table}

\begin{table}[!htbp]
\centering
\caption{\textbf{Variation in the MH--SIR objective gap at $M=9$.}
Primary top-8, horizon-4 setting. The first three numerical summaries
are in $10^{-6}$ nats. Intervals are pointwise 95\% question-bootstrap
intervals for the mean; SD describes per-question variation. The last column
counts absolute per-question gaps no larger than $10^{-10}$ nats.}
\label{tab:llm-tree-variation}
\small
\begin{tabular}{lrrrrr}
\toprule
Context & $w_1$ & Mean gap [95\% CI] & Question SD & Median gap & Unresolved / 50\\
\midrule
Original & 0.25 & 166.1 [0.748, 493.7] & 1149 & 0.0764 & 5\\
Original & 0.5 & 87.49 [0.656, 259.3] & 603.3 & 0.0779 & 5\\
Original & 0.75 & 4.599 [0.296, 12.71] & 27.67 & 0.0284 & 8\\
\midrule
Midpoint & 0.25 & 10.36 [2.049, 21.24] & 35.60 & 0.00150 & 17\\
Midpoint & 0.5 & 14.33 [1.128, 35.09] & 65.42 & 0.00339 & 15\\
Midpoint & 0.75 & 9.362 [0.623, 24.48] & 47.76 & 0.00259 & 18\\
\bottomrule
\end{tabular}
\end{table}

At $M=9$, all six primary mean objective gaps have positive pointwise
bootstrap intervals (Tables~\ref{tab:llm-tree-main} and
\ref{tab:llm-tree-variation}). The absolute advantage is small, and question
heterogeneity is substantial: for original contexts at $w_1=0.25$ and $0.5$,
the largest single-question gap contributes about 98\% of the total gap.
The original-context median gaps are only $2.84\times10^{-8}$ to
$7.79\times10^{-8}$ nats. In midpoint contexts, mean MH KL at $M=9$ is
below the $10^{-10}$-nat reporting resolution at all three weights. This
does not establish exact convergence or an unbounded relative advantage.
The estimated mean certificate is approximately 75--90\% of mean MH gain
at this budget.

\Needspace{12\baselineskip}
\subsection{Horizon and token-mask sensitivity}\label{app:llm-tree-sensitivity}

The three prespecified sensitivity checks use the same 50 questions, two
contexts, three weights, and five budgets. We evaluate top-8 trees at
horizons 2 and 3, and top-4 trees at horizon 4, using only cached
logits. For a shorter horizon we stop at that depth and recompute the target
normalization and output laws. The resulting target generally differs from
the corresponding marginal of the horizon-4 target. For the top-4
check, we keep the first 4 tokens in reference order at each reachable
prefix and renormalize each adapter again. Each setting contributes 300
context--weight cases, for 1,200 cases including the primary setting.

\begin{table}[!htbp]
\centering
\caption{\textbf{Horizon and mask sensitivity at $M=9$.}
Mean MH consensus advantage $D_S-D_H$, in $10^{-6}$ nats.
All 24 pointwise 95\% intervals for these means exclude $0$; the complete
intervals and results at every budget are in the companion tables.
These settings define different restricted targets.}
\label{tab:llm-tree-sensitivity}
\small
\begin{tabular}{lrrrrr}
\toprule
& & Primary & \multicolumn{3}{c}{Sensitivity checks}\\
\cmidrule(lr){3-3}\cmidrule(lr){4-6}
Context & $w_1$ & $K=8,T=4$ & $K=8,T=2$ & $K=8,T=3$ & $K=4,T=4$\\
\midrule
Original & 0.25 & 166.1 & 2.405 & 165.4 & 166.1\\
Original & 0.5 & 87.49 & 1.152 & 86.74 & 87.47\\
Original & 0.75 & 4.599 & 0.280 & 4.284 & 4.596\\
\midrule
Midpoint & 0.25 & 10.36 & 0.647 & 5.904 & 7.753\\
Midpoint & 0.5 & 14.33 & 1.274 & 10.57 & 9.319\\
Midpoint & 0.75 & 9.362 & 0.666 & 7.550 & 6.576\\
\bottomrule
\end{tabular}
\end{table}

The positive mean MH advantage persists in every sensitivity setting at
$M=9$ (Table~\ref{tab:llm-tree-sensitivity}). Many individual gaps remain
unresolved: depending on context and weight, 25--33 of the 50 gaps are
unresolved at horizon 2, 11--23 at horizon 3, and 6--19 with the
top-4 mask at horizon 4. This sweep from 2 to 4 tokens is a
sensitivity check, not a test of the long-response asymptotic theorem.

\paragraph{Probability mass retained by the mask.}
For adapter $i$, we record the original full-vocabulary mass
$m_i(h)=\sum_{v\in A(h)}\pi_i^{\mathrm{full}}(v\mid h)$.
Table~\ref{tab:llm-tree-mask} reports its proposal-weighted mean at each
depth for $w_1=0.5$, conditional on reaching that depth, followed by an
arithmetic mean over questions. Per-prefix minimum, quartiles, maximum,
and proposal-weighted means at every weight are in the released evaluation
records. High proposal-weighted retention need not hold at rare prefixes:
the smallest recorded brevity mass on an original-context top-8 prefix
is $3.61\times10^{-6}$. Retained mass describes the restriction but does not
establish equivalence to unrestricted generation.

\begin{table}[!htbp]
\centering
\caption{\textbf{Retained full-vocabulary mass by depth.}
Question means of proposal-weighted retained mass, in percent, at
$w_1=0.5$ and horizon 4. Depth 0 is the continuation root. Values
rounding to 100.000\% do not imply that every prefix retains all mass.}
\label{tab:llm-tree-mask}
\small
\begin{tabular}{llrrrrr}
\toprule
Context & Adapter & $K$ & Depth 0 & Depth 1 & Depth 2 & Depth 3\\
\midrule
Original & Accuracy & 8 & 100.000 & 100.000 & 100.000 & 100.000\\
Original & Brevity & 8 & 99.998 & 100.000 & 99.999 & 100.000\\
Original & Accuracy & 4 & 100.000 & 100.000 & 99.966 & 99.999\\
Original & Brevity & 4 & 99.994 & 99.998 & 99.920 & 99.998\\
\midrule
Midpoint & Accuracy & 8 & 99.756 & 99.914 & 99.876 & 99.926\\
Midpoint & Brevity & 8 & 99.494 & 99.866 & 99.784 & 99.872\\
Midpoint & Accuracy & 4 & 99.258 & 99.700 & 99.681 & 99.805\\
Midpoint & Brevity & 4 & 98.746 & 98.331 & 99.003 & 99.696\\
\bottomrule
\end{tabular}
\end{table}

\Needspace{11\baselineskip}
\subsection{Compute and reproduction}\label{app:llm-tree-reproduction}

The final run constructs 100 context trees and evaluates 58,415 nonterminal
prefixes per model, with 11,400 GPU forward calls across the reference and
two adapters. We reuse each context's KV cache, batch up to 16 prefixes at
the same depth, and switch adapters on the shared base model. On one H200,
tree evaluation takes 553 seconds and greedy reference generation takes
578 seconds, with peak allocated memory of 21.1 GB. These are offline tree
construction costs, not sampler serving latencies. The numerical laws,
certificates, bootstrap summaries, and all sensitivity checks run on CPU.

The companion suite includes the frozen configuration, question and context
token IDs, checkpoint hashes, seeds, cached trees, numerical validation
records, per-question results, and plotting code. Results are archived under
\nolinkurl{results/llm-tree}. The \texttt{llm-tree-analysis} job recomputes
numerical results from cached model evaluations, and
\texttt{llm-tree-presentation} rebuilds the summaries and two archived plots.
The manuscript figure uses the same summaries. The accompanying instructions
explain how to move saved numerical outputs aside to force fresh computation.
A fresh CPU replay in a separate directory reproduced all 1,320 pilot and
final cases and their summaries, with no differences among 176,477 compared
numerical fields. This is a same-operator reproducibility check, not an
independent replication.

The study measures direct consensus-objective differences on short,
restricted policies and illustrates the ordering in
Theorem~\ref{thm:dominance}. The small absolute gaps do not establish a large
practical advantage, improved task rewards, or the unobserved MH--SIR
objective gap on full GSM8K responses.

\end{document}